\documentclass[a4paper, USenglish, cleveref, autoref, thm-restate]{lipics-v2021}

\title{Optimization Trade-offs in Asynchronous Federated Learning: A~Stochastic Networks Approach}
\titlerunning{Optimization Trade-offs in Asynchronous Federated Learning}

\author{Abdelkrim Alahyane\footnote{Corresponding author}}{EMINES, Mohammed VI Polytechnic University, Ben Guerir, Morocco \and LAAS-CNRS, Université de Toulouse, CNRS, Toulouse, France}{abdelkrim.alahyane@um6p.ma}{https://orcid.org/0009-0006-5142-8949}{}

\author{Céline Comte}{LAAS-CNRS, Université de Toulouse, CNRS, Toulouse, France}{celine.comte@cnrs.fr}{https://orcid.org/0009-0005-9413-7124}{}

\author{Matthieu Jonckheere}{LAAS-CNRS, Université de Toulouse, CNRS, Toulouse, France}{matthieu.jonckheere@laas.fr}{https://orcid.org/0000-0003-3576-5866}{}

\authorrunning{A. Alahyane, C. Comte, and M. Jonckheere}

\Copyright{Abdelkrim Alahyane, Céline Comte, and Matthieu Jonckheere}

\nolinenumbers 
\hideLIPIcs

\usepackage{amsmath}
\usepackage{amsfonts}
\usepackage{amssymb}
\usepackage{hyperref}
\usepackage{algorithmic}
\usepackage{algorithm}
\usepackage{array}
\usepackage{textcomp}
\usepackage{stfloats}
\usepackage{url}
\usepackage{verbatim}
\usepackage{graphicx}
\usepackage{balance}
\usepackage{booktabs}
\usepackage{multirow}
\usepackage{mathtools}
\usepackage{titletoc}

\usepackage{dsfont}  % \mathds{1}

\newcommand\bE{\mathbb{E}}
\newcommand\bEp{\mathbb{E}^0}
\newcommand\bP{\mathbb{P}}
\newcommand\bPp{\mathbb{P}^0}
\newcommand\bR{\mathbb{R}}
\newcommand\bN{\mathbb{N}}
\newcommand\bZ{\mathbb{Z}}
\newcommand\un{{\mathfrak{n}}}
\newcommand\um{{\mathfrak{m}}}

\newcommand\prb[1]{\mathbb{P}\left(#1\right)}
\newcommand\esp[1]{\mathbb{E}\left[#1\right]}

\newcommand\cP{\mathcal{P}}
\newcommand\cX{\mathcal{X}}

\newcommand\muc{\mu^{\mathrm{c}}}
\newcommand\muu{\mu^{\mathrm{u}}}
\newcommand\mud{\mu^{\mathrm{d}}}
\newcommand\muCS{\mu^\mathrm{cs}}

\newcommand\xic{\xi^{\mathrm{c}}}
\newcommand\xiu{\xi^{\mathrm{u}}}
\newcommand\xid{\xi^{\mathrm{d}}}

\newcommand\xc{x^{\mathrm{c}}}
\newcommand\xu{x^{\mathrm{u}}}
\newcommand\xd{x^{\mathrm{d}}}
\newcommand\xCS{x^\mathrm{cs}}

\newcommand\Xc{X^{\mathrm{c}}}
\newcommand\Xu{X^{\mathrm{u}}}
\newcommand\Xd{X^{\mathrm{d}}}

\newcommand\Yc{Y^{\mathrm{c}}}
\newcommand\Yu{Y^{\mathrm{u}}}
\newcommand\Yd{Y^{\mathrm{d}}}
\newcommand\YCS{Y^\mathrm{cs}}

\newcommand\yc{y^{\mathrm{c}}}
\newcommand\yu{y^{\mathrm{u}}}
\newcommand\yd{y^{\mathrm{d}}}
\newcommand\yCS{y^\mathrm{cs}}

\newcommand\tYc{\tilde{Y}^{\mathrm{c}}}
\newcommand\tYu{\tilde{Y}^{\mathrm{u}}}
\newcommand\tYd{\tilde{Y}^{\mathrm{d}}}
\newcommand\tYCS{\tilde{Y}^\mathrm{cs}}

\newcommand\Pc{\mathcal{P}^{\mathrm{c}}}
\newcommand\Pu{\mathcal{P}^{\mathrm{u}}}
\newcommand\Pd{\mathcal{P}^{\mathrm{d}}}
\newcommand\PCS{\mathcal{P}^\mathrm{cs}}

\newcommand\ec{\mathbf{e}^{\mathrm{c}}}
\newcommand\eu{\mathbf{e}^{\mathrm{u}}}
\newcommand\ed{\mathbf{e}^{\mathrm{d}}}
\newcommand\eCS{\mathbf{e}^\mathrm{cs}}

\newcommand\Zc{Z^{\mathrm{c}}}
\newcommand\Zu{Z^{\mathrm{u}}}
\newcommand\Zd{Z^{\mathrm{d}}}
\newcommand\Zcs{Z^\mathrm{cs}}

\newcommand\bZc{\bar{Z}^{\mathrm{c}}}
\newcommand\bZu{\bar{Z}^{\mathrm{u}}}
\newcommand\bZd{\bar{Z}^{\mathrm{d}}}
\newcommand\bZcs{\bar{Z}^\mathrm{cs}}

\newcommand\zc{z^{\mathrm{c}}}
\newcommand\zu{z^{\mathrm{u}}}
\newcommand\zd{z^{\mathrm{d}}}
\newcommand\zCS{z^\mathrm{cs}}

\newcommand\vCS{v^\mathrm{cs}}
\newcommand\MCS{M_\mathrm{cs}}

\newcommand\one[1]{\mathds{1}\mathopen{}\left[#1\right]\mathclose{}}

\newcommand\param{w}

\DeclareMathOperator{\cov}{Cov}

\allowdisplaybreaks  % allow page breaks in equations

\usepackage[acronym]{glossaries}
\glsdisablehyper
\newacronym{CS}{CS}{central server}
\newacronym{FIFO}{FIFO}{first-in-first-out}
\newacronym{FL}{FL}{Federated Learning}
\newacronym{iid}{i.i.d.}{independent and identically distributed}
\newacronym{GD}{GD}{gradient descent}
\newacronym{RCLL}{RCLL}{right-continuous with left limits}
\newacronym{SGD}{SGD}{stochastic gradient descent}
\newacronym{RB}{RB}{resource blocs}
\newacronym{DVFS}{DVFS}{dynamic voltage and frequency scaling}
\newacronym{OFDMA}{OFDMA}{Orthogonal Frequency-Division Multiple Access}

\usepackage{todonotes}

\begin{document}

\maketitle

\begin{abstract}
	Synchronous federated learning scales poorly due to the straggler effect. 
	Asynchronous algorithms increase the update throughput by processing updates upon arrival, but they introduce two fundamental challenges: gradient staleness, which degrades convergence, and bias toward faster clients under heterogeneous data distributions. 
	Although algorithms such as \texttt{AsyncSGD} and \texttt{Generalized AsyncSGD} mitigate this bias via client-side task queues, most existing analyses neglect the underlying queueing dynamics and lack closed-form characterizations of the update throughput and gradient staleness.
	
	To close this gap, we develop a stochastic
	queueing-network framework for \texttt{Generalized AsyncSGD} that jointly models random computation times at the clients and the central server, as well as random uplink and downlink communication delays.
	Leveraging product-form network theory, we derive a closed-form expression for the update throughput, alongside closed-form upper bounds for both the communication round complexity and the expected wall-clock time required to reach an $\epsilon$-stationary point. 
	These results formally characterize the trade-off between gradient staleness and wall-clock convergence speed.
	We further extend the framework to quantify energy consumption under stochastic timing, revealing an additional trade-off between convergence speed and energy efficiency. 
	
	Building on these analytical results, we propose gradient-based optimization strategies to jointly optimize routing and concurrency. 
	Experiments on EMNIST demonstrate reductions of 29\%--46\% in convergence time and 36\%--49\% in energy consumption compared to \texttt{AsyncSGD}.
\end{abstract}

\section{Introduction}\label{sec:intro}

Modern machine learning training relies on stochastic gradient-based methods~\cite{robbins1951stochastic} and their variants, such as Adam~\cite{kingma2014adam}.
To scale these methods to large models and datasets, training is often distributed across multiple clients that compute gradients in parallel while a \gls{CS} aggregates updates, a paradigm called
\gls{FL}.
Despite communication-efficient techniques such as gradient compression~\cite{alistarh2017qsgd,tyurin2022dasha}, decentralized communication~\cite{lian2017can}, and local updates~\cite{mcmahan2017communication} that reduce communication overhead, \gls{FL} systems
typically operate under synchronous protocols
whereby the \gls{CS} waits to have received all requested gradients before updating the model. Synchronous \gls{FL} systems are unfortunately slowed down by
the straggler effect in heterogeneous environments.

Asynchronous training~\cite{xie2019asynchronous,chen2020asynchronous,xu2023asynchronous} alleviates this limitation by allowing the \gls{CS} to process updates as they arrive, thereby improving resource utilization.  
However, naively introducing asynchrony in heterogeneous data settings leads to bias toward faster clients,
i.e., clients that
contribute updates more frequently.  
To mitigate this bias, recent works employ \textit{client-side queues} and tailored sampling schemes. Specifically, \cite{koloskova2022sharper} proposed \texttt{AsyncSGD} with uniform sampling, while \cite{leconte2024queueing} extended this to \texttt{Generalized AsyncSGD} using non-uniform routing probabilities. However, these studies largely oversee the underlying queueing dynamics, either omitting explicit analysis~\cite{koloskova2022sharper} or relying on asymptotic approximations~\cite{leconte2024queueing}. More critically, their convergence analysis is formulated in terms of communication rounds rather than wall-clock time, despite the latter being the primary motivation for asynchronous methods.

Several recent studies~\cite{tyurin2023optimal,tyurin_shadowheart_2024,maranjyan_mindflayer_2024,maranjyan_ringmaster_2025} attempt to address this gap by deriving time-based convergence guarantees.
However, these approaches have several limitations:
they enforce partial synchronization via time thresholds, bias learning against slow clients, discard near-complete computations, and typically assume deterministic computation times, limiting their applicability to realistic edge systems.
See \Cref{sec:related_works} for more details.

This manuscript substantially extends our preliminary conference paper~\cite{alahyane2025optimizing}.  
In that earlier work, we employed a stochastic queueing model to analyze asynchronous \gls{FL} in terms of wall-clock time, revealing a fundamental trade-off between update staleness and update frequency.  
However, the initial model neglected communication delays and did not treat concurrency as a controllable system parameter, despite both being critical determinants of practical performance.

In this extended version, we explicitly model communication phases and introduce concurrency as an optimization variable.  
We further enrich the theoretical framework by incorporating \gls{CS} processing speeds and analyzing the system’s energy footprint, thereby uncovering an additional trade-off between convergence speed and energy efficiency.
This extension enables a more realistic and operationally relevant characterization of asynchronous~\gls{FL}. 

\subsection{Contributions}

Building upon~\cite{koloskova2022sharper,leconte2024queueing,alahyane2025optimizing}, we develop a comprehensive stochastic framework for asynchronous \gls{FL} that explicitly models computation, communication, and \gls{CS}-side dynamics.  
Our main contributions are:
\begin{itemize} %[leftmargin=*, topsep=0pt, itemsep=0pt, parsep=0pt]
	
	\item \textbf{Generalized stochastic network model.}
	We introduce a unified stochastic queueing-network formulation for \texttt{Generalized AsyncSGD} (which extends \texttt{AsyncSGD}) that jointly captures random computation times, uplink and downlink communication delays, \gls{CS} processing speed, routing, and concurrency.  
	This extends prior stochastic models~\cite{leconte2024queueing,alahyane2025optimizing} by explicitly incorporating communication and \gls{CS}-side effects.
	
	\item \textbf{Closed-form delay and throughput characterization.}
	Leveraging product-form stochastic networks, we derive closed-form expressions for the average relative delay, update frequency, and their gradients with respect to routing probabilities, enabling exact performance analysis and gradient-based optimization.
	
	\item \textbf{Convergence analysis.}
	We establish convergence guarantees for \texttt{Generalized AsyncSGD} in terms of both communication rounds and wall-clock time, explicitly quantifying the trade-off between update frequency (speed) and gradient staleness (error).
	
	\item \textbf{Energy-aware modeling and analysis.}
	We provide the first comprehensive energy analysis of \texttt{Generalized AsyncSGD} under heterogeneous data, hardware, and network conditions, incorporating stochastic timing and \emph{phase-dependent} energy costs (computation, uplink transmission, and downlink reception).
	
	\item \textbf{Joint optimization of routing and concurrency.}
	We formulate and solve multi-objective optimization problems that jointly optimize routing probabilities and concurrency, enabling principled navigation of the trade-offs among accuracy, wall-clock time, and energy consumption.
\end{itemize}

\subsection{Related Work}\label{sec:related_works}

\subsubsection[Asynchronous FL]{Asynchronous \gls{FL}} \label{sec:related_works_afl}

Early \gls{FL} research focused on synchronous methods~\cite{wang2020tackling,qu2021feddq,makarenko2022adaptive,mao2022communication,tyurin2022dasha}, which suffer from stragglers and poor scalability in heterogeneous settings~\cite{xie2019asynchronous}.  
These limitations motivated the introduction of asynchronous \gls{FL} algorithms~\cite{xie2019asynchronous,chen2020asynchronous,xu2023asynchronous}, whose wall-clock advantages were formally established under simplified assumptions in~\cite{cohen2021asynchronous,koloskova2022sharper}.

Classical analyses of asynchronous SGD~\cite{agarwal2011distributed,chaturapruek2015asynchronous,feyzmahdavian2016asynchronous,arjevani2020tight,sra2016adadelay,lian2015asynchronous,stich2020error,nguyen2022federated} are typically predicated on worst-case (maximum) delay bounds, rendering them overly pessimistic and highly sensitive to outliers. Moreover, these works often adopt simplified computation models that obscure the variability and stochasticity inherent in edge computing environments. For instance, while \texttt{FedBuff}~\cite{nguyen2022federated} introduces \gls{CS}-side buffering to stabilize convergence, its theoretical guarantees remain tethered to the maximum delay. Furthermore, its analysis assumes that task completion order is uniformly random, an assumption that breaks down in realistic environments where clients exhibit heterogeneous service speeds.

To mitigate this pessimism, subsequent studies shifted toward average-delay-based analysis~\cite{cohen2021asynchronous,aviv2021learning}. \cite{cohen2021asynchronous} achieves tighter guarantees by selectively discarding stale gradients, but this comes at the cost of doubled communication overhead and is restricted to homogeneous data settings. Similarly, \cite{aviv2021learning} derives bounds dependent on the average delay, yet the analysis relies on variance terms that can scale with the maximum delay in worst-case scenarios, and likewise assumes data homogeneity. In a different approach, delay-adaptive learning-rate methods~\cite{mishchenko2022asynchronous} successfully remove the dependence on maximum delay but fail to converge to the exact optimum in heterogeneous data settings. By scaling step sizes based on staleness, these methods disproportionately favor faster clients, effectively distorting the global objective function and converging to a biased solution.

To address bias in heterogeneous data settings, \cite{koloskova2022sharper} proposed \texttt{AsyncSGD}, which ensures unbiased updates through client-side queues and uniform sampling, with convergence guarantees depending on the average delay.  
\cite{islamov2024asgrad} generalized this approach and proposed improvements via random shuffling.  
However, both analyses assume bounded delays,
and are thus less relevant to systems with (more realistic) unbounded queueing dynamics.

\texttt{Generalized AsyncSGD}~\cite{leconte2024queueing} further generalized this framework by allowing non-uniform routing while preserving unbiasedness through appropriately scaled step sizes. \cite{leconte2024queueing} derived convergence guarantees depending on the average delay under unbounded processing times; however, the analysis relies on asymptotic approximations and provides neither closed-form delay expressions nor guidance on selecting routing probabilities and concurrency levels.

Building on this line of work, our preliminary study~\cite{alahyane2025optimizing} leveraged queueing theory to derive explicit closed-form expressions for the average delay and proposed a gradient-based optimization of the routing probabilities.

\subsubsection{Wall-Clock Time Analysis}

Some analyses of asynchronous \gls{FL} derive convergence guarantees explicitly in terms of time~\cite{tyurin2023optimal,tyurin_shadowheart_2024,maranjyan_mindflayer_2024,maranjyan_ringmaster_2025}, but they typically impose partial synchronization mechanisms (e.g., deadline-based updates), which bring them back into the framework of synchronous \gls{FL}. Additionally, these approaches tend to bias learning toward fast clients, discard near-complete computations, and rely on deterministic or simplified computation-time models.

In contrast, most analyses of asynchronous \gls{FL} cited in \Cref{sec:related_works_afl} rely on round-based convergence, thereby obscuring true wall-clock performance. Studies
such as~\cite{dutta2018slow} and our preliminary work~\cite{alahyane2025optimizing} emphasize a fundamental trade-off: optimizing solely for per-round error can significantly degrade wall-clock convergence time, which was the primary motivation for introducing asynchronous \gls{FL}. For instance, \cite{dutta2018slow} analyzes the \texttt{FedBuff} algorithm and shows that an optimal buffer size exists to balance this trade-off; however, the analysis assumes homogeneous data distributions and processing speeds.
In~\cite{alahyane2025optimizing} we built upon the stochastic framework of \cite{leconte2024queueing}
and demonstrated that, under \texttt{Generalized AsyncSGD}, minimizing per-round error reduces staleness by effectively pacing the system to the slowest clients, albeit at the expense of throughput due to underutilized parallelism.

\subsubsection{Energy Analysis} 
Energy-efficient \gls{FL} has been predominantly studied in synchronous settings~\cite{luo2021cost,zhou2022joint,yang2020energy}. In the asynchronous domain, energy optimization is typically addressed through client selection strategies~\cite{sezgin2025energy,gouissem2024comprehensive} constrained by heuristic fairness measures. For instance, \cite{yang2023asynchronous} mandates a minimum selection probability for each client, while \cite{chu2025resource} enforces a minimum number of data contributions. However, such constraints are often ad hoc and difficult to tune. Furthermore, these approaches generally rely on deterministic timing models, which fail to capture the stochastic dynamics of realistic edge systems.

\subsubsection{Synthesis and Limitations of Prior Work}

While our foundational formulation in~\cite{alahyane2025optimizing} derived closed-form delay expressions and identified the staleness–throughput trade-off, it remained limited in scope.  
Specifically, it omitted communication phases, assumed instantaneous \gls{CS} processing, and did not treat concurrency as a controllable optimization parameter.  
More broadly, existing asynchronous \gls{FL} literature has yet to provide a unified stochastic framework that simultaneously accounts for timing dynamics and energy consumption under realistic system heterogeneity.

To overcome these limitations, this extended manuscript introduces a comprehensive framework that jointly models stochastic computation, communication delays, \gls{CS} processing times, and energy costs.  
By formulating routing and concurrency as explicit optimization variables and solving the resulting problems via gradient-based methods, we establish a principled approach to navigating multi-objective trade-offs among wall-clock convergence speed, statistical accuracy, and energy efficiency in heterogeneous edge environments.

\subsection{Notations}\label{sec:notations}

$\bZ, \bN, \bN_{> 0}, \bR, \bR_{\ge 0}, \bR_{> 0}$ denote the sets of integers, non-negative integers, positive integers, real numbers, non-negative real numbers, and positive real numbers, respectively. Let $| \cdot |$ denote the $\ell_1$-norm, and let $\one{\,\cdot\,}$ be the indicator function. For each $n, m \in \bN_{> 0}$, define $\cX_{n, m} = \{x \in \bN^{n}: |x| = m\}$ as the set of $n$-dimensional vectors with non-negative integer entries whose $\ell_1$-norm equals $m$.
For every $n \in \bN_{> 0}$, let $\cP_n = \{p \in \bR^n: 0 < p_i < 1 \text{ for } i \in \{1, 2, \ldots, n\} \text{ and } |p| = 1\}$.

\section{System model and learning mechanism}
Let us first describe our system model for asynchronous \gls{FL}.

\subsection{Problem Setup}
Consider an asynchronous \gls{FL} system composed of a \gls{CS} and 
$n$ clients indexed by $\{1,2,\ldots,n\}$.  
A global model is trained 
collaboratively by solving the optimization problem
\[
\min_{w \in \mathbb{R}^d} f(w)
= \frac{1}{n} \sum_{i=1}^n f_i(w),
\]
where
$
f_i: w \mapsto
\mathbb{E}_{(x,y)\sim \mathcal{D}_i}
\bigl[\ell_i(\mathrm{NN}(x,w),\,y)\bigr]
$
is the objective function of client~$i$.
Here, $w \in \mathbb{R}^d$ is the parameter vector
of the global model 
(e.g., a deep neural network), $d$ the number of trainable parameters, 
$\mathrm{NN}(x,w)$ the model’s prediction function, $\ell_i$ the local loss function of client~$i$, and $\mathcal{D}_i$ its local data distribution.  
Each client~$i$ approximates its gradient using the deterministic mapping $g_i(w, \zeta_i)$, where the stochasticity arises solely from the sampling of $\zeta_i \sim \mathcal{D}_i$. This computation is termed a \emph{task}.

\subsection{Algorithm}
This paper focuses on \texttt{Generalized AsyncSGD}~\cite{leconte2024queueing}, an extension of \texttt{AsyncSGD}~\cite[Algorithm~2]{koloskova2022sharper} that incorporates a routing mechanism to address heterogeneous resources and data distributions.
The \texttt{Generalized AsyncSGD} procedure is detailed in Algorithms~\ref{alg:CS} (\gls{CS}) and~\ref{alg:client} (Client~$i$).

We begin by examining the \gls{CS} perspective (Algorithm~\ref{alg:CS}). The \gls{CS} initializes the global model parameters $\param_0$ and dispatches a batch of $m$ initial tasks (i.e., requests for gradient evaluation). This initialization is performed by selecting, for each of the $m$ tasks, a recipient client uniformly at random from the population. As we demonstrate later, the system's stationary dynamics are robust to this initial configuration.
Note that because $m$ may exceed $n$, or simply due to the independent random selection, a single client may receive multiple concurrent tasks. Upon receipt, clients immediately commence gradient computation.

The optimization process proceeds in discrete \emph{rounds}, indexed by $k \in \{0,1,\ldots,K\}$. We define round~$k$ as the time interval between the $k$-th and $(k+1)$-th update events of the global parameters~$\param$. The procedure unfolds as follows (Line~4):
When a client $C_k$ completes a task and returns a gradient estimate $g_{C_k}(w_{I_k}, \zeta_{C_k})$ (Line~5; where $I_k$ denotes the round index of the model used for computation), the \gls{CS} immediately updates the global parameters (Line~6).
Following the update, the \gls{CS} dispatches the new model $\param_{k+1}$ to a client $A_{k+1}$ selected independently according to the probabilities $\mathbb{P}(A_{k+1} = i) = p_i$ (Lines~7-8).

Turning to the client perspective (Algorithm~\ref{alg:client}), incoming tasks are processed in a \gls{FIFO} manner. If the client is idle upon receiving a model from the \gls{CS}, computation begins immediately; otherwise, the task is queued locally until the client becomes available. Upon completion, the result is transmitted back to the \gls{CS}.

Finally, observe that the stepsize in the update rule (Line~6 of Algorithm~\ref{alg:CS}) is scaled by the inverse routing probability $p_{C_k}^{-1}$ to correct for the bias introduced by non-uniform routing.

As we demonstrate in the following sections, the algorithm's performance is governed by two key system hyperparameters:
(i) the \textbf{routing probability} vector~$p$, which controls the distribution of computational load; and
(ii) the \textbf{concurrency} level~$m$ \cite{koloskova2022sharper}, defined as the constant number of tasks circulating in the system (either queued or under computation).

% \texttt{Generalized AsyncSGD} recovers
We recover the standard \texttt{AsyncSGD} baseline~\cite[Algorithm~2]{koloskova2022sharper} when $p$ is uniform (i.e., $p_i = 1/n$) and the concurrency matches the network size ($m = n$).

\begin{algorithm}[tbp]
	\caption{\texttt{Generalized AsyncSGD (CS)}}
	\label{alg:CS}
	\begin{algorithmic}[1]
		\STATE {\bf Input:} Numbers $K$, $n$, and $m$ of rounds, clients, and tasks; routing $p$; learning rate $\eta$
		\STATE Initialize parameters $\param_0$ randomly
		\STATE Dispatch $m$ instances of $\param_0$ to selected clients
		\FOR{$k = 0, \dots, K$ \label{line:for}}
		\STATE \gls{CS} receives stochastic gradient $g_{C_k}(\param_{I_k},\zeta_{C_k})$ from a client $C_k$ \label{line:CS1}
		\STATE Update $\param_{k+1} \leftarrow \param_k - \frac{\eta}{n p_{C_k}} g_{C_k}(\param_{I_k},\zeta_{C_k})$ \label{line:CS2}
		\STATE Sample a new client $A_{k+1}$ according to $p$ \label{line:CS3}
		\STATE Send model parameter $\param_{k+1}$ to client $A_{k+1}$
		\ENDFOR
	\end{algorithmic}
\end{algorithm}

\begin{algorithm}[tbp]
	\caption{\texttt{Generalized~AsyncSGD~(Client~$i$)}}
	\label{alg:client}
	\begin{algorithmic}[1]
		\STATE {\bf Input:} Queue of received parameters, local dataset
		\IF{Queue is not empty}
		\STATE Take the received parameter $\param$ from the queue using a FIFO policy
		\STATE Sample a mini-batch $\zeta_i$ from the local dataset $\mathcal{D}_i$
		\STATE Compute the gradient estimate $g_{i}(\param,\zeta_{i})$
		\STATE Send the gradient to the \gls{CS}
		\STATE Repeat
		\ENDIF
	\end{algorithmic}
\end{algorithm}

\subsection{Ensuring Fairness via Queueing}

The queueing mechanism in \texttt{Generalized AsyncSGD} is critical for ensuring consistent participation from all clients, regardless of their processing speeds. This inclusivity is vital in heterogeneous (non-IID) settings to prevent the model from becoming biased toward the data distributions of faster clients.

Consider, for example, a system containing an extremely slow client.  
In standard asynchronous approaches that restrict task assignment to idle clients in order to avoid queueing, such as~\cite{mishchenko2022asynchronous}, where $m=n$ and each completed update is immediately sent back to the originating client, the straggler may process only a single task while faster clients complete thousands of updates.
This behavior skews the model toward the data distributions of faster clients.
In contrast, \texttt{AsyncSGD} and \texttt{Generalized AsyncSGD} assign tasks probabilistically, regardless of client state.  
As a result, if a client is extremely slow, tasks may accumulate in its local queue, potentially to the point where it holds all $m$ circulating tasks. This effectively throttles the system, forcing the algorithm to wait until the straggler contributes an update.  
Although this reduces throughput, it ensures that the slow client’s data is incorporated, thereby preserving the statistical unbiasedness of the learned model.

\subsection{Relative Delay}

A key drawback of asynchronous \gls{FL} is the staleness of model parameters arising from updates computed on outdated global models.  
We quantify this effect using the \emph{relative delay}, a central quantity in the convergence analysis of asynchronous SGD.  
In this work, we derive the main performance metrics, namely, the number of communication rounds, wall-clock time, and energy required to reach an $\epsilon$-approximate stationary point (i.e., an average squared gradient norm at most~$\epsilon$), and we show that they depend explicitly on the average relative delay.  
Consequently, characterizing this delay is essential for evaluating and optimizing system performance.

For each round~$k \in \{1,\ldots,K\}$ and client~$i \in \{1,\ldots,n\}$, the relative delay $D_{i,k}$ is defined as the number of model updates performed by the \gls{CS} between (i) the time a task is assigned to client~$i$ in round~$k$, and (ii) the time the resulting gradient is applied by the \gls{CS}.
For example, if client~$i$ receives model parameters $w_k$ and exactly one gradient from another client is applied while it computes its update (so that the model advances to $w_{k+1}$) then $D_{i,k}=1$.  
If no task is assigned to client~$i$ in round~$k$ (i.e., $A_k \neq i$), we set $D_{i,k}=0$.

\subsection{Data Model} \label{data_model}
Consistently with the decentralized learning literature for non-convex settings \cite{nguyen2022federated, mishchenko2022asynchronous, koloskova2022sharper,islamov2024asgrad}, we assume:
\begin{description}
	\item[\textbf{A1}] \textbf{Lower Boundedness:} 
	There exists a scalar $f^* \in \mathbb{R}$ such that for all $w \in \mathbb{R}^d$, $f(w) \ge f^*$.
	
	\item[\textbf{A2}] \textbf{$L$-Smoothness:} 
	Each $f_i$ is continuously differentiable and $L$-smooth. Specifically, there exists a constant $L > 0$ such that for all $i \in \{1, \dots, n\}$ and all $u, v \in \mathbb{R}^d$:
	\[
	\|\nabla f_i(u) - \nabla f_i(v)\| \le L \|u - v\|.
	\]
	
	\item[\textbf{A3}] \textbf{Unbiased Stochastic Gradients with Bounded Variance:} 
	There exists a constant $\sigma \ge 0$ such that for all $i \in \{1, \dots, n\}$ and $w \in \mathbb{R}^d$:
	\begin{align*}
		&\mathbb{E}_{\zeta_i \sim \mathcal{D}_i}\bigl[g_i(w,\zeta_i)\bigr] = \nabla f_i(w), \\
		&\mathbb{E}_{\zeta_i \sim \mathcal{D}_i}\bigl[\|g_i(w,\zeta_i) - \nabla f_i(w)\|^2\bigr] \le \sigma^2.
	\end{align*}
	
	\item[\textbf{A4}] \textbf{Bounded Gradient Dissimilarity:} 
	There exists a constant $M \ge 0$ such that for all clients $i \in \{1, \dots, n\}$ and parameters $w \in \mathbb{R}^d$, $\|\nabla f_i(w) - \nabla f(w)\|^2 \le M^2$.
	
	\item[\textbf{A5}] \textbf{Bounded Gradients:} 
	There exists a constant $G \ge 0$ such that for all $w \in \mathbb{R}^d$ and $i \in \{1, \dots, n\}$, $\|\nabla f_i(w)\| \le G$.
\end{description}
Our analysis requires to have a bounded gradient \textbf{A5} but, in contrast to \cite{mishchenko2022asynchronous}, the stochastic gradients can be unbounded.  
In practice, gradient clipping, commonly used for Byzantine robustness, ensures bounded update norms and constrains the constant~$G$~\cite{islamov2024asgrad}.

\subsection{Computation and Communication Model} \label{sec:model_description}
We generalize the frameworks of \cite{alahyane2025optimizing,leconte2024queueing,koloskova2022sharper} to explicitly account for communication delays, resulting in the comprehensive queueing network depicted in \Cref{fig:model_communication}. 
In realistic edge environments, deterministic models fail to capture the stochastic nature of system resources: computation times fluctuate due to thermal throttling, \gls{DVFS}, while communication links suffer from fading, congestion, and interference. 
To perform a rigorous analysis of routing and concurrency that reflects this inherent variability while remaining mathematically tractable, we model these components as stochastic processes. Within this framework, we derive closed-form performance metrics that rigorously capture these system dynamics.

Following prior work \cite{mitliagkas2016asynchrony,lee2017speeding,dutta2016short,dutta2018slow,hannah2017more,leconte2024queueing,alahyane2025optimizing}, we assume that the computation times for successive tasks (i.e., stochastic gradient evaluations) at client~$i$ are \gls{iid} exponential random variables with rate $\muc_i > 0$. For communication, we assume the uplink transmission times (client~$i$ sending gradients to the \gls{CS}) and downlink transmission times (client~$i$ receiving parameters from the \gls{CS}) are \gls{iid} exponentially distributed with rates $\muu_i > 0$ and $\mud_i > 0$, respectively.

Although our theoretical analysis relies on the assumption of exponentially distributed computation and communication times, we verify numerically in \Cref{num:optimize-time} and \Cref{num:optimize-round} of the supplementary materials that other distributions actually yield similar performance.

From a network perspective, this formulation effectively maps the system to a stochastic queueing network: computation at each client is represented as a single-server \gls{FIFO} queue, while communication delays in both the uplink and downlink directions are modeled as infinite-server (IS) queues.
Initially, we neglect the processing time required for the \gls{CS} to integrate updates, as its computational resources typically far exceed those of the clients; however, \Cref{sec:cs_buffer_model} discusses how our analysis extends when this assumption is relaxed.

\begin{figure}[htbp]
	\centering
	\includegraphics[width=0.6\textwidth]{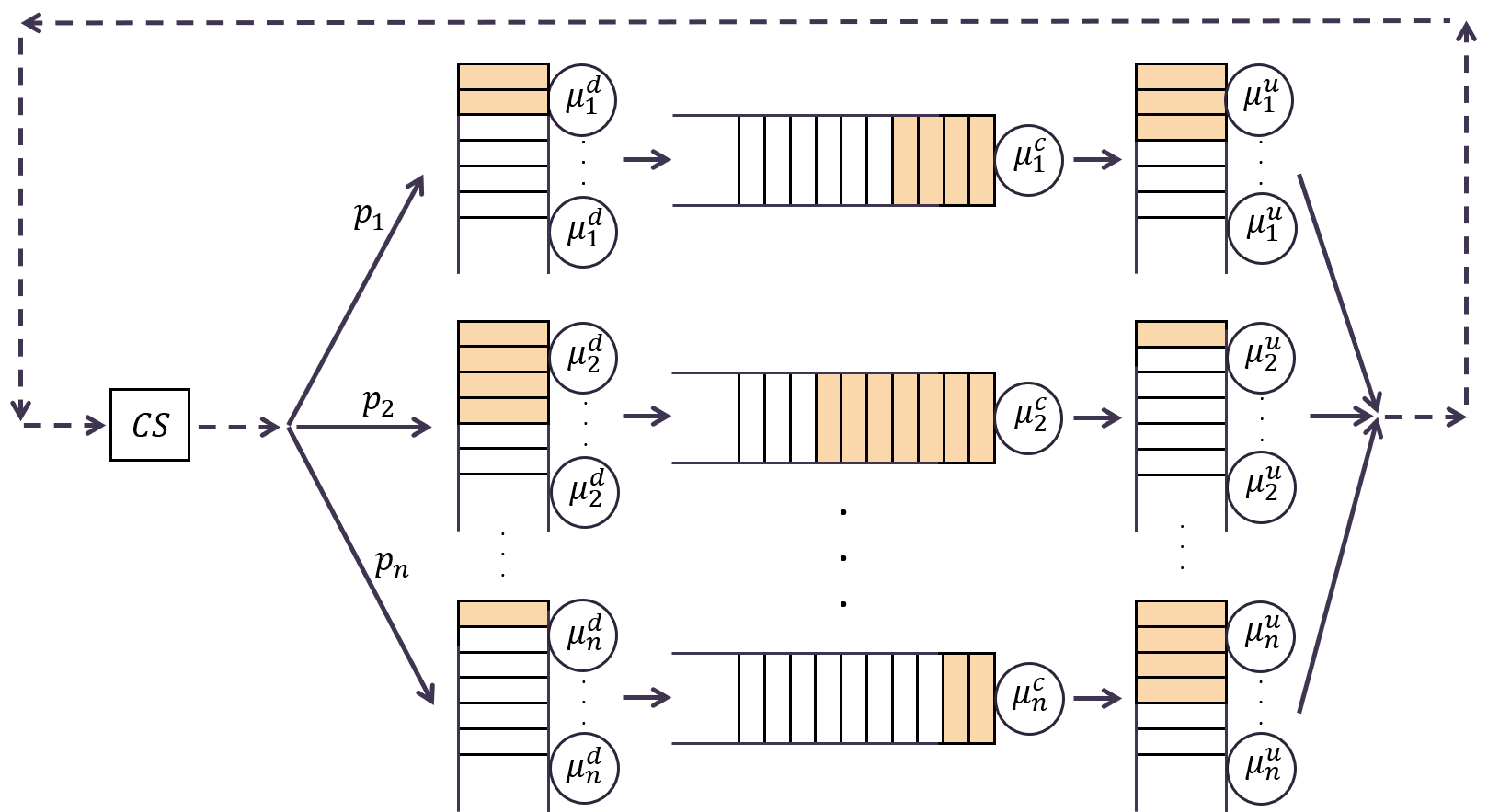}
	\caption{Queueing model with \textbf{computation} times and \textbf{communication} delays.}
	\label{fig:model_communication}
\end{figure}

\section{Stationary Analysis}
As detailed in Sections~\ref{sec:round_complexity}, \ref{sec:time_complexity}, and \ref{sec:energy_complexity}, the average relative delay is the key quantity governing all subsequent performance metrics, including round complexity, wall-clock convergence time, and energy consumption. Accordingly, this section focuses on analyzing the system’s steady-state behavior to explicitly characterize the relative delay.

We begin in Section~\ref{sec:stochastic_network} by modeling the system dynamics as a stochastic queueing network. Section~\ref{sec:stationary_dist} then derives the stationary distribution of this model. Building on these results, Section~\ref{sec:delay} provides closed-form expressions for the average relative delay as well as its gradient, which is required for the gradient-based optimization strategies proposed later.

\subsection{Stochastic Queueing Network} \label{sec:stochastic_network}
For each $i \in \{1,\ldots,n\}$, the downlink communication from the \gls{CS} to client~$i$ is modeled by an infinite-server queue~$d_i$, the computation of gradient at client~$i$ by a single-server queue~$c_i$, and the uplink communication from client~$i$ back to the \gls{CS} by an infinite-server queue~$u_i$.

At any time $t \ge 0$, the network state is represented by the right-continuous (càdlàg) random vector:
\[
\xi(t)
= \bigl(
\xid_i(t),\, 
\xic_i(t),\,
\xiu_i(t),\;
i \in \{1,\ldots,n\}
\bigr)
\in \mathcal{X}_{3n,m},
\]
where each component records the number of tasks present at the corresponding server at time~$t$.
Recall that $\mathcal{X}_{3n, m}$ is the set of $3n$-dimensional vectors with non-negative integer entries summing to~$m$.

Because relative delay is measured at parameter-update instants, we define, for each $k \in \{0,\ldots,K\}$, the embedded state vector $X_k$ as the system state \textit{observed immediately after a parameter update and prior to the dispatch of the next task}.  
Consequently, $X_k$ belongs to the reduced state space $\mathcal{X}_{3n,m-1}$ and is given as follows: for each $i \in \{1,\ldots,n\}$,
\begin{align} \label{def:X}
	\begin{cases}
		\Xd_{i,k} = \xid_i(T_k) - \mathbf{1}\{A_k = i\}, \\[0.5ex]
		\Xc_{i,k} = \xic_i(T_k), \\[0.5ex]
		\Xu_{i,k} = \xiu_i(T_k).
	\end{cases}
\end{align}
Here, $\{T_k\}_{k = 0}^K$ denotes
the sequence of service completion times
at the uplink servers $\{\mathrm{u}_i\}_{i = 1}^n$,
with $T_0 = 0$.
Each time $T_k$ marks the start of round~$k$, which has a duration of $T_{k+1} - T_k$. Furthermore, $A_k$ denotes the index of the client selected to receive the new task at the beginning of round~$k$, with $\mathbb{P}(A_k = i) = p_i$, and $C_k$ denotes the client whose uplink transmission completes at time $T_{k+1}$, marking the end of round~$k$.

Throughout this analysis, we let $\mathbb{P}$ and $\mathbb{E}$ denote the stationary probability measure and expectation of the system, respectively.
To analyze the inter-parameter update times, we also introduce the Palm probability measure $\mathbb{P}^0$ and its expectation $\mathbb{E}^0$, associated with the point process $\{T_k\}_{k \in \bN}$. Under $\mathbb{P}^0$, the time origin $t=0$ is conditioned to be an epoch of service completion at the set of uplink servers $\{\mathrm{u}_i\}_{i = 1}^n$.

\paragraph*{Relative Delay}
For each $k \in \{0, \ldots, K\}$, $D_{i,k}$ is defined as the number of service completions occurring at all uplink nodes $\{u_j\}_{j=1}^n$ during the sojourn of a task assigned to client~$i$ in round~$k$.
This sojourn encompasses the entire cycle: dispatch to server~$d_i$, local processing at client~$i$, and transmission via server~$u_i$.
If no task is assigned to client~$i$ in round~$k$ ($A_k \neq i$), we define $D_{i,k} = 0$.

In the remainder of the paper, we assume the system operates in steady state. 
This assumption is justified for sufficiently large $K$, as the distribution of $\xi(t)$ converges exponentially fast to its stationary distribution \cite{lorek2011speed}. 
Consequently, we can drop the time index, e.g., we can write $\bEp[D_i]$ for $\bEp[D_{i,k}]$.

\subsection{Stationary Distributions} \label{sec:stationary_dist}

Under the assumptions of \Cref{sec:model_description}, our system evolves as a stochastic \textbf{closed Jackson network}~\cite{jackson1957networks,gordon1967closed}. The following proposition characterizes its stationary behavior, establishing the foundation for our subsequent delay analysis.

\begin{proposition} \label{prop:jackson}
	In the setting of \Cref{sec:model_description}, the processes $(\xi(t))_{t \ge 0}$ and $(X_k)_{k \in \bN}$ are irreducible, positive recurrent Markov chains with unique stationary distributions $\pi_{n, m}$ and $\pi_{n, m-1}$, respectively.
	For any $\um \in \bN_{>0}$, the distribution $\pi_{n, \um}$ is given by the following product form:
	for any state $x \in \mathcal{X}_{3n,\um}$
	\begin{align} \label{eq:pi}
		\pi_{n, \um}(x)
		= \frac{1}{Z_{n, \um}} \prod_{i = 1}^n \left( \frac{p_i}{\muc_i} \right)^{\xc_i}
		\frac{1}{\xd_i!}\left( \frac{p_i}{\mud_i} \right)^{\xd_i}
		\frac{1}{\xu_i!}\left( \frac{p_i}{\muu_i} \right)^{\xu_i},
	\end{align}
	where $Z_{n,\um}$ is the normalizing constant.
\end{proposition}

\begin{proof}
	See \Cref{proof_prop:jackson} of the supplementary materials.
\end{proof}

The relative delay expressions derived in the next section depend on the stationary distribution~$\pi_{n,\um}$ primarily through the associated normalization constants~$Z_{n,\um}$.  
Direct computation of these constants is infeasible for large $n$ and $m$ due to combinatorial growth.  
To address this, in Proposition~\ref{prop:buzen} of the supplementary materials, we adapt \textit{Buzen}'s recursive algorithm~\cite{buzen1973computational}, which enables their computation in $\mathcal{O}(nm^2)$ time and $\mathcal{O}(m)$ memory, where $n$ is the number of clients and $m$ the concurrency level (i.e., the number of tasks).  
This computational efficiency is crucial, as it allows the exact evaluation of both the expected relative delay and its gradient presented in the next section.

\subsection{Delay and Gradient Computation}\label{sec:delay}

In this section, we derive closed-form expressions for the mean relative delay and its gradient with respect to the routing vector~$p$, explicitly characterizing their dependence on~$p$, the system concurrency~$m$, and the heterogeneous service rates~$\mu = \left( \mud_i, \muc_i, \muu_i \right)_{i=1}^n$. These results enable performance sensitivity analysis and facilitate gradient-based optimization. Crucially, they allow us to efficiently determine both the optimal routing strategy and the optimal concurrency level~$m$ to minimize any performance criterion defined as a functional of the mean relative delay.

\begin{theorem} \label{theo:little_law}
	In the model of \Cref{sec:model_description}, the following identities hold for each $i,j \in \{1,\ldots,n\}$:
	\begin{align}
		\label{eq:D}
		\bEp[D_i]
		&= \, \sum_{s \in \{\mathrm{c},\mathrm{d},\mathrm{u}\}} \mathbb{E}\!\left[ X_i^s \right], 
		\\[1ex]
		\label{eq:gradD}
		\frac{\partial}{\partial p_j}\,\bEp[D_i]
		&= \sum_{s \in \{\mathrm{c},\mathrm{d},\mathrm{u}\}} \sum_{r \in \{\mathrm{c},\mathrm{d},\mathrm{u}\}}
		\frac{1}{p_j}\,
		\operatorname{Cov}\!\left(
		X_i^s,\;
		X_j^r
		\right).
	\end{align}
	
	Moreover, the following closed-form expressions hold:
	\begin{align}
		\sum_{s \in \{\mathrm{c},\mathrm{d},\mathrm{u}\}} \mathbb{E}[X_i^s]
		&= \beta_{i,1} + \gamma_i \frac{Z_{n,m-2}}{Z_{n,m-1}}, \label{eq:EX}
		\\[1ex]
		\sum_{s,r \in \{\mathrm{c},\mathrm{d},\mathrm{u}\}} \mathbb{E}[X_i^s X_j^r]
		&= \alpha_{i,j} + \beta_{i,2}\gamma_j + \beta_{j,2}\gamma_i + \psi_{i,j},
		\label{eq:EXX}
	\end{align}
	where the coefficients are defined as:
	\begin{align*}
		\gamma_i &= p_i \left( \frac{1}{\mud_i} + \frac{1}{\muu_i} \right), \qquad
		\beta_{i,\ell} = \sum_{k=1}^{m-\ell} \left( \frac{p_i}{\muc_i} \right)^k \frac{Z_{n,m-\ell-k}}{Z_{n,m-1}},
		\\[1ex]
		\alpha_{i,j} &=
		\begin{cases}
			\displaystyle \sum_{k=1}^{m-1} (2k-1) \left(\frac{p_i}{\muc_i}\right)^{k} \frac{Z_{n,m-1-k}}{Z_{n,m-1}} & \text{if } i=j, \\[2ex]
			\displaystyle \sum_{\substack{k, \ell = 1 \\ k + \ell \le m-1}}^{m-2}
			\left( \frac{p_i}{\muc_i} \right)^{k}
			\left( \frac{p_j}{\muc_j} \right)^{\ell}
			\frac{Z_{n, m-1-k-\ell}}{Z_{n, m-1}} & \text{if } i\neq j,
		\end{cases}
		\\[1ex]
		\psi_{i,j} &= \frac{\gamma_i}{Z_{n,m-1}} \left( \gamma_j Z_{n,m-3} + \mathbf{1}_{\{i=j\}} Z_{n,m-2} \right).
	\end{align*}
	and the constants $Z_{n,\um}$ for $\um \in \{0,1,\ldots,m-1\}$ are computed using the recursion of Proposition~\ref{prop:buzen} in \Cref{sec:buzen} of the supplementary material.
\end{theorem}

\begin{proof}
	The proof is detailed in Section~\ref{proof_theo:little_law} of the supplementary material. Intuitively, Equation~\eqref{eq:D} represents a Little's Law~\cite{john2011little} analogue: it equates the expected total delay cost defined at assignment ($\bEp[D_i]$) to the cost accumulated incrementally during service ($\sum_{s \in \{\mathrm{c},\mathrm{d},\mathrm{u}\}} \mathbb{E}\!\left[ X_i^s \right]$).
\end{proof}

A naive evaluation of $\bEp[D_i]$ is intractable because a task's relative delay may depend on an arbitrarily large number of future rounds.
By expressing the delay and its gradient via the recursive constants $Z_{n,k}$ (computable in $\mathcal{O}(nm^2)$ time and $\mathcal{O}(m)$ memory), \Cref{theo:little_law} ensures that evaluating delays remains efficient and scalable even for large networks.

\paragraph*{Analysis}
Equation~\eqref{eq:D} implies the simple identity
\begin{equation} \label{eq:uni_sum}
	\sum_{i=1}^n \bEp[D_i]
	= \sum_{i=1}^n \sum_{s \in \{\mathrm{c},\mathrm{d},\mathrm{u}\}} \mathbb{E}\!\left[X_i^s\right]
	= m-1.
\end{equation}
This yields two immediate but counterintuitive consequences:  
(i) the total mean relative delay depends only on $n$ and $m$, while $p$ and $\mu$ determine only how this delay is distributed across clients;  
(ii) thus, for fixed $n$ and $m$, decreasing the delay of a given client~$i$ (e.g., by reducing $p_i$ or increasing $\muc_i$, $\mud_i$, or $\muu_i$) necessarily comes at the expense of at least one other client experiencing an increased delay.

Moreover, for fixed routing and server speeds, combining \eqref{eq:D} with the fact that $\mathbb{E}[X_i^s]$ is non-decreasing in $m$ for each $s \in \{\mathrm{c},\mathrm{u},\mathrm{d}\}$ \cite[Lemma~2]{suri1985monotonicity} implies that $\bEp[D_i]$ is itself a non-decreasing function of the number of tasks~$m$.

In addition to permitting exact derivation, \eqref{eq:D}--\eqref{eq:gradD} allow for the estimation of the expected relative delay $\bEp[D_i]$ and its gradient $\nabla_p \bEp[D_i]$ via Monte Carlo simulations.

\paragraph*{Gradient Descent}
As detailed in Sections~\ref{sec:round_complexity}, \ref{sec:time_complexity}, and~\ref{sec:energy_complexity}, all considered system performance metrics depend explicitly on the relative delay. Consequently, Equation~\eqref{eq:gradD} enables the application of gradient descent to optimize the routing vector~$p$ for these objectives.

\section{Round Complexity} \label{sec:round_complexity}

This section analyzes the convergence behavior of \texttt{Generalized AsyncSGD} in heterogeneous environments.
First, Section~\ref{sec:complexity} leverages the performance bounds established in~\cite{leconte2024queueing} to characterize the round complexity as a function of the average relative delay.
Next, Section~\ref{sec:discussion} examines the practical implications and limitations of relying on round-based convergence rates as a primary performance metric.

\subsection[Number of rounds to achieve ε-accuracy]{Number of rounds to achieve $\epsilon$-accuracy}\label{sec:complexity}
The following theorem bounds the number of updates~$K$ required for \texttt{Generalized AsyncSGD} to reach an $\epsilon$-approximate stationary point.
Crucially, this result explicitly captures the dependence on routing probabilities $p$, concurrency $m$, and client speed heterogeneity. As we explain in \Cref{proof_th:opt_steps} of the supplementary material, this result builds on a previous result from Leconte et al.~\cite{leconte2024queueing}.

\begin{theorem} \label{th:opt_steps}
	Under Assumptions~\textbf{A1}--\textbf{A5}, there exists $\epsilon_0 > 0$ such that for any target accuracy $\epsilon \in (0, \epsilon_0]$ and any learning rate satisfying $\eta \le \eta_{\max}(p,m)$, where
	\begin{equation} \label{eq:eta_max}
			\eta_{\max}(p,m) = \min \Biggl\{ \frac{n^2}{8L \sum_{i=1}^n p_i^{-1}}, \, \frac{n^2 \epsilon}{2 L B \sum_{i=1}^n p_i^{-1}},
			\frac{n \sqrt{\epsilon}}{2 L} \left(C (m-1) \sum_{i=1}^n \frac{\bEp[D_i]}{p_i^2}\right)^{-1/2} \Biggr\},
	\end{equation}
	the expected gradient norm satisfies
	\[
	\frac{1}{K} \sum_{t=0}^{K-1} \bEp\!\left[\|\nabla f(w_t)\|^2\right] \le \epsilon
	\]
	whenever $K \ge K_\epsilon(p,m)$, where:
	\begin{equation} \label{eq:K_epsilon_bounded}
		K_{\epsilon}(p,m) = \frac{24L\Delta}{n\epsilon} \Biggl[ \left( 4 + \frac{B}{\epsilon} \right) \sum_{i=1}^n \frac{1}{n p_i}
		+ \left(\frac{C (m-1)}{\epsilon} \sum_{i=1}^n \frac{\bEp[D_i]}{p_i^2}\right)^{1/2} \Biggr],
	\end{equation}
	with constants $\Delta = f(w_0) - f^*$, $B = 6(\sigma^2 + 2M^2)$, and $C = 6(\sigma^2 + G^2)$. The expected steady-state relative delays $\bEp[D_i]$  can be computed explicitly using Theorem~\ref{theo:little_law}.
\end{theorem}

\begin{proof}
	See \Cref{proof_th:opt_steps} of the supplementary materials.
\end{proof}

If Assumption~\textbf{A5} is relaxed, an alternative expression for $K_{\epsilon}(p,m)$, derived under the model of \Cref{sec:model_description}, is provided in Section~\ref{sec:K_epsilon_unbounded} of the supplementary materials.

The convergence rate in \eqref{eq:K_epsilon_bounded} comprises two distinct terms. 
The first term captures the impact of the routing strategy and data heterogeneity. Under uniform routing ($p^\text{uni} = (\frac1n, \ldots, \frac1n)$) and homogeneous data ($M=0$), this term reduces to the classical SGD convergence rate $\mathcal{O}\!\left(\frac{L\Delta}{\epsilon}\left(1+\frac{\sigma^{2}}{\epsilon}\right)\right)$~\cite{ghadimi2013stochastic,arjevani2023lower}. 
The second term accounts for the delays induced by asynchrony, thus quantifying the additional iterations
required to attain accuracy comparable to synchronous SGD, despite gradient staleness.
For uniform routing, this penalty reduces to $\mathcal{O}\!\left(\frac{L\Delta \sqrt{\sigma^2+G^2}}{\epsilon \sqrt{\epsilon}} (m-1) \right)$, a value driven purely by the system concurrency~$m$ rather than the magnitude of computation or communication speeds.
In general, the explicit dependence on $G$ and $\sigma$ confirms that high variance or large gradient magnitudes worsen the impact of asynchrony, as ``old'' gradients deviate more significantly from the current true direction.

Numerical experiments on image classification tasks under full concurrency ($m=n$), presented in Supplementary Section~\ref{num:optimize-round}, demonstrate that the round-optimized routing $p^{\ast K}$ (i.e., the routing minimizing $K_\epsilon$) achieves a substantial reduction in communication rounds relative to baseline methods.

\subsection{Do Fewer Rounds Mean Faster Training?}\label{sec:discussion}

With the closed-form expression for relative delay established (see \eqref{eq:D}), we now examine the dependence of the round complexity $K_{\epsilon}$ on the routing vector $p$ and concurrency level $m$. The first term in \eqref{eq:K_epsilon_bounded} is minimized under uniform routing probabilities, reflecting the benefit of unbiased client participation. The second term quantifies the staleness overhead via $\bEp[D_i]$ and is non-decreasing with respect to~$m$. Consequently, for any fixed routing strategy, increasing concurrency exacerbates the performance penalty incurred by stale gradients.

From a strict round-complexity perspective, this implies that the optimal configuration is uniform routing with $m=1$. Indeed, setting $m=1$ eliminates the delay term entirely, effectively recovering serial SGD. However, this trivial solution severely bottlenecks throughput and negates the principal advantages of asynchrony: straggler mitigation and parallel efficiency.

Even when concurrency is fixed at $m>1$ (e.g., full concurrency $m=n$, as in~\cite{koloskova2022sharper}), strictly minimizing the round complexity $K_\epsilon$ effectively synchronizes the system to the pace of the slowest clients.  
As illustrated in Supplementary Section~\ref{num:optimize-round} and in \Cref{num:learning_results}, reducing staleness requires the optimizer to heavily reallocate routing probability toward slower clients, causing the update frequency to drop dramatically (from 41 updates per time unit under uniform routing to only 2.4 in our experiment of Supplementary Section~\ref{num:optimize-round}).  
Although this conservative strategy reduces the total number of communication rounds $K_\epsilon$, the resulting loss in update frequency outweighs this benefit when performance is measured in wall-clock time.  
Consequently, the round-optimized routing can be outperformed by a simple uniform baseline in terms of real-time convergence speed.

Since the delay penalty scales with the gradient bound $G$, aggressive gradient clipping could theoretically mitigate this staleness. However, such a strategy would simultaneously impede learning progress and ultimately decelerate convergence.

This analysis illustrates a fundamental limitation of the standard rounds-based convergence metric widely used in the \gls{FL} literature: \textbf{it ignores performance in \textit{wall-clock~time}}, which is the primary motivation for adopting asynchronous methods. For this reason, the next section introduces a time-aware metric that more accurately reflects the practical efficiency of asynchronous \gls{FL} systems.

\section{Clock-Time Complexity} \label{sec:time_complexity}

To account for random system delays, we evaluate performance using the \textbf{expected wall-clock time to reach $\epsilon$-accuracy}. Unlike round-based metrics, this measure captures the actual physical duration of the training process, directly addressing the core motivation for asynchronous learning.

\subsection[Time to achieve ε-accuracy]{Time to achieve $\epsilon$-accuracy}  
Let $\tau_{\epsilon}$ denote the random total wall-clock time required to execute the $K_\epsilon(p,m)$ rounds necessary to guarantee $\epsilon$-accuracy. The following proposition establishes a closed-form expression for its expectation, $\bEp[\tau_{\epsilon}]$, within the stochastic network model of \Cref{sec:model_description}.

\begin{proposition} \label{prop:time-eps}
	Under Assumptions~\textbf{A1}--\textbf{A5}, there exists $\epsilon_0 > 0$ such that for any target accuracy $\epsilon \in (0, \epsilon_0]$ and any learning rate satisfying $\eta \le \eta_{\max}(p,m)$, the expected wall-clock time required to reach $\epsilon$-accuracy is given by:
	\begin{align} \label{eq:tau}
		\bEp[\tau_{\epsilon}] = \frac{K_{\epsilon}(p,m)}{\lambda(p,m)},
		\qquad \text{where:}
	\end{align}
	$K_{\epsilon}(p,m)$ denotes the number of rounds required to achieve $\epsilon$-accuracy, given in \Cref{eq:K_epsilon_bounded} of Theorem~\ref{th:opt_steps}. \\
	$\lambda(p,m)$ is the expected number of rounds completed per unit of wall-clock time, given by
	\begin{align}
		\label{eq:throughput}
		\lambda(p,m)
		&= \sum_{i=1}^n \muu_i\, \mathbb{E}[\xiu_i]
		= \frac{Z_{n,m-1}}{Z_{n,m}}, \\[0.5ex]
		\label{eq:grad_throughput}
		\frac{\partial}{\partial p_j}\,\lambda(p,m)
		&= \frac{1}{p_j}\,\lambda(p,m)
		\sum_{s \in \{\mathrm{c},\mathrm{d},\mathrm{u}\}} \mathbb{E}[\,X_j^s - \xi_j^s\,],
	\end{align}
	where $Z_{n,m}$ and $Z_{n,m-1}$ are the normalization constants defined in Proposition~\ref{prop:jackson}, 
	$\xi \sim \pi_{n,m}$, and $X \sim \pi_{n,m-1}$.
\end{proposition}

\begin{proof}
	See \Cref{proof_prop:time-eps} of the supplementary materials.
\end{proof}

\subsection{Discussion}

To illustrate the behavior of the proposed time-aware metric, we consider a simple two-client system under two scenarios:
(i) \textbf{homogeneous resources}, where both clients have identical computation and communication rates
($\muc_{i}=\muu_{i}=\mud_{i}=1$ for $i \in \{1,2\}$); and
(ii) \textbf{heterogeneous resources}, where Client~2 is three times faster in both computation and communication
($\muc_{2}=\muu_{2}=\mud_{2}=3$), while all other parameters are held constant.
Figure~\ref{fig:optimal_m} shows the expected wall-clock time $\bEp[\tau_{\epsilon}]$ as a function of the concurrency level~($m$) and the routing probability of Client~1 ($p_1 = 1 - p_2$).

\begin{figure}[htbp]
	\centering
	\includegraphics[width=0.8\textwidth]{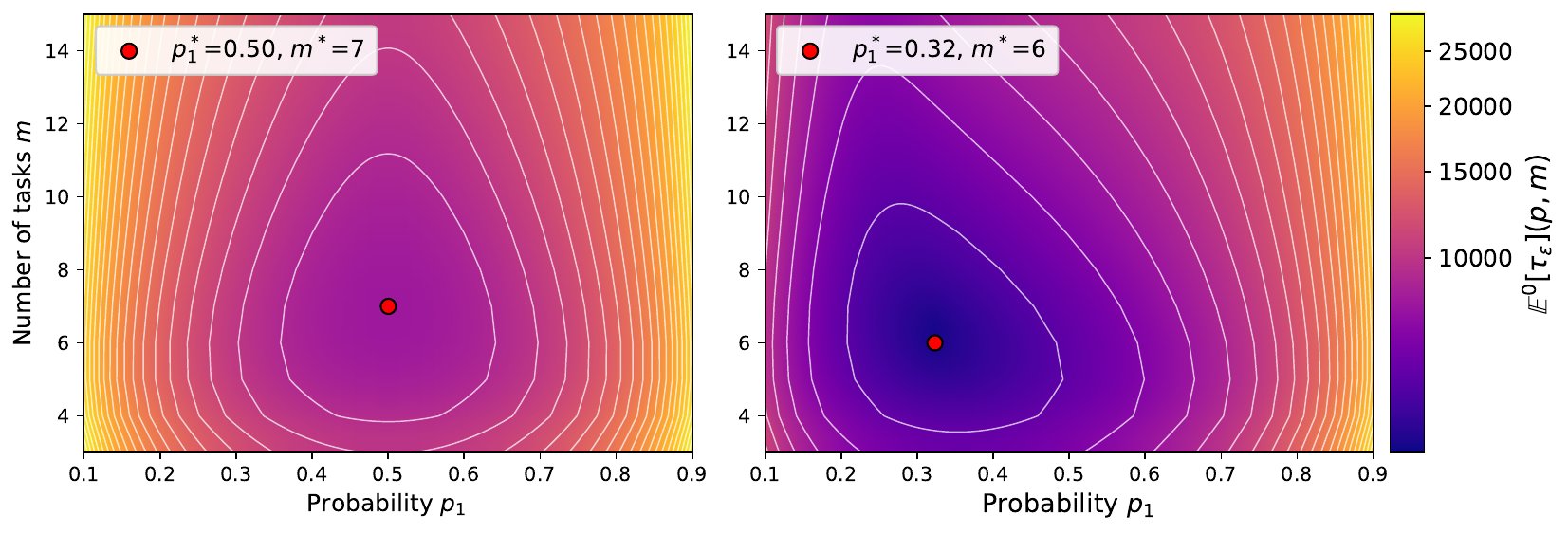}
	\caption{Impact of concurrency $m$ and routing probability $p_1$ on $\bEp[\tau_{\epsilon}]$. Left: Homogeneous resources. Right: Heterogeneous resources, where Client~2 is three times faster than Client~1 ($\mu^s_2 = 3\mu^s_1$ for $s \in \{c,u,d\}$). Constants: $\Delta=L=\sigma=1$, $M=5$, and $G=14$.}
	\label{fig:optimal_m}
\end{figure}

Unlike the round complexity $K_\epsilon$, the wall-clock time $\bEp[\tau_{\epsilon}]$ is generally \emph{not} monotonic in $m$. 
Both scenarios demonstrate the existence of an optimal concurrency level $m^* \ge 1$ that minimizes $\bEp[\tau_{\epsilon}]$. 
Intuitively, when $m < m^*$, the system is underutilized, limiting throughput; conversely, when $m > m^*$, the throughput keeps increasing, but the resulting gradient staleness degrades the convergence rate, outweighing the benefits of parallelization.

Moreover, in the heterogeneous case, the optimized routing favors the faster client, but to a lesser extent than a strategy driven purely by throughput maximization.

Consequently, $\bEp[\tau_{\epsilon}]$ captures the fundamental \textbf{trade-off} between two competing objectives: (i) reducing $K_{\epsilon}$ to improve update quality, and (ii) increasing the system throughput $\lambda$ to increase update frequency. 
These objectives are inherently conflicting, as improving one typically degrades the other.

By jointly accounting for both effects, $\bEp[\tau_{\epsilon}]$ provides a principled performance criterion that balances staleness control with convergence speed, making it a robust metric for optimizing asynchronous \gls{FL} systems.

\subsection{Numerical Results} \label{num:optimize-time}

In this section, we validate our theoretical findings by demonstrating that jointly optimizing the routing vector $p$ and the concurrency level $m$ significantly improves wall-clock convergence speed. We compare our proposed method against three baselines, resulting in the following four strategies:

\begin{enumerate}
	\item \textbf{Time-Optimized \texttt{Generalized AsyncSGD} (Proposed):} This strategy employs the optimal parameters $(p^{\ast \tau}, m^{\ast \tau})$ derived to minimize the expected time $\bEp[\tau_{\epsilon}]$ to achieve $\epsilon$-accuracy, as characterized in Proposition~\ref{prop:time-eps}.
	
	\item \textbf{Standard Baseline (\texttt{AsyncSGD})} \cite[Algorithm 2]{koloskova2022sharper}: Represents the conventional approach using full concurrency ($m=n$) and uniform routing ($p^{\text{uni}}$).
	
	\item \textbf{Round-Optimized \texttt{Generalized AsyncSGD}} \cite{leconte2024queueing,alahyane2025optimizing}: Uses the routing vector $p^{\ast K}$ that minimizes round complexity $K_\epsilon$ while maintaining full concurrency ($m=n$). This serves as a representative for methods in the literature that focus solely on optimizing round-based convergence bounds.
	
	\item \textbf{Max-Throughput \texttt{Generalized AsyncSGD}:} A strategy that maximizes the system update frequency ($p^{\ast \lambda}$) with full concurrency ($m=n$).
\end{enumerate}
We include the final baseline specifically to highlight a critical trade-off: simply maximizing the number of updates per second (throughput) can be detrimental to convergence stability and final accuracy.

\subsubsection{Experimental Setup} \label{num:exp_scenario}
We simulate a heterogeneous network of $n = 100$ clients divided into five clusters (Types A--E), spanning high-performance workstations to resource-constrained devices. To model a latency-critical edge \gls{FL} environment, we skew the population toward stragglers; specifically, Type D constitutes the largest cluster (40\%), while high-performance nodes (Type E) make up only 10\%, with the remaining 50\% comprising mid-range devices. The specific service rates for computation ($\mu^{\text{c}}$), uplink ($\mu^{\text{u}}$), and downlink ($\mu^{\text{d}}$) are detailed in Table~\ref{tab:client_profiles}.

\begin{table}[htbp]
	\centering
	\caption{Client clusters and service rates (Rate $\mu$ in tasks/sec).}
	\label{tab:client_profiles}
	\begin{tabular}{llcccc}
		\toprule
		\textbf{Type} & \textbf{Description} & $\mu^{\text{c}}$ & $\mu^{\text{u}}$ & $\mu^{\text{d}}$ & \textbf{Count} \\
		\midrule
		A & Fast compute, slow network & 10.0 & 2.0 & 2.5 & 15 \\
		B & Slow compute, fast network & 0.3 & 9.0 & 10.0 & 15 \\
		C & Balanced & 5.0 & 6.0 & 7.0 & 20 \\
		D & Straggler & 0.15 & 0.1 & 0.12 & 40 \\
		E & Super Client & 12.0 & 10.0 & 11.0 & 10 \\
		\bottomrule
	\end{tabular}
\end{table}

We evaluate performance on the EMNIST~\cite{cohen2017emnist} dataset under two distinct distribution scenarios:
\begin{itemize}
	\item \textbf{Homogeneous (IID):} Data is distributed identically across clients, with each client holding an equal number of samples from every class.
	\item \textbf{Heterogeneous (Non-IID):} We simulate feature and label heterogeneity using a Dirichlet distribution. For each class~$k$, the proportion of samples allocated to client~$j$ is drawn from $q_k \sim \text{Dir}_n(\alpha)$, where $\alpha = 0.2$ is the concentration parameter, following~\cite{yurochkin2019bayesian,li2021federatedlearningnoniiddata}.
\end{itemize}
Additional experiments on CIFAR-100~\cite{Krizhevsky2009LearningML} are provided in Section~\ref{num:time_c100} of the supplementary materials.

\subsubsection{Optimization Strategy and Results} \label{num:opt_results}
We address the minimization of $\bEp[\tau_{\epsilon}]$ via a sequential optimization approach due to the discrete nature of $m$. Iterating from $m=2$, we optimize the routing vector $p$ for each fixed $m$ using gradient descent. The search terminates when the objective function $\bEp[\tau_{\epsilon}]$ stops decreasing, signaling that the optimal pair $(m^{\ast \tau}, p^{\ast \tau})$ has been surpassed.
To accelerate convergence, we use a \textit{warm-start strategy}: the optimization for level $m+1$ is initialized using the optimal vector $p^\ast$ found at level $m$. The optimization of $p$ is performed using the Adam optimizer~\cite{kingma2014adam}, with gradients of $\bEp[\tau_{\epsilon}]$ computed in closed form via Theorem~\ref{theo:little_law} and Proposition~\ref{prop:time-eps} (see \Cref{num:opt_time_strat} of the supplementary materials for further details).

Similarly, we compute the max-throughput vector $p^{\ast \lambda}$ and the round-optimized vector $p^{\ast K}$ via Adam, utilizing the gradient expressions for $\lambda$ and $\bEp[D_i]$ provided in Equations~\eqref{eq:grad_throughput} and~\eqref{eq:gradD}.
The constants $\sigma$, $M$, and $G$, introduced in Section~\ref{data_model}, are estimated empirically from the training data. We set the target gradient norm bound to $\epsilon = 1$.

The optimized routing probabilities and corresponding staleness metrics are detailed in \Cref{tab:opt_results}. We specifically analyze the quantity $\bEp[D_i(p,m)] \, p_i^{-2}$ as a \emph{staleness impact factor}; this term represents each client's contribution to the staleness term in $K_\epsilon$, helping to identify specific clusters that may degrade algorithm stability.

\begin{table}[htbp]
	\centering
	\caption{Comparison of optimized routing probabilities and staleness impact factors across different client clusters (Types A--E).}
	\label{tab:opt_results}
	\begin{tabular}{lccccccc}
		\toprule
		& \multicolumn{3}{c}{
			\begin{tabular}{@{}c@{}}
				\textbf{Routing probabilities} \\
				$p \times 100$
			\end{tabular}
		}
		& \multicolumn{4}{c}{
			\begin{tabular}{@{}c@{}}
				\textbf{Staleness Impact Factor} \\
				$\bEp[D_i(p,m)] \, p_i^{-2} \times 10^{-2}$
			\end{tabular}
		} \\
		\cmidrule(lr){2-4} \cmidrule(lr){5-8}
		\textbf{Type}
		& $p^{\ast \tau}$
		& $p^{\ast \lambda}$
		& $p^{\ast K}$
		& $(p^{\ast \tau},\,m^\ast)$
		& $(p^{\ast \lambda},\,n)$
		& $(p^{\ast K},\,n)$
		& $(p^{\text{uni}},\,n)$ \\ 
		\midrule
		A & 1.307 & 0.845 & 0.526 & 14.2 & 181.7 & 8.6 & 7.4 \\
		B & 0.514 & 0.011 & 0.627 & 182.0 & 49\,543.7 & 28.0 & 33.9 \\
		C & 1.752 & 1.591 & 0.506 & 5.5 & 65.7 & 4.6 & 3.8 \\
		D & 0.34 & 0.005 & 1.691 & 1\,615.7 & 783\,209.8 & 84.6 & 229.6 \\
		E & 2.405 & 5.514 & 0.496 & 2.1 & 12.4 & 2.5 & 2.0 \\
		\bottomrule
	\end{tabular}%
\end{table}

Regarding concurrency, the optimization yields an optimal level of $m^{\ast \tau} = 91$, strictly less than the total client count $n=100$. This result challenges the standard convention in the literature, which typically assumes full concurrency ($m = n$).

The throughput-optimized routing ($p^{\ast \lambda}$) strongly favors fast clients, particularly the \emph{super clients} (Type~E), while aggressively down-weighting stragglers (Type~D). While this strategy maximizes the overall update frequency, it risks introducing significant bias. As evidenced in \Cref{tab:opt_results}, Types~D and B clients exhibit exploded staleness factors (orders of magnitude larger than other types). This indicates that their updates are not only rare but also extremely stale, which significantly hinders convergence.

In contrast, the round-optimized routing ($p^{\ast K}$) prioritizes stragglers to minimize global staleness, resulting in uniformly low staleness factors across all clusters, albeit
at the expense of lower system throughput.

The proposed routing ($p^{\ast \tau}$) strikes a balance: while it assigns higher probabilities to fast clients compared to the uniform baseline, it ensures stragglers retain a non-negligible selection probability. This approach effectively balances update frequency against staleness, limiting the error induced by delayed updates while maintaining higher throughput than~$p^{\ast K}$.

Quantitatively, maximizing network throughput yields $\lambda(p^{\ast \lambda}, n)=152$ updates/time unit, substantially higher than standard \texttt{AsyncSGD} ($\lambda(p^\text{uni}, n)=7.4$). On the other hand, round-optimized routing focuses on staleness reduction, resulting in the lowest throughput ($\lambda(p^{\ast K}, n)=4.5$). Finally, our proposed time-optimized configuration achieves an effective middle ground ($\lambda(p^{\ast \tau}, m^\ast)=18.7$).

These sharp contrasts set the stage for analyzing the speed-accuracy trade-off: does an $8\times$ increase in throughput actually translate to faster convergence? And does a strategy designed solely to reduce error (staleness) perform efficiently in terms of wall-clock time?

\subsubsection{Learning Performance and Trade-off Analysis} \label{num:learning_results}
We adopt the experimental setup described in Section~\ref{num:exp_scenario} to simulate the training process. To assess robustness against initial transients, we initialize the system out of equilibrium: at $t=0$, the $m$ tasks are assigned uniformly at random to the clients' downlink servers, rather than being sampled from the stationary distribution.
Furthermore, to verify that our results are not artifacts of the exponential assumption, we evaluate performance when \textbf{both computation and communication times} follow three distinct distributions:
\begin{enumerate}
	\item[(i)] \textbf{Exponential:} Note that exponential service times are specifically required for the theoretical derivations.
	\item[(ii)] \textbf{Deterministic:} Fixed service times equal to $1/\mu$ (zero variance).
	\item[(iii)] \textbf{Lognormal:} A heavy-tailed distribution with mean $1/\mu$. We set the underlying normal variance to $\sigma_N^2=1$, reflecting the high variability of real-world edge environments. This choice imposes a fixed coefficient of variation across all clients, isolating the impact of service rates.
\end{enumerate}

Models are trained using standard multi-class cross-entropy loss, and performance is reported on an unseen, label-balanced test set. Learning rates are tuned via grid search. Implementation details are provided in Supplementary Section~\ref{exp:details}.

\begin{table}[htbp]
	\centering
	\caption{Percentage time reduction relative to baselines for a target test accuracy of $0.6$ ($0.75$ in parentheses).}
	\label{tab:time_target_acc}
	
	\begin{tabular}{llccc}
		\toprule
		\multicolumn{2}{c}{\textbf{Scenario}} & \multicolumn{3}{c}{\textbf{Time Reduction (\%) vs.}} \\
		\cmidrule(lr){1-2} \cmidrule(lr){3-5}
		\textbf{Dist.} & \textbf{Data} & \textbf{Max-Throughput} & \textbf{Round-Optimized} & \textbf{AsyncSGD} \\
		\midrule
		\multirow{2}{*}{Exp.} 
		& IID     & 67.5 (66.59)  & 62.59 (66.71) & 46.28 (46.88) \\
		& Non-IID & 79.3          & 57.04         & 35.6 (36.56)  \\
		\midrule
		\multirow{2}{*}{LogN.}   
		& IID     & 59.08 (70.46) & 64.35 (66.1)  & 41.77 (46.41) \\
		& Non-IID & 79.17         & 62.44         & 37.0 (42.38)  \\
		\midrule
		\multirow{2}{*}{Det.} 
		& IID     & 52.35 (64.24) & 49.35 (61.38) & 29.84 (38.86) \\
		& Non-IID & N/A           & 58.88         & 31.16 (37.81) \\
		\bottomrule
	\end{tabular}%
\end{table}

\begin{figure*}[t]
	\centering
	\includegraphics[width=\textwidth]{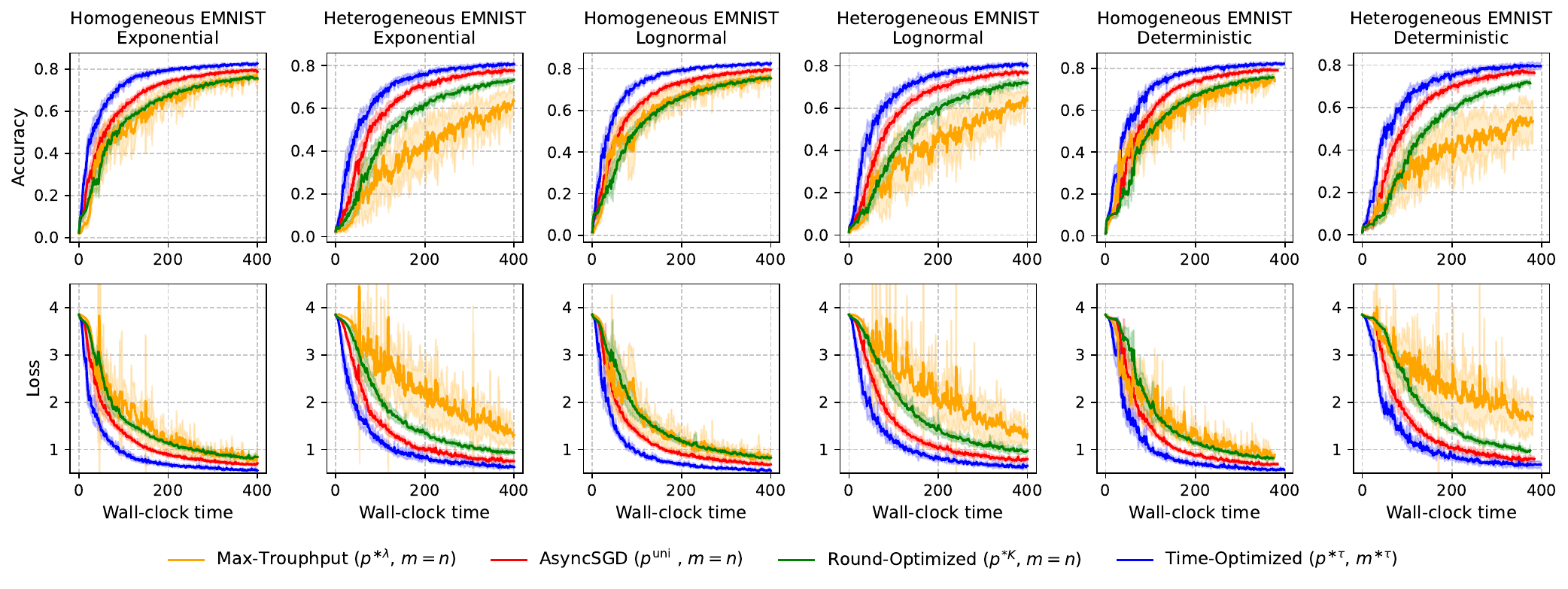}
	\caption{Test set performance at the \gls{CS} for the scenario in \Cref{num:exp_scenario}. The top row displays test accuracy and the bottom row shows loss, both plotted against wall-clock time. Columns correspond to different service time distributions under homogeneous and heterogeneous data settings. Each subplot compares the four strategies. Simulations were repeated 10 times for 400 wall-clock time units. Solid lines indicate means; shaded areas represent standard deviations.}
	\label{fig:time_results}
\end{figure*}

\Cref{fig:time_results} demonstrates that \texttt{Generalized AsyncSGD} equipped with the optimized parameters ($p^{\ast \tau}, m^{\ast \tau}$) consistently outperforms the baseline methods throughout the learning process. This advantage is quantified in Table~\ref{tab:time_target_acc}, which reports the substantial reduction in wall-clock time required to achieve a target accuracy of $0.6$ (values in parentheses denote the reduction for a higher target accuracy of $0.75$, where achievable). Notably, our approach yields significant speedups compared to standard \texttt{AsyncSGD}, as well as the Max-Throughput and Round-Optimized strategies. Furthermore, these gains are consistent across all simulated service time distributions confirming that our method's robustness extends well beyond the theoretical assumption of exponential computation times.

The results in \Cref{fig:time_results} and \Cref{tab:time_target_acc} starkly illustrate the fundamental tension between \textit{update quantity} (frequency) and \textit{update quality} (staleness).
On one extreme, the max-throughput strategy ($p^{\ast \lambda}$) prioritizes quantity. Within the fixed wall-clock window of 400 units, it executes over $60\,000$ parameter updates, more than $10\times$ the volume of our proposed method ($(p^{\ast \tau},m^{\ast \tau})$, $\approx 6\,000$). However, this unbridled speed comes at the cost of quality: in the Non-IID deterministic setting, it yields a final accuracy 60\% lower than the optimized approach. The strategy exhibits high instability (frequent loss spikes) and required a learning rate $20\times$ smaller than other baselines to prevent immediate divergence. This confirms that maximizing update frequency in isolation degrades model quality and squanders computational resources.

On the other extreme, the round-optimized strategy ($p^{\ast K}$) prioritizes quality. By aggressively favoring stragglers to minimize staleness, it ensures high gradient fidelity and stability. However, this focus severely limits the quantity of updates, completing only $1\,800$ in the same time frame. While effective in terms of ``progress per round,'' the excessive duration of each round makes it inefficient in terms of wall-clock time.

The proposed time-optimized strategy ($p^{\ast \tau},m^{\ast \tau}$) effectively bridges this gap. It strikes a critical balance, maintaining sufficient update frequency while bounding staleness enough to ensure stable convergence.
Finally, uniform routing performs acceptably in this specific setup because its probabilities closely align with the time-optimized distribution, though it lacks the targeted acceleration of the proposed approach.

\section{Energy Complexity} \label{sec:energy_complexity}

In many practical scenarios, devices operate under strict energy constraints (e.g., limited battery capacity). Consequently, it is imperative to design learning schemes that minimize energy usage while maintaining convergence guarantees.

\subsection{Energy Model} \label{sec:nrg_model}

To assess the system's energy efficiency, we adopt a state-dependent power consumption model that explicitly captures the distinct hardware characteristics of heterogeneous clients.

Following established models in the literature~\cite{sezgin2025energy,yang2020energy,zhou2022joint,luu2025energy,chu2025resource}, each client $i \in \{1,\ldots,n\}$ is assigned a power profile corresponding to its three active phases:
\begin{enumerate}[(i)]
	\item \textbf{Local Computation ($\Pc_i$):} The power consumed during gradient calculation. Following \gls{DVFS} laws~\cite{mao2017stochastic,luo2020hfel,sezgin2025energy}, power consumption is generally proportional to the cube of the average CPU frequency. In our model, this is expressed as $\Pc_i \propto (\mu_i^{\text{c}})^3$.
	\item \textbf{Uplink Transmission ($\Pu_i$):} The power required to transmit the computed gradients to the \gls{CS} for each assigned task.
	\item \textbf{Downlink Reception ($\Pd_i$):} The power required to download model parameters from the \gls{CS} for each assigned task.
\end{enumerate}
The total energy consumption for a specific task~$k$ at client~$i$, encompassing the full cycle from receiving parameters to sending updates, is given by:
\begin{align}
	E_{i,k} = \Pd_i T_{i,k}^{\text{d}} + \Pc_i T_{i,k}^{\text{c}} + \Pu_i T_{i,k}^{\text{u}},
\end{align}
where $T_{i,k}^{\text{c}}$, $T_{i,k}^{\text{u}}$, and $T_{i,k}^{\text{d}}$ denote the random time durations for computation, uplink, and downlink transmission, respectively, associated with the $k$-th~task received by client~$i$.

We assume the \gls{CS} is connected to the fixed power grid and is not battery-constrained. Therefore, we focus on the energy consumption of the edge devices. However, this model is easily generalizable: if the server's energy is a factor, its transmission power can be integrated into the downlink cost $\Pd_i$.

Accordingly, the instantaneous power cost of the system at time~$t$ can be written as:
\begin{align}
	P(t) = \sum_{i=1}^n \Pc_i \,\mathbf{1}\{\xic_i(t) > 0\}
	+ \Pu_i \,\xiu_i(t) 
	+ \Pd_i \,\xid_i(t).
\end{align}

\subsection[Energy to achieve ε-accuracy]{Energy to Achieve $\epsilon$-Accuracy}

To evaluate the algorithmic energy efficiency, we analyze the expected total energy required to reach $\epsilon$-accuracy, denoted by $\bEp[E_{\epsilon}]$. The following proposition provides a closed-form characterization of this metric.

\begin{proposition} \label{prop:energy-eps}
	Under Assumptions~\textbf{A1}--\textbf{A5}, there exists $\epsilon_0 > 0$ such that for any target accuracy $\epsilon \in (0, \epsilon_0]$ and any learning rate satisfying $\eta \le \eta_{\max}(p,m)$, the expected energy required to reach $\epsilon$-accuracy is given by:
	\begin{align}
		\bEp[E_{\epsilon}] 
		= K_{\epsilon}(p,m)\, \frac{\mathbb{E}[P(0)]}{\lambda(p,m)} \notag
		= K_{\epsilon}(p,m) \sum_{i=1}^n p_i \mathcal{E}_i,
	\end{align}
	where $\mathcal{E}_i \triangleq \frac{\Pc_i}{\muc_i} + \frac{\Pu_i}{\muu_i} + \frac{\Pd_i}{\mud_i}$ is client~$i$'s average energy cost per task,
	$K_{\epsilon}$ is the round complexity defined in \Cref{eq:K_epsilon_bounded} of Theorem~\ref{th:opt_steps}, and $\lambda$ is the system throughput defined in~\eqref{eq:throughput}.
\end{proposition}

\begin{proof}
	See Section~\ref{proof_prop:energy-eps} of the supplementary materials.
\end{proof}

The quantity $\sum_{i=1}^n p_i \mathcal{E}_i$ represents the \emph{average energy consumed per round}. 
Importantly, this quantity depends only on the routing vector $p$ and the clients’ hardware characteristics, and is independent of the concurrency level~$m$.

\subsection{Energy-Latency Trade-off} \label{nrg_latency}

Since the round complexity $K_{\epsilon}$ is non-decreasing with respect to $m$ (see Section~\ref{sec:discussion}) while the energy consumed per round is invariant with respect to $m$, the expected total energy $\bEp[E_{\epsilon}]$ is strictly minimized when $m=1$.
Consequently, an energy-optimal strategy requires setting $m=1$ and selecting the routing vector $p_E^\ast$ that minimizes the following product (rather than the uniform routing optimal for $K_{\epsilon}$ at $m=1$):
\begin{align} \label{opt_E}
	\min_{p \in \mathcal{P}_n} \quad
	\underbrace{\left[ \frac{24L\Delta}{n^2 \epsilon} \left(4 + \frac{B}{\epsilon}\right) \sum_{j=1}^n \frac{1}{p_j} \right]}_{K_{\epsilon}(p,\, m=1)}
	\;\times\;
	\underbrace{\left[ \sum_{i=1}^n p_i \mathcal{E}_i \right]}_{\text{Energy per round}}.
\end{align}
By applying the Cauchy-Schwarz inequality (see Section~\ref{proof_opt_E_routing} of the supplementary materials), the closed-form solution to this minimization problem is given by:
\begin{align} \label{opt_E_routing}
	p^{\ast E}_i \propto \frac{1}{\sqrt{\mathcal{E}_i}}, \quad i \in \{1,\ldots,n\}.
\end{align}
Evaluating the objective in \eqref{opt_E} at $p^{\ast E}$ yields the minimal average energy consumption:
\begin{align}
	E^\ast = \frac{24L\Delta}{n^2 \epsilon} \left(4 + \frac{B}{\epsilon}\right)  \left( \sum_{i=1}^n \sqrt{\mathcal{E}_i} \right)^2.
\end{align}

However, operating with $m=1$ eliminates parallelism and the benefits of asynchrony, leading to prohibitively large wall-clock training times.  
This exposes a fundamental energy-latency trade-off: increasing the concurrency level~$m$ accelerates training in wall-clock time but increases the total energy cost, as higher staleness induces extra rounds to reach the target accuracy.  
Conversely, minimizing energy favors sequential execution and routing toward low-power devices.

This tension is further amplified by hardware heterogeneity: energy-efficient devices are often the slowest (e.g., IoT sensors compared to GPUs), so routing strategies that minimize energy typically incur the largest latency penalties.

\subsection{Joint Optimization Problem}

Following \cite{yang2023asynchronous,chu2025resource,zhou2022joint,luo2021cost,luo2020hfel}, we address the inherent trade-off between wall-clock training time and total energy consumption while ensuring convergence by formulating a joint optimization problem over the routing probabilities~$p$ and the concurrency level~$m$.  
The two objectives (expected time $\bEp[\tau_{\epsilon}]$ and energy $\bEp[E_{\epsilon}]$) have different units and scales, and are therefore combined using a normalized scalarization.

Specifically, each objective is normalized by its theoretical minimum, corresponding to optimizing that metric alone. The resulting problem is
\begin{align} \label{eq:joint_opt}
	\min_{m \in \mathbb{N}_{>0},\, p \in \mathcal{P}_n} \quad 
	\rho \,\frac{\bEp[E_{\epsilon}(p,m)]}{E^\ast}
	\;+\; 
	(1-\rho) \,\frac{\bEp[\tau_{\epsilon}(p,m)]}{\tau^\ast},
\end{align}
where $\tau^\ast=\bEp[\tau_{\epsilon}(p_\tau^\ast,m_\tau^\ast)]$ is the minimum achievable training time (see Section~\ref{num:optimize-time}), $E^\ast=\bEp[E_{\epsilon}(p_E^\ast,1)]$ is the minimum achievable energy consumption, and $\rho\in[0,1]$ is a user-defined trade-off parameter.

The weights $\rho$ and $1-\rho$ represent the relative importance of energy and time, respectively.  
Equivalently, the ratio $\rho/(1-\rho)$ defines a marginal rate of substitution, quantifying how much relative training speed one is willing to sacrifice for a given relative reduction in energy consumption.

\subsection{Numerical Results}

\subsubsection{Experimental Setup}
We address the joint optimization problem~\eqref{eq:joint_opt} by extending the experimental framework of \Cref{num:exp_scenario} with an explicit power-consumption model.  
Client service rates follow \Cref{tab:client_profiles}, while local computation power is modeled using a cubic \gls{DVFS} scaling law,
$P_{\text{comp}} = \kappa (\mu^c)^3$, where $\kappa$ is a hardware-specific energy coefficient.

\begin{table}[htbp]
	\centering
	\caption{Energy coefficients and power profiles (normalized units). Computation power follows $P_{\text{comp}} = \kappa (\mu^c)^3$, where a lower $\kappa$ indicates higher hardware efficiency.}
	\label{tab:power_coefficients}
	\begin{tabular}{ll c c c c}
		\toprule
		& & \textbf{Coeff.} & \multicolumn{3}{c}{\textbf{Relative Power Profile}} \\
		\cmidrule(l){4-6}
		\textbf{Type} & \textbf{Characteristic Profile} & $\boldsymbol{\kappa}$ & $\mathbf{P_{\text{comp}}}$ & $\mathbf{P_{\text{u}}}$ & $\mathbf{P_{\text{d}}}$ \\
		\midrule
		A & Compute-Efficient \& Network-Limited & 0.08 & 80.0 & 5.0 & 3.0 \\
		B & Compute-Constrained \& Comm-Optimized & 200.0 & 5.4 & 15.0 & 10.0 \\
		C & Balanced Speed \& Power & 0.25 & 31.3 & 4.0 & 3.0 \\
		D & Straggler \& Highly Inefficient & 14400.0 & 48.6 & 0.5 & 0.2 \\
		E & High Throughput \& Power-Intensive & 1.50 & 2592.0 & 50.0 & 40.0 \\
		\bottomrule
	\end{tabular}%
\end{table}

As \Cref{tab:power_coefficients} illustrates, the energy coefficient $\kappa$ captures the extreme heterogeneity across the diverse client profiles in our network. While Type~A's efficient architecture (low $\kappa$) delivers high throughput at moderate power, Type~E requires massive power for only marginal speed gains due to architectural and cooling overhead. Most notably, Type~D draws significant baseline computing power despite abysmal compute speeds. Because neither the fastest nor the slowest clients are strictly energy-optimal, navigating this hardware diversity makes our energy-aware optimization highly non-trivial.

\subsubsection{Optimization Strategy and Results}
We solve the joint optimization problem~\eqref{eq:joint_opt} by identifying the optimal pair $(p^{\ast \rho}, m^{\ast \rho})$ for a given value of $\rho$.  
To handle the discrete nature of the concurrency~$m$, we adopt a sequential optimization strategy: starting from $m=1$, we optimize the routing vector $p$ for each fixed $m$ using gradient descent, with a warm-start initialization from the solution obtained at the previous level.

Figure~\ref{fig:pareto_opt} summarizes the joint optimization: the left panel illustrates the Time–Energy Pareto frontier annotated with $\rho$, while the right panel shows how optimal routing probabilities $p^{\ast \rho}$ and concurrency $m^{\ast \rho}$ evolve as the objective shifts from strictly time-centric ($\rho = 0$) to strictly energy-centric ($\rho = 1$).

\begin{figure}[htbp]
	\centering
	\includegraphics[width=\linewidth]{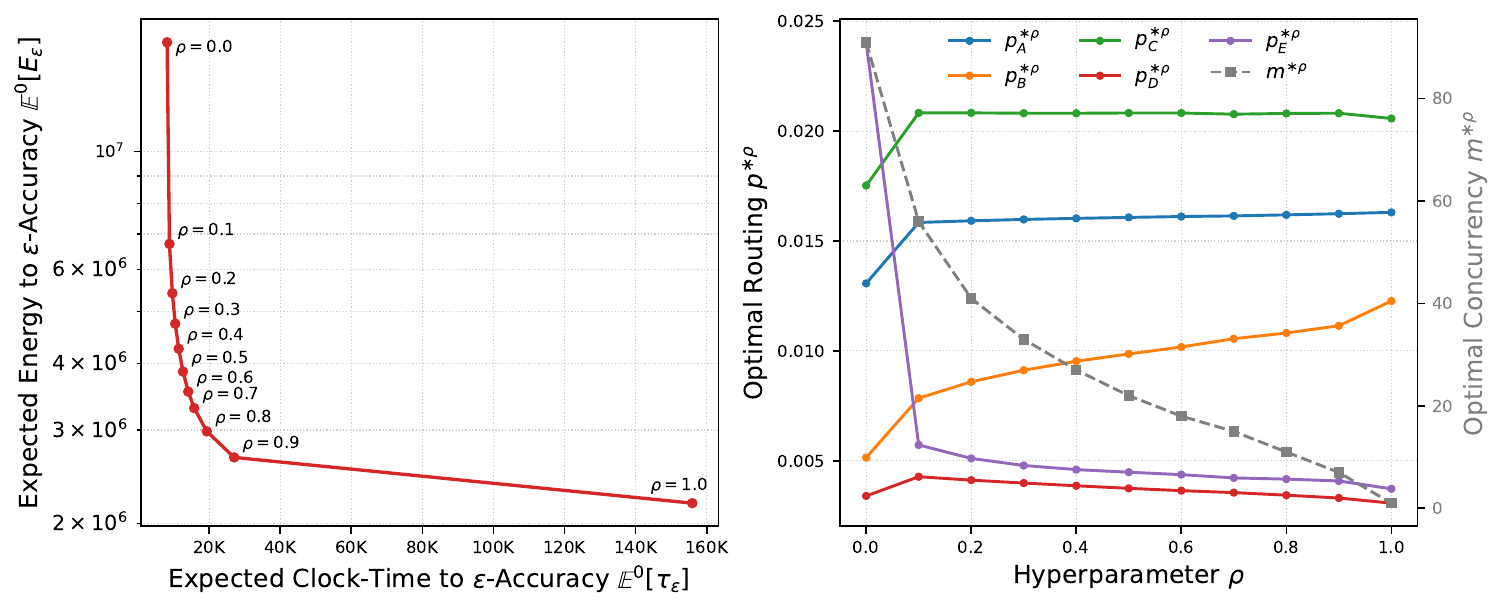}
	\caption{Joint optimization analysis. \textbf{Left:} Time–Energy Pareto frontier annotated with $\rho$. \textbf{Right:} Evolution of optimal routing probabilities $p^{\ast \rho}$ and optimal concurrency $m^{\ast \rho}$ as a function of $\rho$, illustrating the systematic shift away from energy-inefficient Type~E clients.}
	\label{fig:pareto_opt}
\end{figure}

Introducing a slight energy penalty ($\rho = 0.1$) sharply reduces concurrency from $m=91$ to $56$, driving the initial energy drop on the Pareto frontier. As $\rho \to 1$, the system further mitigates power waste by further reducing concurrency, ultimately converging to strictly serial execution ($m=1$).

Simultaneously, the routing dynamics reflect the network's hardware heterogeneity. Time-centric optimization ($\rho = 0$) heavily favors fast Type~E clients. However, increasing $\rho$ rapidly throttles this cluster due to its immense power overhead, shifting weights toward efficient Type~C and Type~A clients. Conversely, the Type~D stragglers are universally penalized across all regimes, lacking both the speed to reduce time and the efficiency to save energy.

\subsubsection{Learning Performance and Trade-off Analysis} \label{num:energy_learning}
To validate these theoretical findings in a realistic learning scenario, we simulate training on KMNIST dataset~\cite{clanuwat2018deep} using identical time and power profiles. For each $\rho$, we apply the corresponding optimal configuration $(p^{\ast \rho}, m^{\ast \rho})$. We consider a heterogeneous data distribution (Dirichlet with $\alpha = 0.2$) and exponential service times. Averaging 10 independent runs to a target test accuracy of 0.75, the empirical results (\Cref{fig:pareto_opt_kmnist}) match our theoretical performance trends. They reveal a non-linear trade-off governed by $\rho$, where a slight relaxation in convergence speed can yield large energy savings.

\begin{figure}[htbp]
	\centering
	\includegraphics[width=0.6\linewidth]{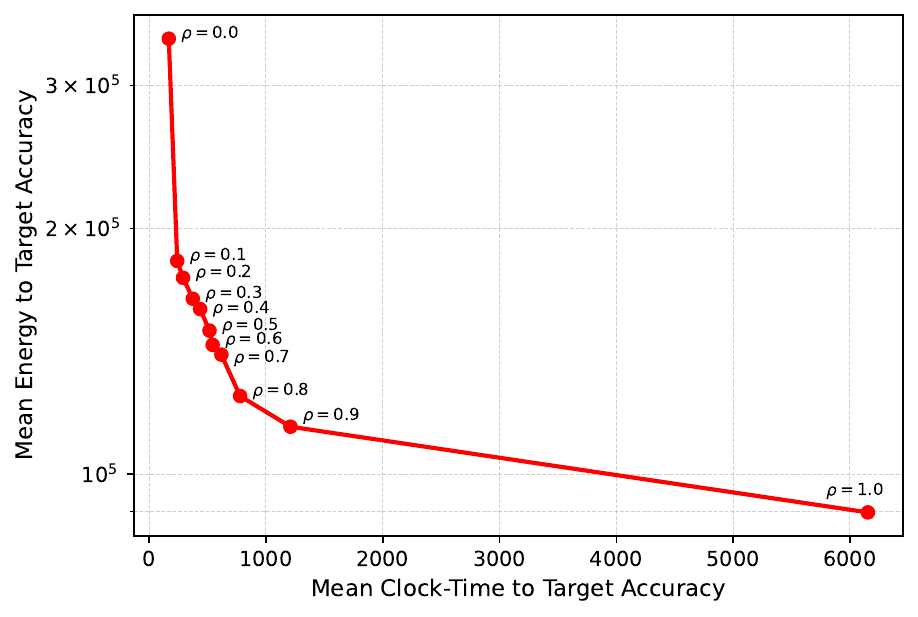}
	\caption{Empirical Time-Energy Pareto frontier for the KMNIST dataset, illustrating the mean clock-time versus mean energy required to reach a target test accuracy of 0.75 across 10 independent runs.}
	\label{fig:pareto_opt_kmnist}
\end{figure}

For further empirical validation, we evaluate our joint optimization strategy $(p^{\ast \rho}, m^{\ast \rho})$ against the \texttt{AsyncSGD} baseline~\cite[Algorithm~2]{koloskova2022sharper} on the KMNIST and EMNIST datasets. As revealed by our prior analysis, we set $\rho=0.1$ to secure substantial energy savings with minimal convergence delay. Simulations follow the setup in Section~\ref{num:learning_results} using power profiles from \Cref{tab:power_coefficients}. Table~\ref{tab:relative_gains} details the relative improvements over the baseline across all evaluated scenarios.
Detailed learning trajectories for the EMNIST dataset, plotted against both clock-time and energy, are provided in Section~\ref{num:optimize-energy} of the supplementary materials.

By penalizing energy-inefficient clients and throttling concurrency to $m=56$, the optimizer consistently cuts total energy consumption by $36\%\text{--}49\%$ across all evaluated network conditions. Crucially, it simultaneously accelerates wall-clock convergence by up to $19\%$ in nearly every scenario. The sole exception is the EMNIST deterministic Non-IID setting, which intelligently trades a marginal $3.15\%$ time increase for a massive $36.19\%$ energy reduction. This confirms that naive \texttt{AsyncSGD} is inherently suboptimal in
both clock-time and energy. Conversely, our joint strategy actively navigates the Pareto frontier to secure large energy savings without sacrificing overall convergence speed.

\begin{table}[htbp]
	\centering
	\caption{Relative time and energy reduction (\%) of $(p^{\ast \rho}, m^{\ast \rho})$ at $\rho=0.1$ compared to \texttt{AsyncSGD} at target accuracies of 0.80 (KMNIST) and 0.70 (EMNIST). Positive values indicate savings; negative denote increases.}
	\label{tab:relative_gains}
	\begin{tabular}{ll cccccc}
		\toprule
		\multirow{2}{*}{\textbf{Dataset}} & \multirow{2}{*}{\textbf{Metric}} & \multicolumn{2}{c}{\textbf{Exponential}} & \multicolumn{2}{c}{\textbf{Lognormal}} & \multicolumn{2}{c}{\textbf{Deterministic}} \\
		\cmidrule(lr){3-4} \cmidrule(lr){5-6} \cmidrule(lr){7-8}
		& & \textbf{IID} & \textbf{Non-IID} & \textbf{IID} & \textbf{Non-IID} & \textbf{IID} & \textbf{Non-IID} \\
		\midrule
		\multirow{2}{*}{\textbf{KMNIST}} 
		& \textbf{Time}   & 6.11  & 10.58 & 18.95 & 8.63  & 15.46 & 0.53 \\
		& \textbf{Energy} & 46.90 & 46.10 & 45.53 & 40.90 & 47.19 & 36.62 \\
		\midrule
		\multirow{2}{*}{\textbf{EMNIST}} 
		& \textbf{Time}   & 15.41 & 3.90  & 6.12  & 6.93  & 7.12  & -3.15 \\
		& \textbf{Energy} & 46.26 & 41.60 & 41.83 & 48.97 & 39.40 & 36.19 \\
		\bottomrule
	\end{tabular}%
\end{table}

\section[Incorporating a CS-Side Buffer]{Incorporating a \gls{CS}-Side Buffer}\label{sec:cs_buffer_model}

While theoretical analyses of asynchronous \gls{FL} often assume instantaneous global updates, practical edge deployments necessitate sequential processing to ensure atomic access to shared parameters. This serialization creates a bottleneck where updates arriving in quick succession must queue at the \gls{CS} before incorporation~\cite{zuo2024asynchronous}. This \gls{CS}-side latency stems several factors:
\begin{itemize}
	\item \textbf{Model Size:} Updating large-scale models (e.g., LLMs) incurs non-negligible memory I/O and vector arithmetic delays, particularly on CPU-based servers \cite{khan2023towards}.
	\item \textbf{Security \& Privacy Overheads:} Robust deployment requires the \gls{CS} to perform computationally intensive tasks on every incoming update (such as decrypting payloads, verifying digital signatures, or screening for poisoning attacks) before the update is safe to apply~\cite{zhao2023privacy,de2025practical,ghinani2025fusefl}.
	\item \textbf{High Concurrency:} In massive IoT networks, the aggregate arrival rate of client updates can temporarily exceed the \gls{CS}'s service capacity, leading to significant congestion and unavoidable queueing delays~\cite{leconte2024queueing,nguyen2022federated}.
\end{itemize}
Consequently, we extend our framework to explicitly model the \gls{CS} as a queueing system where incoming gradients are buffered and processed sequentially.

\subsection{Model Description}\label{sec:model_description2}

We build on the network model introduced in \Cref{sec:model_description} by relaxing the assumption that the \gls{CS} service time is negligible.  
Instead, we model the \gls{CS} as a single-server queue operating under a \gls{FIFO} discipline, where model-update processing times are assumed to be \gls{iid} exponential random variables with rate~$\muCS$.  
Incorporating these dynamics yields the closed queueing network depicted in \Cref{fig:model_cs}.

To derive a closed-form expression for the relative delay, it is necessary to track the number of parameter updates that occur between the time a task is dispatched to a client and the moment its corresponding gradient is applied at the \gls{CS}. This requires identifying the origin of every task waiting in the \gls{CS} queue.
Since a standard queueing representation aggregates all tasks at the \gls{CS} and thus obscures their origin, we reformulate the system as a \textbf{multi-class Jackson network}, illustrated in \Cref{fig:model_cs} using distinct colors.  
Each task belongs to one of $n$ classes, identified by the client to which it is assigned. Specifically, when the \gls{CS} routes a task to client~$i$, it is assigned label~$i$ and retains this class identity throughout the entire cycle until it completes service at the \gls{CS}.

\begin{figure}[htbp]
	\centering
	\includegraphics[width=0.7\textwidth]{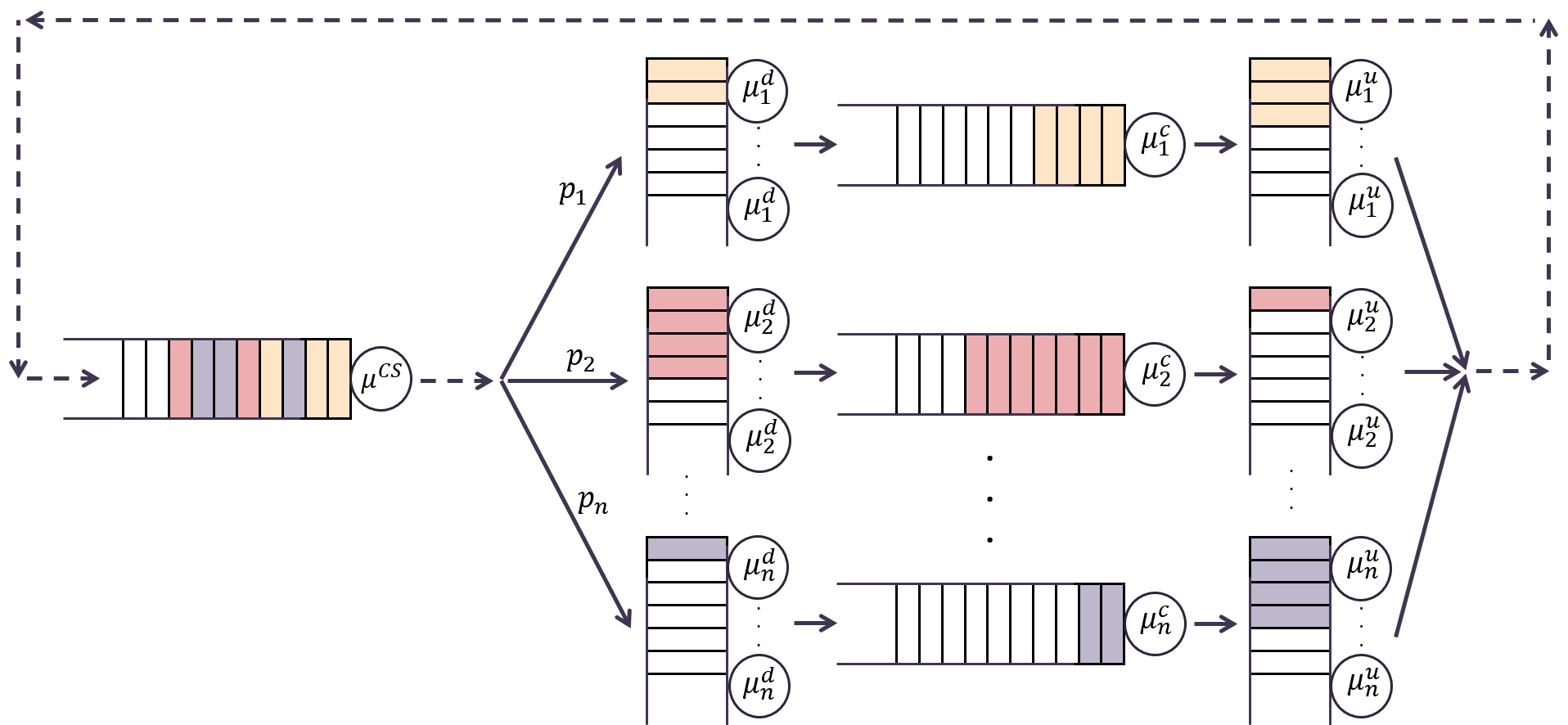}
	\caption{Extended queueing network model incorporating computation, communication, and \gls{CS}-side update latencies. Colors correspond to task classes.}
	\label{fig:model_cs}
\end{figure}

The system dynamics are described by the right-continuous (càdlàg) stochastic process $Y = (Y(t))_{t \ge 0}$ defined on the expanded state space $\mathcal{X}_{4n,m}$.  
At any time $t \ge 0$, the state is given by
\[
Y(t) = \bigl(\YCS_i(t), \,\Yd_i(t), \,\Yc_i(t), \,\Yu_i(t) \; ; \; i \in \{1,\ldots,n\}\bigr),
\]
where $\YCS_i(t)$ denotes the number of class-$i$ tasks currently present at the \gls{CS} (either waiting or in service).  
Similarly, for $s \in \{\mathrm{d},\mathrm{c},\mathrm{u}\}$, $Y^s_i(t)$ denotes the number of tasks at the corresponding server associated with client~$i$.

In this context, a \emph{round} is defined as the time interval between two consecutive service completions at the \gls{CS}, corresponding to successive model updates.  
Let $\{T_k\}_{k \in \mathbb{N}}$ denote the sequence of these completion times, with $T_0 = 0$; the interval $[T_k, T_{k+1})$ thus defines round~$k$.  
As before, $\mathbb{P}$ denotes the stationary probability measure, while $\mathbb{P}^0$ (and $\mathbb{E}^0$) denote the Palm probability (and expectation) associated with $\{T_k\}_{k \in \mathbb{N}}$, conditioning on a service completion at the \gls{CS} at time~$t=0$.

Characterizing the relative delay requires analyzing the system state observed at model-update instants.  
To this end, we define the embedded process
\begin{align*}
	\tilde{Y}_k = \bigl(\tYCS_{i,k}, \,\tYd_{i,k}, \,\tYc_{i,k}, \,\tYu_{i,k} \; ; \; i \in \{1,\ldots,n\}\bigr)
	\in \mathcal{X}_{4n,m-1},
\end{align*}
which represents the network state immediately after a service completion at the \gls{CS}, excluding the task that has just completed service.  
Its components are obtained from the post-jump state $Y(T_k)$ as
\begin{align} \label{def:tilde_Y}
	\begin{cases}
		\tYCS_{i,k} = \YCS_i(T_k), \\[0.5ex]
		\tYd_{i,k} = Y^d_i(T_k) - \mathbf{1}\{A_k = i\}, \\[0.5ex]
		\tYc_{i,k} = Y^c_i(T_k), \\[0.5ex]
		\tYu_{i,k} = Y^u_i(T_k),
	\end{cases}
	\qquad i \in \{1,\ldots,n\},  
\end{align}
where $A_k$ denotes the index of the client selected to receive the task completed at time~$T_k$, with $\mathbb{P}(A_k = i) = p_i$.

\paragraph*{Relative Delay}
In this extended framework, the relative delay $D_{i,k}$ is defined as the number of service completions occurring at the \gls{CS} during the sojourn of a task assigned to client~$i$ in round~$k$.
This sojourn encompasses the entire cycle: dispatch to server~$d_i$, queueing and local processing at client~$i$, transmission via server~$u_i$, and finally, queueing and update application at the \gls{CS}.
If no task is assigned to client~$i$ in round~$k$ ($A_k \neq i$), we set $D_{i,k} = 0$.

\subsection{Stationary Analysis}\label{sec:stationary_analysis}

Although the processes $(Y(t))_{t \ge 0}$ and $(\tilde{Y}_k)_{k \in \mathbb{N}}$ track task histories via class labels, they are not Markovian in the chosen state space, as it records only class counts and ignores the \gls{FIFO} ordering at the \gls{CS}.  
Nevertheless, the system admits a product-form marginal stationary distribution, stated in the following proposition and used in the subsequent delay analysis.

\begin{proposition} \label{prop:jackson2}
	In the setting of \Cref{sec:model_description2}, the processes $(Y(t))_{t \ge 0}$ and $(\tilde{Y}_k)_{k \in \bN}$ admit unique stationary distributions $\phi_{n, m}$ and $\phi_{n, m-1}$, respectively.
	For any $\um \in \bN_{>0}$, the distribution $\phi_{n, \um}$ is given by the product form:
	\begin{align} \label{eq:phi}
		\phi_{n,\um}(x)
		&= \frac{1}{W_{n,\um}}
		\frac{\left(\sum\limits_{j=1}^n \xCS_j\right)!}{\prod\limits_{j=1}^n \xCS_j!}
		\prod_{i=1}^n
		\frac{1}{\xd_i!}
		\frac{1}{\xu_i!}
		\prod_{\substack{s \in \{\mathrm{CS},\\ \mathrm{c},\mathrm{d},\mathrm{u}\}}}
		\left( \frac{p_i}{\mud_i} \right)^{\xd_i},
	\end{align}
	for any state $x \in \mathcal{X}_{4n,\um}$, where $W_{n,\um}$ is the normalizing constant.
\end{proposition}

\begin{proof}
	See Section~\ref{proof_prop:jackson2} of the supplementary materials.
\end{proof}

Since the performance metrics for this extended model depend on $\phi_{n,\um}$ primarily through the normalization constants~$W_{n,\um}$, their efficient computation is critical to avoid combinatorial bottlenecks. Accordingly, we extend \textit{Buzen}'s recursive algorithm to explicitly incorporate the \gls{CS} service rate, enabling the computation of these constants in $\mathcal{O}(nm^2)$ time and $\mathcal{O}(m)$ memory (see Proposition~\ref{prop:buzen_cs} in the supplementary materials). This efficiency is crucial, as it allows for the exact evaluation of the expected relative delay and its gradient detailed in the next section.

\subsection{Delay and Gradient Computation}\label{sec:delay2}
The following result extends \Cref{theo:little_law} to our new setting with \gls{CS}-side delay.

\begin{theorem} \label{theo:little_law2}
	In the model of \Cref{sec:model_description2}, the following identities hold for each $i,j \in \{1,\ldots,n\}$:
	\begin{align}
		\label{eq:D_cs}
		\bEp[D_i]
		&= \sum_{s \in \{\mathrm{CS}, \mathrm{c}, \mathrm{d}, \mathrm{u}\}} \mathbb{E}\!\left[ \tilde{Y}_i^s \right], 
		\\[1ex]
		\label{eq:gradD_cs}
		\frac{\partial}{\partial p_j} \bEp[D_i]
		&= \frac{1}{p_j} \sum_{s,r \in \{\mathrm{CS}, \mathrm{c}, \mathrm{d}, \mathrm{u}\}} 
		\operatorname{Cov}\!\left( \tilde{Y}_i^s, \tilde{Y}_j^r \right).
	\end{align}
	
	Moreover, the following closed-form expressions hold:
	\begin{align}
		\label{exp:delay_cs}
		&\sum_{s \in \{\mathrm{CS}, \mathrm{c}, \mathrm{d}, \mathrm{u}\}} \mathbb{E}[\tilde{Y}_i^s]
		= \textcolor{red}{p_i \tilde{\beta}_{\mathrm{CS},1}} + \tilde{\beta}_{i,1} + \frac{W_{n,m-2}}{W_{n,m-1}}\,\gamma_i,
		\\[2ex]
		\label{exp:grad_cs}
		&\sum_{s,r \in \{\mathrm{CS}, \mathrm{c}, \mathrm{d}, \mathrm{u}\}} \mathbb{E}[\tilde{Y}_i^s \tilde{Y}_j^r]
		= \tilde{\alpha}_{i,j} + \tilde{\beta}_{i,2}\gamma_j + \tilde{\beta}_{j,2}\gamma_i + \tilde{\psi}_{i,j} \nonumber \\[0.5ex]
		&\quad + \textcolor{red}{\tilde{\alpha}^{\mathrm{CS}}_{i,j} 
			+ \tilde{\beta}_{\mathrm{CS},2} \left(p_i \gamma_j + p_j \gamma_i \right) 
			+ p_i \tilde{\alpha}_{\mathrm{CS},j} 
			+ p_j \tilde{\alpha}_{\mathrm{CS},i}}.
	\end{align}
	Here, for each $i,j \in \{1,\ldots,n\}$ and $\ell \in \{1,2\}$:
	\begin{align*}
		\gamma_i &= p_i \left( \frac{1}{\mud_i} + \frac{1}{\muu_i} \right), \qquad
		\tilde{\beta}_{i,\ell} = \sum_{k=1}^{m-\ell} \left( \frac{p_i}{\muc_i} \right)^k \frac{W_{n,m-\ell-k}}{W_{n,m-1}},
		\\[1ex]
		\tilde{\alpha}_{i,j} &=
		\begin{cases}
			\displaystyle \sum_{k=1}^{m-1} (2k-1) \left(\frac{p_i}{\muc_i}\right)^{k} \frac{W_{n,m-1-k}}{W_{n,m-1}} & \text{if } i=j, \\[2ex]
			\displaystyle \sum_{\substack{k, \ell = 1 \\ k + \ell \le m-1}}^{m-2}
			\left( \frac{p_i}{\muc_i} \right)^{k}
			\left( \frac{p_j}{\muc_j} \right)^{\ell}
			\frac{W_{n, m-1-k-\ell}}{W_{n, m-1}} & \text{if } i\neq j,
		\end{cases}
		\\[1ex]
		\tilde{\psi}_{i,j} &= \frac{\gamma_i}{W_{n,m-1}} \left( \gamma_j W_{n,m-3} + \mathbf{1}_{\{i=j\}} W_{n,m-2} \right),
		\\[1ex]
		\textcolor{red}{\tilde{\alpha}^{\mathrm{CS}}_{i,j}} 
		&= p_i \sum_{k=1}^{m-1} \left(\frac{1}{\muCS}\right)^{k} \frac{W_{n,m-1-k}}{W_{n,m-1}} \left[ 2 p_j (k-1) + \mathbf{1}_{\{i=j\}} \right],
		\\[1ex]
		\textcolor{red}{\tilde{\alpha}_{\mathrm{CS},i}}
		&= \sum_{k=1}^{m-2} \sum_{\ell = 1}^{m-1-k}
		\left(\frac{1}{\muCS}\right)^{k}
		\left(\frac{p_i}{\muc_i}\right)^{\ell}
		\frac{W_{n,m-1-k-\ell}}{W_{n,m-1}},
		\\[1ex]
		\textcolor{red}{\tilde{\beta}_{\mathrm{CS},\ell}}
		&= \sum_{k=1}^{m-\ell}
		\left( \frac{1}{\muCS} \right)^k
		\frac{W_{n,m-\ell-k}}{W_{n,m-1}},
	\end{align*}
	and the constants $W_{n,\um}$ for $\um \in \{0,1,\ldots,m-1\}$ are computed using the recursion of Proposition~\ref{prop:buzen_cs} in \Cref{sec:buzen_cs} of the supplementary material.
\end{theorem}

\begin{proof}
	See Section~\ref{proof_theo:little_law2} of the supplementary materials.
\end{proof}

In Equations~\eqref{exp:delay_cs} and~\eqref{exp:grad_cs},
the terms highlighted in \textcolor{red}{red} quantify the additional components introduced specifically by the \gls{CS} queueing process.
Conversely, the remaining terms (in black) correspond exactly to the interactions found in the original model without \gls{CS}-side congestion.
Observe that as $\muCS \to \infty$, we recover the exact closed-form expressions of \Cref{theo:little_law}. This holds because $\lim_{\muCS \to \infty} W_{n,m} = Z_{n,m}$ for all $n,m \in \mathbb{N}_{>0}$, while the \gls{CS}-specific terms vanish (i.e., $\lim_{\muCS \to \infty} \tilde{\alpha}_{\mathrm{CS},i} = \lim_{\muCS \to \infty} \tilde{\alpha}^{\mathrm{CS}}_{i,j} = \lim_{\muCS \to \infty} \tilde{\beta}_{\mathrm{CS},\ell} = 0$).
Finally, the analysis and insights regarding delays presented in Section~\ref{sec:delay} remain applicable to this extended model.

In the remainder of this section, we establish equivalent properties for wall-clock time and energy complexity, explicitly accounting for the \gls{CS} service rate.

\subsection{Clock-Time Complexity}
Theorem~\ref{th:opt_steps} still applies to the model of \Cref{sec:model_description2}. 
Hence, the expression \eqref{eq:K_epsilon_bounded} for $K_{\epsilon}$ remains valid; the only modification concerns the closed-form formula for the expected relative delay~$\bEp[D_i]$, which now is as in Equation~\eqref{exp:delay_cs}.

\begin{proposition}[Time to achieve $\epsilon$-accuracy]
	\label{prop:time-eps2}
	Under Assumptions~\textbf{A1}--\textbf{A5}, there exists $\epsilon_0 > 0$ such that for any target accuracy $\epsilon \in (0, \epsilon_0]$ and any learning rate satisfying $\eta \le \eta_{\max}(p,m)$, the expected wall-clock time required to reach $\epsilon$-accuracy is given by:
	\begin{align} \label{eq:tau2}
		\bEp[\tau_{\epsilon}] = \frac{K_{\epsilon}(p,m)}{\tilde{\lambda}(p,m)}, \quad \text{where:}
	\end{align}
	$K_{\epsilon}(p,m)$ denotes the number of rounds required to achieve $\epsilon$-accuracy, given in \Cref{eq:K_epsilon_bounded} of Theorem~\ref{th:opt_steps}. \\
	$\tilde{\lambda}(p,m)$ is the expected number of rounds completed per unit of wall-clock time, given by
	\begin{align}
		\label{eq:throughput2}
		\tilde{\lambda}(p,m) &= \muCS\, \prb{\sum_{i=1}^n \YCS_i > 0}
		= \frac{W_{n,m-1}}{W_{n,m}}, \\[0.5ex]
		\label{eq:grad_throughput2}
		\frac{\partial}{\partial p_j}\,\tilde{\lambda}(p,m)
		&= \frac{1}{p_j}\,\tilde{\lambda}(p,m)
		\sum_{s \in \{\mathrm{CS},\mathrm{c},\mathrm{d},\mathrm{u}\}} \mathbb{E}[\,\tilde{Y}_j^s - Y_j^s\,],
	\end{align}
	with $W_{n,m}$ and $W_{n,m-1}$ as defined in Proposition~\ref{prop:jackson}, $Y \sim \phi_{n,m}$ and $\tilde{Y} \sim \phi_{n,m-1}$.
\end{proposition}

\begin{proof}
	See Section~\ref{proof_prop:time-eps2} of the supplementary materials.
\end{proof}

\subsection{Energy Complexity} \label{sec:nrg_model2}

We extend our framework by relaxing the assumption that the \gls{CS} is grid-powered; instead, we model it as an energy-constrained edge device.  
Let $\PCS$ denote the power consumed by the \gls{CS} during local processing, including decryption and arithmetic updates.
To simplify the energy accounting, we aggregate communication power costs at the link level.  
Specifically, $\Pd_i$ denotes the total power consumed during a downlink transmission to client~$i$, combining the \gls{CS}'s transmission power and the client's reception power, while $\Pu_i$ represents the total power consumed during an uplink transmission, combining client~$i$'s transmission power and the \gls{CS}'s reception power.
Under this model, the instantaneous power consumption of the system $P(t)$ is given by:
\begin{align*}
	P(t) = \PCS \, \mathbf{1}\{ \sum_{i=1}^n \YCS_i(t) > 0\}
	+ \sum_{i=1}^n \left( \Pc_i \, \mathbf{1}\{\Yc_i(t) > 0\}
	+ \Pu_i \, \Yu_i(t)
	+ \Pd_i \, \Yd_i(t) \right).
\end{align*}
The average energy complexity of the algorithm is then derived as follows:
\begin{proposition}[Energy to Achieve $\epsilon$-Accuracy] \label{prop:energy-eps2}
	Under Assumptions~\textbf{A1}--\textbf{A5}, there exists $\epsilon_0 > 0$ such that for any target accuracy $\epsilon \in (0, \epsilon_0]$ and any learning rate satisfying $\eta \le \eta_{\max}(p,m)$, the expected energy required to reach $\epsilon$-accuracy is given by:
	\begin{align}
		\bEp[E_{\epsilon}] 
		= K_{\epsilon}(p,m)\, \frac{\mathbb{E}[P(0)]}{\tilde{\lambda}(p,m)} \notag
		= K_{\epsilon}(p,m) \left( \frac{\PCS}{\muCS} + \sum_{i=1}^n p_i \mathcal{E}_i \right),
	\end{align}
	where $\mathcal{E}_i \triangleq \frac{\Pc_i}{\muc_i} + \frac{\Pu_i}{\muu_i} + \frac{\Pd_i}{\mud_i}$, $K_{\epsilon}$ is the round complexity defined in \Cref{eq:K_epsilon_bounded} of Theorem~\ref{th:opt_steps}, and $\tilde{\lambda}$ is the system throughput defined in \eqref{eq:throughput2}.
\end{proposition}

\begin{proof}
	See Section~\ref{proof_prop:energy-eps2} of the supplementary materials.
\end{proof}

Proposition~\ref{prop:energy-eps2} exhibits the same energy-latency trade-off discussed in Section~\ref{nrg_latency}; consequently, the conclusions derived therein remain applicable. To minimize energy consumption strictly, without regard for training latency, one would set the concurrency level to $m=1$ and assign routing probabilities according to:
\begin{align}
	p^{\ast E}_i \propto \left( \sqrt{\frac{\PCS}{\muCS} + \mathcal{E}_i} \right)^{-1}, \quad i \in \{1,\ldots,n\}.
\end{align}
This configuration yields the global minimum average energy consumption, denoted by $E^\ast$:
\begin{align}
	E^\ast = \frac{24L\Delta}{n^2 \epsilon} \left(4 + \frac{B}{\epsilon}\right)  \left( \sum_{i=1}^n \sqrt{\frac{\PCS}{\muCS} + \mathcal{E}_i} \right)^2.
\end{align}
However, such a strategy results in prohibitive training duration. Therefore, a practical balance must be struck by solving the joint optimization problem defined in \eqref{eq:joint_opt}.

\section{Conclusion}

Client-side buffering has emerged as a key mechanism for mitigating data heterogeneity in asynchronous \gls{FL}, as illustrated by algorithms such as \texttt{AsyncSGD} and \texttt{Generalized AsyncSGD}.  
However, existing analyses largely neglect
the underlying queueing dynamics that govern their practical behavior.  
We close this gap by deriving closed-form optimizable upper-bounds for key performance metrics, relying on a Jackson network model that bridges stochastic-network theory and asynchronous \gls{FL} theory.

Unlike prior work that focuses on round-based convergence, assumes deterministic or bounded processing-time models, or neglects communication delays, we develop a wall-clock convergence analysis that captures realistic edge-system heterogeneity.  
Our model integrates stochastic computation, communication, and \gls{CS}-side processing times, as well as energy constraints, under heterogeneous data distributions.

Within this framework, we characterize the trade-off between model-parameter staleness and update frequency, and we derive closed-form expressions for key performance metrics, showing that optimizing a single objective in isolation degrades overall performance.

By combining gradient-based optimization with an extension of Buzen’s recursive algorithm, we enable efficient computation of optimal routing and concurrency parameters, making exact performance optimization tractable at scale.

We further extend the analysis to energy consumption, revealing a complementary trade-off between training latency and energy efficiency.  
Our results demonstrate that common strategies (e.g., uniform routing) can operate far from optimality.  
To address this, we propose a joint optimization strategy that navigates the Pareto frontier between wall-clock time and energy consumption via a tunable parameter~$\rho$, achieving energy savings of up to $36\%\text{--}49\%$ while maintaining practical training speeds.

Overall, our findings indicate that sustainable and efficient edge intelligence requires moving beyond iteration counts or raw speed metrics toward holistic, system-aware performance objectives.  
Future work may for instance extend this framework to unreliable environments with dynamic client participation.

\bibliography{paper}

\newpage

\appendix

\begin{center}
	\textbf{\\[-0.4cm] {\Large Optimization Trade-offs in Asynchronous Federated Learning: \\[0.2cm] A~Stochastic Networks Approach} \\[0.3cm] {\large Supplementary Materials}}
	\par\noindent\rule{\textwidth}{0.4pt}
\end{center}

\startcontents[appendix]
% Arguments: {name}{prefix}{depth}{code}
\printcontents[appendix]{}{1}{\setcounter{tocdepth}{1}}
\newpage

\section{General Lemmas}

We begin by stating some results for stationary Markov chains that will be instrumental in the proofs of the main results.

\subsection{Discrete-time Markov chain}
\begin{lemma} \label{lem:2-pi-disc}
	Consider an irreducible positive-recurrent
	discrete-time Markov chain $(X_t, t \in \bN)$
	with discrete state space $\cX$,
	transition probabilities $p(x, y)$, $x, y \in \cX$,
	and invariant distribution $\pi(x)$, $x \in \cX$.
	The process $((X_t, X_{t+1}), t \in \bN)$
	is also an irreducible positive-recurrent Markov chain,
	with state space
	$\cX_2 = \{(x, y) \in \cX \times \cX: p(x, y) > 0\}$
	and invariant distribution
	$\pi_2(x, y) = \pi(x) p(x, y)$, $(x, y) \in \cX_2$.
\end{lemma}

\begin{proof}
	We leave it to the reader to verify that $((X_t, X_{t+1}), t \in \bN)$
	is indeed an irreducible Markov chain
	with state space
	$\cX_2 = \{(x, y) \in \cX \times \cX: p(x, y) > 0\}$.
	The transition probabilities of this Markov chain are given by
	\begin{align*}
		\prb{(X_{t+1}, X_{t+2}) = (y, z) ~| ~ (X_t, X_{t+1}) = (x, y)}
		&= p(y, z),
		\quad t \in \bN,
	\end{align*}
	for each $x, y, z \in \cX$ such that
	$(x, y) \in \cX_2$ and $(y, z) \in \cX_2$.
	Therefore, the balance equations
	of this Markov chain are given by
	\begin{align} \label{eq:2-balance}
		\sum_{x \in \cX: \, p(x, y) > 0} \pi_2(x, y) p(y, z)
		&= \pi_2(y, z),
		\quad (y, z) \in \cX_2.
	\end{align}
	Injecting the definition of $\pi_2$
	given in the text of the lemma into
	the right-hand side of this balance equation yields
	\begin{align*}
		\sum_{x \in \cX: \, p(x, y) > 0} \pi_2(x, y) p(y, z)
		&= \sum_{x \in \cX} \pi(x) p(x, y) p(y, z)
		= \pi(y) p(y, z),
	\end{align*}
	where the second equality follows by applying
	the balance equations of the Markov chain~$(X_t, t \in \bN)$.
	To conclude, it suffices to observe that
	$\pi(y) p(y, z) = \pi_2(y, z)$.
	That $\pi_2$ is indeed a probability distribution follows by writing
	\begin{align*}
		\sum_{(x, y) \in \cX_2} \pi_2(x, y)
		&= \sum_{x \in \cX} \pi(x) \sum_{y \in \cX: p(x, y) > 0} p(x, y)
		= \sum_{x \in \cX} \pi(x) = 1.
	\end{align*}
\end{proof}

\begin{corollary} \label{coro:2-pi-disc}
	Consider an irreducible positive-recurrent
	discrete-time Markov chain $(X_t, t \in \bN)$
	with state space $\cX$,
	transition probabilities $p(x, y)$, $x, y \in \cX$,
	and invariant distribution $\pi(x)$, $x \in \cX$.
	Also let $(X, Y)$ denote a random couple
	distributed according to the distribution~$\pi_2$
	defined in \Cref{lem:2-pi-disc}.
	Given $H \subseteq \cX \times \cX$
	and $g: H \to \bR^+$,
	we have
	\begin{align*}
		\esp{g(X, Y) | (X, Y) \in H}
		&= \frac
		{\sum_{(x, y) \in H} g(x, y) \pi(x) p(x, y)}
		{\sum_{(x, y) \in H} \pi(x) p(x, y)}.
	\end{align*}
\end{corollary}

\subsection{Continuous-time Markov chain}

We adapt the statements of the results
to account for the fact that the initial Markov chain
is defined using transition rates
instead of transition probabilities.

\begin{lemma} \label{lem:2-pi-cont}
	Consider an irreducible positive-recurrent
	continuous-time Markov chain $(X_\tau, \tau \in [0, +\infty))$
	with discrete state space $\cX$,
	transition rates $q(x, y)$, $x, y \in \cX$,
	and invariant distribution $\pi(x)$, $x \in \cX$.
	Let $(\hat{X}_t, t \in \bN)$
	denote its embedded discrete-time Markov chain.
	The process $((\hat{X}_t, \hat{X}_{t+1}), t \in \bN)$
	is an irreducible positive-recurrent
	discrete-time Markov chain,
	with state space
	$\cX_2 = \{(x, y) \in \cX \times \cX: q(x, y) > 0\}$
	and invariant measure
	$\pi_2(x, y) = \pi(x) q(x, y)$, $(x, y) \in \cX_2$.
\end{lemma}

\begin{proof}
	This result follows from \Cref{lem:2-pi-disc},
	upon observing that the Markov chain
	$(\hat{X}_t, t \in \bN)$ has
	transition probabilities
	$q(x, y) / \sum_{z \in \cX} q(x, z)$, $x, y \in \cX$,
	and invariance measure
	$\pi(x) \sum_{y \in \cX} q(x, y)$, $x \in \cX$.
\end{proof}

\begin{remark}
	Contrary to the definition of embedded Markov chains
	in \cite[Section 13.3.2]{Bremaud_mc_2020},
	here both the continuous-time Markov chain
	$(X_\tau, \tau \in [0, +\infty))$
	and its embedded Markov chain
	$(\hat{X}_t, t \in \bN)$
	can have jumps from a state to itself.
\end{remark}

\begin{corollary} \label{coro:2-pi-cont}
	Consider an irreducible positive-recurrent
	continuous-time Markov chain $(X_t, t \in \bN)$
	with state space $\cX$,
	transition rates $q(x, y)$, $x, y \in \cX$,
	and invariant distribution $\pi(x)$, $x \in \cX$.
	Also let $(X, Y)$ denote a random couple
	distributed according to the distribution~$\pi_2$
	defined in \Cref{lem:2-pi-disc}.
	Given $H \subseteq \cX \times \cX$
	and $g: H \to \bR^+$,
	we have
	\begin{align*}
		\esp{g(X, Y) | (X, Y) \in H}
		&= \frac
		{\sum_{(x, y) \in H} g(x, y) \pi(x) q(x, y)}
		{\sum_{(x, y) \in H} \pi(x) q(x, y)}.
	\end{align*}
\end{corollary}

\section{Experimental Details} \label{exp:details}

\subsection{Datasets and Neural Network Architectures}
We evaluate our methods on three standard balanced datasets:
\begin{itemize}
	\item \textbf{KMNIST:} Comprises 70,000 grayscale images ($28 \times 28$ pixels) evenly distributed across 10 classes, partitioned into 60,000 training and 10,000 testing samples.
	\item \textbf{EMNIST:} Contains 131,600 grayscale images ($28 \times 28$ pixels) evenly distributed across 47 classes, partitioned into 112,800 training and 18,800 testing samples.
	\item \textbf{CIFAR-100:} Consists of 60,000 RGB images ($32 \times 32$ pixels) spanning 100 classes, partitioned into 50,000 training and 10,000 testing samples.
\end{itemize}

\noindent \textbf{KMNIST and EMNIST Architecture:} For these datasets, we employ a Convolutional Neural Network (CNN) structured as follows:
\begin{itemize}
	\item Two convolutional layers with $7 \times 7$ filters and ReLU activation. The first layer has 20 channels, and the second has 40 channels.
	\item A $2 \times 2$ max pooling layer.
	\item A final fully connected layer with 10 neurons and a softmax activation function.
\end{itemize}

\noindent \textbf{CIFAR-100 Architecture:} For CIFAR-100, we utilize a deeper CNN architecture comprising:
\begin{itemize}
	\item \textbf{Three Sequential Convolutional Blocks:} Each block consists of two $3 \times 3$ convolutional layers, followed by ReLU activation and Group Normalization. The blocks have channel depths of 32, 64, and 128, respectively. Each block concludes with a $2 \times 2$ max pooling layer and a dropout layer ($p=0.25$).
	
	\item \textbf{Classification Head:} This block includes a flattening layer, a fully connected layer with 128 neurons, a dropout layer ($p=0.25$), and a final output layer corresponding to the 100 classes with softmax activation.
\end{itemize}

All experiments are implemented in PyTorch and executed on an NVIDIA Tesla P100 GPU. Unless otherwise stated, the stochastic gradient for each task is computed using a batch size of $128$.

\subsection{Softmax Reparameterization for Constrained Optimization}

To optimize a generic differentiable routing objective $h(p)$ via gradient descent while satisfying the simplex constraints on $p$ (i.e., $p_i > 0$ and $\sum p_i = 1$), we utilize a softmax reparameterization.
Instead of optimizing $p$ directly, we introduce unconstrained auxiliary parameters $\Theta = (\theta_1, \dots, \theta_n) \in \mathbb{R}^n$ and define the routing probabilities as:
\begin{align} \label{eq:softmax}
	p_j(\Theta) = \frac{e^{\theta_j}}{\sum_{i=1}^n e^{\theta_i}}, \quad j \in {1, \ldots, n}.
\end{align}
This transformation inherently guarantees that the resulting vector $p$ remains a valid probability distribution throughout the optimization process.

Consequently, the gradients with respect to the auxiliary parameters $\Theta$ are computed using the chain rule:
\begin{align}
	\frac{\partial h}{\partial \theta_j}= \left\langle \nabla_p h , \frac{\partial p}{\partial \theta_j} \right\rangle.
\end{align}
Here, the Jacobian term is given by $\frac{\partial p}{\partial \theta_j} = p_j (e_j - p)$, where $e_j$ denotes the standard basis vector in $\mathbb{R}^n$ (with 1 at index $j$ and 0 elsewhere), and $\langle \cdot , \cdot \rangle$ represents the Euclidean dot product.

\section{Proof of Proposition~\ref{prop:jackson}} \label{proof_prop:jackson}
We first establish that the continuous-time process $\xi = (\xi(t))_{t \ge 0}$ is an ergodic Markov chain with stationary distribution $\pi_{n,m}$. Subsequently, we prove that the discret process $(X_k)_{k \in \mathbb{N}}$ is also an ergodic Markov chain, but with stationary distribution $\pi_{n,m-1}$.

\subsection{Continuous-Time Process $\texorpdfstring{\xi}{xi}$}
Since the routing probabilities satisfy $p_i > 0$ for all $i \in \{1,\ldots,n\}$, the underlying routing graph is strongly connected, implying that the sequence $\xi(t)$ is irreducible~\cite[Proposition 1.10]{serfozo1999}. Furthermore, since the state space $\mathcal{X}_{n,m}$ is finite, the process is positive recurrent.

Based on the model assumptions, the non-zero entries of the infinitesimal generator $Q = (q(x,y))_{x,y \in \mathcal{X}_{n,m}}$ are given by:
\begin{align} \label{eq:generator}
	q(x,y) =
	\begin{cases}
		\mud_i\,\xd_i, & \text{if } y = x - \ed_i + \ec_i, \quad i \in \{1,\dots,n\}, \\[1ex]
		\muc_i\,\mathbf{1}(\xc_i \ge 1), & \text{if } y = x - \ec_i + \eu_i, \quad i \in \{1,\dots,n\}, \\[1ex]
		\muu_i\,\xu_i\,p_j, & \text{if } y = x - \eu_i + \ed_j, \quad i,j \in \{1,\dots,n\}, \\[1ex]
		-\sum_{z \ne x} q(x,z), & \text{if } y = x, \\[1ex]
		0, & \text{otherwise}.
	\end{cases}
\end{align}

Here, $\ed_i, \ec_i, \eu_i \in \{0,1\}^{3n}$ denote the canonical unit vectors associated with the downlink, computation, and uplink queues of client $i$, respectively. Specifically, if the state is indexed as $x = (\xd_1, \xc_1, \xu_1, \dots, \xd_n, \xc_n, \xu_n)$, then $\ed_i$ has a 1 at index $3(i-1)+1$ and 0 elsewhere (and similarly for $\ec_i, \eu_i$).

It is straightforward to verify that the distribution $\pi_{n,m}$ satisfies the global balance equations:
\begin{align*}
	\pi_{n,m}(x) \sum_{y \ne x} q(x,y) = \sum_{y \ne x} \pi_{n,m}(y) q(y,x), \quad x \in \mathcal{X}_{n,m}.
\end{align*}
proving it is the unique stationary distribution.

\subsection[Discrete-Time Process X]{Discrete-Time Process $X$}

Recall that $(T_k)_{k \in \mathbb{N}}$ denotes the sequence of time instants at which service completions occur at servers $(u_i)_{i \in \{1,\ldots,n\}}$. Clearly, the event $\{T_k \le t\}$ is measurable with respect to the history $\{\xi(s) : 0 \leq s \leq t\}$; hence, for all $k \in \mathbb{N}$, $T_k$ is a stopping time.

By the strong Markov property, the sequence $(\xi(T_k))_{k \in \mathbb{N}}$ forms a discrete-time homogeneous Markov chain, and so does $(X_k)_{k \in \mathbb{N}}$. Since $\xi$ is irreducible and the state space is finite, $(X_k)_{k \in \mathbb{N}}$ is therefore irreducible and positive recurrent.  
We now proceed to derive its stationary distribution.

Let $(\hat{\xi}_k)_{k \in \mathbb{N}_{\ge 0}}$ denote the jump chain of the continuous-time Markov chain $(\xi(t))_{t \in \mathbb{R}_{\ge 0}}$. By Lemma~\ref{lem:2-pi-cont}, the sequence of pairs
\[
\bigl\{ (\hat{\xi}_k, \hat{\xi}_{k+1}) : k \in \mathbb{N}_{\ge 0} \bigr\}
\]
forms an irreducible discrete-time homogeneous Markov chain with invariant measure
\[
\hat\pi(x,y) = \pi_{n,m}(x)\,q(x,y), \quad x,y \in \mathcal{X}_{3n,m},
\]
where $q$ denotes the infinitesimal generator of $(\xi(t))_{t \in \mathbb{R}_{\ge 0}}$ defined in \eqref{eq:generator}. The stationary distribution of $(X_k)_{k \in \mathbb{N}_{\ge 0}}$ satisfies, for all $x \in \mathcal{X}_{3n,m-1}$,
\begin{align} \label{stat_def_X}
	\mathbb{P}(X = x) 
	= \mathbb{E}\!\left[\sum_{i=1}^n \sum_{j=1}^n \mathbf{1}\{\hat{\xi}_k = x + \eu_j, \,\hat{\xi}_{k+1} = x + \ed_i\} 
	\;\middle|\; (\hat{\xi}_k, \hat{\xi}_{k+1}) \in H\right],
\end{align}
where the set of valid transitions $H$ is defined as
\[
H = \bigl\{ (y,z) \in \mathcal{X}_{3n,m} \times \mathcal{X}_{3n,m} : z = y - \eu_i + \ed_j \text{ for some } i,j \in \{1,\ldots,n\} \bigr\}.
\]
Applying Corollary~\ref{coro:2-pi-cont} to Equation~\eqref{stat_def_X} with
\[
g(y,z) = \sum_{i=1}^n \sum_{j=1}^n \mathbf{1}\{y = x + \eu_j, \, z = x + \ed_i\},
\]
the stationary distribution can be expressed, for each $x \in \mathcal{X}_{3n,m-1}$, as
\begin{align}
	\mathbb{P}(X = x) 
	&= \frac{\sum_{i=1}^n \sum_{j=1}^n \pi_{n,m}(x + \eu_j) \,\muu_j \,(\xu_j+1)\,p_i}
	{\sum_{y \in \mathcal{X}_{n,m}} \sum_{i=1}^n \sum_{j=1}^n \pi_{n,m}(y)\,\muu_j\,\yu_j\,p_i} \\
	\label{stat_X}
	&= \frac{\sum_{j=1}^n \pi_{n,m}(x + \eu_j) \,\muu_j \,(\xu_j+1)}
	{\sum_{j=1}^n \muu_j \,\mathbb{E}[\xiu_j]}.
\end{align}
We simplify the numerator and denominator separately. For the denominator, observe that for each $j \in \{1,\ldots,n\}$,
\begin{align*}
	\mathbb{E}[\xiu_j] 
	&= \sum_{x \in \mathcal{X}_{n,m}} \xu_j \,\pi_{n,m}(x) \\[0.5ex]
	&= \frac{1}{Z_{n,m}} \sum_{x \in \mathcal{X}_{n,m}} \xu_j 
	\prod_{i=1}^n \left(\frac{p_i}{\muc_i}\right)^{\xc_i} \frac{1}{\xd_i!}\left(\frac{p_i}{\mud_i}\right)^{\xd_i} \frac{1}{\xu_i!}\left(\frac{p_i}{\muu_i}\right)^{\xu_i} \\[0.5ex]
	&= \frac{1}{Z_{n,m}} \frac{p_j}{\muu_j}
	\underbrace{\sum_{\substack{x \in \mathcal{X}_{n,m} \\ \xu_j \ge 1}} 
		\prod_{i=1}^n \left(\frac{p_i}{\muc_i}\right)^{\xc_i} \frac{1}{\xd_i!}\left(\frac{p_i}{\mud_i}\right)^{\xd_i} \frac{1}{(\xu_i-\mathbf{1}\{i=j\})!}\left(\frac{p_i}{\muu_i}\right)^{\xu_i - \mathbf{1}\{i=j\}}}_{= Z_{n,m-1}} \\[0.5ex]
	&= \frac{p_j}{\muu_j}\,\frac{Z_{n,m-1}}{Z_{n,m}}.
\end{align*}
The second equality follows from the definition of $\pi_{n,m}$ in \Cref{eq:pi}, the third by factoring out $\tfrac{p_j}{\muu_j}$, and the last by performing the change of variables $y = x - \eu_j$ alongside the definition of the normalizing constant $Z_{n,m-1}$,
observing that as $x$ ranges over $\{x \in \mathcal{X}_{n,m} : \xu_j \ge 1\}$, the variable $y = x - \eu_j$ ranges over the entire state space $\mathcal{X}_{n,m-1}$.
Therefore, summing over all $j$, the denominator becomes:
\begin{align} \label{nume}
	\sum_{j=1}^n \muu_j \,\mathbb{E}[\xiu_j]
	= \sum_{j=1}^n \muu_j \,\frac{Z_{n,m-1}}{Z_{n,m}} \frac{p_j}{\muu_j}
	= \frac{Z_{n,m-1}}{Z_{n,m}}.
\end{align}

For the numerator, we compute, for each $j \in \{1, \ldots, n\}$:
\begin{align*}
	&\pi_{n,m}(x + \eu_j) \,\muu_j\,(\xu_j+1) \\
	&= \frac{\muu_j (\xu_j+1)}{Z_{n,m}}
	\prod_{i=1}^n \left(\frac{p_i}{\muc_i}\right)^{\xc_i} 
	\frac{1}{\xd_i!}\left(\frac{p_i}{\mud_i}\right)^{\xd_i}
	\frac{1}{(\xu_i + \mathbf{1}\{i=j\})!}\left(\frac{p_i}{\muu_i}\right)^{\xu_i + \mathbf{1}\{i=j\}} \\[0.5ex]
	&= \frac{\muu_j}{Z_{n,m}} \cdot \frac{p_j}{\muu_j}
	\underbrace{\prod_{i=1}^n \left(\frac{p_i}{\muc_i}\right)^{\xc_i} 
		\frac{1}{\xd_i!}\left(\frac{p_i}{\mud_i}\right)^{\xd_i}
		\frac{1}{\xu_i!}\left(\frac{p_i}{\muu_i}\right)^{\xu_i}}_{= Z_{n,m-1}\,\pi_{n,m-1}(x)} \\[0.5ex]
	&= \frac{Z_{n,m-1}}{Z_{n,m}}\,p_j\,\pi_{n,m-1}(x).
\end{align*}
Thus, summing over $j$ yields:
\begin{align} \label{deno}
	\sum_{j=1}^n \pi_{n,m}(x + \eu_j) \,\muu_j\,(\xu_j+1)
	= \frac{Z_{n,m-1}}{Z_{n,m}}\,\pi_{n,m-1}(x).
\end{align}

Finally, substituting \eqref{nume} and \eqref{deno} back into \eqref{stat_X}, we conclude:
\[
\mathbb{P}(X = x) = \pi_{n,m-1}(x), 
\quad \text{for each } x \in \mathcal{X}_{3n,m-1}.\qedhere
\]

\section{Buzen's Recursive Algorithm} \label{sec:buzen}

All performance metrics derived for the model in \Cref{sec:model_description} depend on the stationary distribution $\phi_{n,\um}$ primarily through its normalization constants $Z_{n,\um}$.  
A direct computation of these constants would require enumerating all admissible states, $\binom{n+m-1}{m-1}$ in total, which grows combinatorially with $n$ and $m$ and quickly becomes intractable.  

To overcome this limitation, we adapt Buzen’s recursive algorithm~\cite{buzen1973computational}, enabling the computation of the normalization constants in $\mathcal{O}(nm^2)$ time and $\mathcal{O}(m)$ memory.  
Moreover, the recursive structure allows all constants $Z_{n,\um}$, for $\um \in \{0,1,\ldots,m\}$, to be computed simultaneously in a single pass.  
The resulting algorithm is formalized in the following proposition.

\begin{proposition} \label{prop:buzen}
	The normalization constants $Z_{n,\um}$, for all $\um \in \{0,1,\ldots,m\}$, can be computed in $\mathcal{O}(nm^2)$ time and $\mathcal{O}(m)$ memory using Buzen’s recursive algorithm. In particular, $Z_{n,\um} = U_{3n,\um}$ for all $\um \in \{0,1,\ldots,m\}$, where the quantities $U_{\un,\um}$ are defined recursively by
	\begin{alignat*}{2}
		\bullet \quad & U_{\un,0} = 1,
		&\quad& \text{for } \un \in \{1,\ldots,3n\} \text{.} \\
		\bullet \quad & U_{1,\um} = \left(\frac{p_1}{\muc_1}\right)^{\um},
		&& \text{for } \um \in \{0,\ldots,m\} \text{.} \\
		\bullet \quad & U_{\un,\um} = U_{\un-1,\um} + \frac{p_\un}{\muc_\un}\,U_{\un,\um-1},
		&& \text{for } \un \in \{2,\ldots,n\} \text{ and } \um \in \{1,\ldots,m\} \text{.} \\
		\bullet \quad & U_{\un,\um} = \sum_{k=0}^{\um} \frac{1}{k!} \left(\frac{p_{\un-n}}{\mud_{\un-n}}\right)^{k} U_{\un-1,\um-k},
		&& \text{for } \un \in \{n+1,\ldots,2n\} \text{ and } \um \in \{1,\ldots,m\} \text{.} \\
		\bullet \quad & U_{\un,\um} = \sum_{k=0}^{\um} \frac{1}{k!} \left(\frac{p_{\un-2n}}{\muu_{\un-2n}}\right)^{k} U_{\un-1,\um-k},
		&& \text{for } \un \in \{2n+1,\ldots,3n\} \text{ and } \um \in \{1,\ldots,m\} \text{.}
	\end{alignat*}
\end{proposition}

\section{Proof of Theorem~\ref{theo:little_law}} \label{proof_theo:little_law}

We proceed to establish each equation of the theorem in turn.

\subsection{Proof of Equation~\eqref{eq:D}}
Fix $i \in \{1,\ldots,n\}$ and $k \in \mathbb{N}$. In the stationary regime, we have
\[
\bEp[D_i] = \bEp[D_{i,k}]
= \sum_{j=1}^n \bEp[D_{i,k} \mid A_k=j] \,\bPp(A_k=j)
= p_i \,\bEp[D_{i,k} \mid A_k=i]
= p_i \,\bEp[R_i],
\]
where $R_{i,l}$ denotes the number of model-parameter updates occurring between the instant the \gls{CS} sends the $l$-th task to client~$i$ and the instant the resulting gradient is applied. This corresponds to the sojourn time (measured in "number of tasks") of the task in the subsystem composed of servers $\mathrm{d}_i$, $\mathrm{c}_i$, and $\mathrm{u}_i$. We write $\bEp[R_i]$ for the expectation of $R_{i,l}$ under stationarity.

Thus, proving \Cref{eq:D} is equivalent to showing that, for each $i \in \{1,\ldots,n\}$,
\begin{align} \label{eq:discret_little}
	\bEp[R_i] = \frac{\mathbb{E}[\Xd_i + \Xc_i + \Xu_i]}{p_i},
\end{align}
which can be interpreted as Little’s law applied at the discrete instants of service completion at the uplink queues $\mathrm{u}_{i=1}^n$.

\textbf{Step 1: Uniformization Construction.}

For technical convenience, we introduce an auxiliary continuous-time system, referred to as the \emph{uniform (continuous-time) system}. In this system, the time between any two consecutive model parameter updates (rounds) is exponentially distributed with mean~1. While this assumption implies a fictitious time scale, it allows us to apply the classical continuous-time Little’s law. We will subsequently map the results back to the original discrete-time process~$(X_k)_{k \in \mathbb{N}}$.

Recall that $(T_k)_{k \in \mathbb{N}}$ denotes the sequence of time instants at which service completions occur at servers $(u_i)_{i \in \{1,\ldots,n\}}$.
As shown in \Cref{proof_prop:jackson}, $(\xi(T_k))_{k \in \mathbb{N}}$ forms an ergodic discrete-time homogeneous Markov chain. We construct the \textit{uniform continuous-time Markov chain} $(\Bar{\xi}(t))_{t \in \mathbb{R}_{\ge 0}}$ by subordinating $(\xi(T_k))_{k \in \mathbb{N}}$ to a Poisson process $(N(t))_{t \ge 0}$ with rate~1:
\[
\Bar{\xi}(t) = \xi(T_{N(t)}), \quad t \ge 0.
\]

\textbf{Step 2: Arrivals and Sojourns.}

For each $k \in \mathbb{N}$, let $\bar{T}_{i, k}$ denote the arrival time of the $k$-th task assigned to client~$i$ in the \textit{uniform system}, and let $N_i(t)$ be the corresponding counting process.

In this simplified model, $(\bar{T}_{i, k})_{k \in \mathbb{N}}$ forms a homogeneous Poisson point process with rate $p_i$. Indeed, since round durations are \gls{iid} exponential with mean~1, the global task stream from the \gls{CS} forms a Poisson process of rate~1. By the property of Poisson thinning, since each task is routed to client~$i$ independently with probability~$p_i$, the resulting process is Poisson with rate $p_i$.

The sojourn time of the $k$-th task assigned to client~$i$ (traversing servers $\mathrm{d}_i, \mathrm{c}_i, \mathrm{u}_i$) is given~by
\[
\bar{R}_{i,k} = \sum_{l=1}^{R_{i,k}+1} E_l,
\]
where $(E_l)_{l \in \mathbb{N}}$ are \gls{iid} exponential random variables with mean~1.
Given these sequences, the total number of tasks within the subsystem of client~$i$ (comprising servers $\mathrm{d}_i, \mathrm{c}_i$ and $\mathrm{u}_i$) at time $t$ is:
\begin{align} \label{eq:client_load}
	\Bar{\xi}^{\mathrm{d}}_i(t) + \Bar{\xi}^{\mathrm{c}}_i(t) + \Bar{\xi}^{\mathrm{u}}_i(t) = \sum_{k \in \mathbb{N}} \mathbf{1}\{\bar{T}_{i, k} \le t < \bar{T}_{i, k} + \bar{R}_{i,k}\}.   
\end{align}

\textbf{Step 3: Applying Little’s Law.}

The constructed system satisfies the standard conditions for Little's Law~\cite[Theorem~5.2]{serfozo1999}:
(i) $(\Bar{\xi}(t))_{t \ge 0}$ is an ergodic Markov chain;
(ii) The arrival process $(\bar{T}_{i, k})_{k \in \mathbb{N}}$ is a Poisson process (and thus stationary and simple);
(iii) The sojourn time $\bar{R}_{i,k}$ depends only on the process $(\Bar{\xi}(t), t \ge \bar{T}_{i, k})$,  
and (iv) $\mathbb{E}[\Bar{\xi}^{\mathrm{d}}_i + \Bar{\xi}^{\mathrm{c}}_i + \Bar{\xi}^{\mathrm{u}}_i] < \infty$ and $\mathbb{E}[N_i(1)] = p_i$.

Therefore, \Cref{eq:client_load} yields: 
\begin{align} \label{little_uniform}
	\mathbb{E}[\Bar{\xi}^{\mathrm{d}}_i + \Bar{\xi}^{\mathrm{c}}_i + \Bar{\xi}^{\mathrm{u}}_i] = p_i\,\bEp[\bar{R}_i]. 
\end{align}

\textbf{Step 4: Mapping Back to the Discrete Model.}

By the properties of the \textit{uniformization}, the stationary distribution of $\Bar{\xi}$ is identical to that of the subordinated chain $(\xi(T_k))_{k \in \mathbb{N}}$.
Thus, $\mathbb{E}[\Bar{\xi}^s_i] = \mathbb{E}[\xi^s_i(T_k)]$ for each component $s \in \{\mathrm{d},\mathrm{c},\mathrm{u}\}$ and client $i$. Furthermore, since the variables $(E_l)_{l \in \mathbb{N}}$ are \gls{iid} with mean~1 and independent of $R_i$, Wald's identity yields $\mathbb{E}[\bar{R}_i] = \mathbb{E}[R_i + 1]$.
Substituting these into Equation~\eqref{little_uniform}:
\[
\mathbb{E}[\xid_i(T_k) + \xic_i(T_k) + \xiu_i(T_k)] = p_i\,\bEp[R_i+1] = p_i\,\bEp[R_i] + p_i.
\]
Using the fact that $\mathbb{E}[\mathbf{1}\{A_k=i\}] = p_i$, we rearrange this as:
\[
\mathbb{E}[\xid_i(T_k) - \mathbf{1}\{A_k=i\} + \xic_i(T_k) + \xiu_i(T_k)] = p_i\,\bEp[R_i].
\]
Finally, recalling the definition of the discrete state $X$ (the state at parameter update instants immediately \textit{before} the dispatch of the new task~\eqref{def:X}), we have the relation: $X_{i,k}^d + X_{i,k}^c + X_{i,k}^u = \xid_i(T_k) - \mathbf{1}\{A_k=i\} + \xic_i(T_k) + \xiu_i(T_k)$. Taking expectations yields:
\[
\mathbb{E}[\Xd_i + \Xc_i + \Xu_i] = p_i\,\bEp[R_i],
\]
which is exactly \eqref{eq:discret_little}. This completes the proof of \Cref{eq:D}.

\subsection{Proof of Equation~\eqref{eq:gradD}} \label{proof_eq:gradD}

Let $i, j \in \{1, 2, \ldots, n\}$.
Our goal is to prove~\eqref{eq:gradD}. By~\eqref{eq:D} and the \textbf{bilinearity} of the covariance, this is equivalent to showing:
\begin{align*}
	\frac{\partial \bE[\Xd_i + \Xc_i + \Xu_i]}{\partial (\log p_j)}
	&= \cov[\Xd_i + \Xc_i + \Xu_i,\;\Xd_j + \Xc_j + \Xu_j].
\end{align*}
Recall that the random vector $X$ follows the stationary distribution~ $\pi_{n, m-1}$ (see Equation~\eqref{eq:pi}),
which we can rewrite as
\begin{multline} \label{eq:pi-log}
	\log \pi_{n, m-1}(x)
	= - \log Z_{n, m-1} + \sum_{i = 1}^n \left( \xd_i+\xc_i+\xu_i \right) \log p_i \\
	-  \log \left( (\muc_i)^{\xc_i} (\mud_i)^{\xd_i} \xd_i! (\muu_i)^{\xu_i} \xu_i! \right),
	\quad x \in \cX_{3n, m-1},
\end{multline}
where $Z_{n, m-1}$ follows by normalization:
\begin{align} \label{eq:Z-log}
	Z_{n, m-1}
	= \sum_{x \in \cX_{n, m-1}} \exp \left( \sum_{i = 1}^n \left( \xd_i+\xc_i+\xu_i \right) \log p_i 
	-  \log \left( (\muc_i)^{\xc_i} (\mud_i)^{\xd_i} \xd_i! (\muu_i)^{\xu_i} \xu_i! \right) \right).
\end{align}

Let us first prove the following intermediary result:
\begin{align} \label{eq:grad-log}
	\begin{aligned}
		\frac{\partial \log Z_{n, m-1}}{\partial (\log p_j)}
		&= \mathbb{E}[\Xd_j + \Xc_j + \Xu_j],
		\\
		\frac{\partial \log \pi_{n, m-1}(x)}{\partial (\log p_j)}
		&= \left( \xd_j+\xc_j+\xu_j \right) - \mathbb{E}[\Xd_j + \Xc_j + \Xu_j],
		\quad x \in \cX_{3n, m-1}.
	\end{aligned}
\end{align}

The first part of~\eqref{eq:grad-log} follows by
taking the partial derivative of~\eqref{eq:Z-log}
and rearranging the terms to retrieve the definition of~$\pi_{n, m-1}$:
\begin{align*}
	&\frac{\partial \log(Z_{n, m-1})}{\partial (\log p_j)} \\
	&= \frac1{Z_{n, m-1}} \frac{\partial Z_{n, m-1}}{\partial (\log p_j)}, \\
	&= \frac1{Z_{n, m-1}}
	\sum_{x \in \cX_{n, m-1}} \left( \xd_j+\xc_j+\xu_j \right) \exp \left( \sum_{i = 1}^n \left( \xd_j+\xc_j+\xu_j \right) \log p_i 
	-  \log \left( (\muc_i)^{\xc_i} (\mud_i)^{\xd_i} \xd_i! (\muu_i)^{\xu_i} \xu_i! \right) \right), \\
	&= \sum_{x \in \cX_{n, m-1}} \left( \xd_j+\xc_j+\xu_j \right) \exp \left( \sum_{i = 1}^n \left( \xd_j+\xc_j+\xu_j \right) \log p_i 
	-  \log \left( (\muc_i)^{\xc_i} (\mud_i)^{\xd_i} \xd_i! (\muu_i)^{\xu_i} \xu_i! \right) - \log Z_{n, m-1} \right), \\
	&= \sum_{x \in \cX_{n, m-1}} \left( \xd_j+\xc_j+\xu_j \right) \pi_{n, m-1}(x), \\
	&= \mathbb{E}[\Xd_j + \Xc_j + \Xu_j].
\end{align*}

Now, the second part of~\eqref{eq:grad-log} follows
by taking the partial derivative of~\eqref{eq:pi-log}
and injecting the previous result:
\begin{align*}
	\frac{\partial \log \pi_{n, m-1}(x)}{\partial (\log p_j)}
	&= \xd_j+\xc_j+\xu_j - \frac{\partial \log Z_{n, m-1}}{\partial (\log p_j)}
	= \xd_j+\xc_j+\xu_j - \mathbb{E}[\Xd_j + \Xc_j + \Xu_j],
	\quad x \in \cX_{3n, m-1}.
\end{align*}

To conclude, it suffices to inject the second part of~\eqref{eq:grad-log}
into the definition of expectation:
\begin{align*}
	&\frac{\partial \mathbb{E}[\Xd_i + \Xc_i + \Xu_i]}{\partial (\log p_j)} \\
	&= \sum_{x \in \cX_{3n, m-1}}
	\left( \xc_i+\xd_i+\xu_i \right) \frac{\partial \pi_{n, m-1}(x)}{\partial (\log p_j)}, \\
	&= \sum_{x \in \cX_{3n, m-1}}
	\pi_{n, m-1}(x) \left( \xc_i+\xd_i+\xu_i \right) \frac{\partial \log \pi_{n, m-1}(x)}{\partial (\log p_j)}, \\
	&= \sum_{x \in \cX_{3n, m-1}}
	\pi_{n, m-1}(x) \left( \xc_i+\xd_i+\xu_i \right) \left( \xd_j+\xc_j+\xu_j - \mathbb{E}[\Xd_j + \Xc_j + \Xu_j]\right), \\
	&= \cov[\Xd_i + \Xc_i + \Xu_i, \Xd_j + \Xc_j + \Xu_j]. \qedhere
\end{align*}

\subsection{Proof of Equation~\eqref{eq:EX}} \label{proof_eq:EX}

To derive \Cref{eq:EX}, we explicitly compute the expectations $\mathbb{E}[\Xc_i]$, $\mathbb{E}[\Xd_i]$, and $\mathbb{E}[\Xu_i]$ for an arbitrary client $i \in \{1,\ldots,n\}$ and then sum them to obtain the final result.

We begin by computing the expected number of tasks in the computing queue, $\mathbb{E}[\Xc_i]$. Utilizing the tail sum formula for the expectation of a non-negative integer-valued random variable, we have $\mathbb{E}[\Xc_i] = \sum_{k=1}^{m-1} \mathbb{P}(\Xc_i \geq k)$.First, for any $k \in \{0,\ldots,m-1\}$, we derive the expression for the probability $\mathbb{P}(\Xc_i \geq k)$ by summing the stationary distribution $\pi_{n,m-1}$ over all states where client $i$ has at least $k$ computing tasks:
\begin{align}
	\mathbb{P}(\Xc_i \geq k) 
	&= \sum_{\substack{x \in \mathcal{X}_{n,m-1} \\ \xc_i \ge k}} \pi_{n,m-1}(x) \nonumber \\[0.5ex]
	&= \frac{1}{Z_{n,m-1}} \sum_{\substack{x \in \mathcal{X}_{n,m-1} \\ \xc_i \ge k}} 
	\left[ \left(\frac{p_i}{\muc_i}\right)^{\xc_i} \prod_{j \neq i} \left(\frac{p_j}{\muc_j}\right)^{\xc_j} \right]
	\prod_{l=1}^n \frac{1}{\xd_l!}\left(\frac{p_l}{\mud_l}\right)^{\xd_l} \frac{1}{\xu_l!}\left(\frac{p_l}{\muu_l}\right)^{\xu_l} \nonumber \\[0.5ex]
	&= \frac{1}{Z_{n,m-1}} \left(\frac{p_i}{\muc_i}\right)^k
	\underbrace{
		\sum_{\substack{x \in \mathcal{X}_{n,m-1} \\ \xc_i \ge k}} 
		\left(\frac{p_i}{\muc_i}\right)^{\xc_i - k} 
		\left[ \prod_{j \neq i} \left(\frac{p_j}{\muc_j}\right)^{\xc_j} \right] 
		\prod_{l=1}^n \frac{1}{\xd_l!}\left(\frac{p_l}{\mud_l}\right)^{\xd_l} \frac{1}{\xu_l!}\left(\frac{p_l}{\muu_l}\right)^{\xu_l} 
	}_{= Z_{n,m-1-k}} \nonumber \\[0.5ex]
	&= \left(\frac{p_i}{\muc_i}\right)^k \frac{Z_{n,m-1-k}}{Z_{n,m-1}}. \label{eq:prob_Xc}
\end{align}
Here, the second equality follows from the definition of $\pi_{n,m-1}$ in \Cref{eq:pi}, while the third is obtained by factoring out $(p_i/\muc_i)^k$. The remaining sum is identified as the normalizing constant $Z_{n,m-1-k}$ by performing the change of variables $y = x - k\ec_i$.
Recall that $\ed_i, \ec_i, \eu_i$ denote the canonical unit vectors corresponding to the components $\xd_i, \xc_i, \xu_i$ within the state vector $x$. We observe that as $x$ ranges over $\{x \in \mathcal{X}_{n,m-1} : \xc_i \ge k\}$, the shifted variable $y$ maps bijectively to the entire state space $\mathcal{X}_{n,m-1-k}$. Summing over $k$, we recover the definition of $\beta_{i,1}$:
\begin{align} \label{eq:EXc}
	\mathbb{E}[\Xc_i] = \sum_{k=1}^{m-1} \left(\frac{p_i}{\muc_i}\right)^k \frac{Z_{n,m-1-k}}{Z_{n,m-1}} = \beta_{i,1}.
\end{align}

Next, we proceed to compute the expected downlink delay $\mathbb{E}[\Xd_i]$ using the definition of expectation directly:
\begin{align}
	\mathbb{E}[\Xd_i] 
	&= \sum_{x \in \mathcal{X}_{n,m-1}} \xd_i \,\pi_{n,m-1}(x) \nonumber \\[0.5ex]
	&= \frac{1}{Z_{n,m-1}} \sum_{x \in \mathcal{X}_{n,m-1}} \xd_i 
	\prod_{j=1}^n \left(\frac{p_j}{\muc_j}\right)^{\xc_j} \frac{1}{\xd_j!}\left(\frac{p_j}{\mud_j}\right)^{\xd_j} \frac{1}{\xd_j!}\left(\frac{p_j}{\muu_j}\right)^{\xd_j} \nonumber \\[0.5ex]
	&= \frac{1}{Z_{n,m-1}} \frac{p_i}{\mud_i}
	\underbrace{\sum_{\substack{x \in \mathcal{X}_{n,m-1} \\ \xd_i \ge 1}} 
		\prod_{j=1}^n \left(\frac{p_j}{\muc_j}\right)^{\xc_j} \frac{1}{(\xd_j-\mathbf{1}\{j=i\})!}\left(\frac{p_j}{\muu_j}\right)^{\xd_j - \mathbf{1}\{j=i\}} \frac{1}{\xu_j!}\left(\frac{p_j}{\muu_j}\right)^{\xu_j}}_{= Z_{n,m-2}}  \nonumber \\[0.5ex]
	&= \frac{p_i}{\mud_i}\,\frac{Z_{n,m-2}}{Z_{n,m-1}}. \label{eq:EXd}
\end{align}

The simplification in the third step uses the identity $\xd_i \cdot \frac{1}{\xd_i!} = \frac{1}{(\xd_i-1)!}$ and factors out $p_i/\mud_i$. The remaining sum corresponds to the normalizing constant of a system with population $m-1-1 = m-2$, via the change of variables $y = x - \mathbf{e}^{\mathrm{d}}_i$.

Due to the symmetry between the uplink and downlink processes in the product-form solution, the derivation for the uplink delay $\mathbb{E}[\Xu_i]$ is identical to that of $\Xd_i$. Thus, we directly obtain:
\begin{align} \label{eq:EXu}
	\mathbb{E}[\Xu_i] = \frac{p_i}{\muu_i} \frac{Z_{n,m-2}}{Z_{n,m-1}}.
\end{align}

Finally, summing the components derived in \eqref{eq:EXc}, \eqref{eq:EXd}, and \eqref{eq:EXu} yields:
\begin{align*}
	\sum_{s \in \{\mathrm{c},\mathrm{d},\mathrm{u}\}} \mathbb{E}[X_i^s] 
	= \beta_{i,1} + \left( \frac{p_i}{\mud_i} + \frac{p_i}{\muu_i} \right) \frac{Z_{n,m-2}}{Z_{n,m-1}}
	= \beta_{i,1} + \frac{Z_{n,m-2}}{Z_{n,m-1}}\gamma_i,
\end{align*}
where we have substituted $\gamma_i = p_i (1/\mud_i + 1/\muu_i)$. This completes the proof of \Cref{eq:EX}.

\subsection{Proof of Equation~\eqref{eq:EXX}} \label{proof_eq:EXX}

To prove \Cref{eq:EXX}, we decompose the sum on the left-hand side into three distinct categories of interaction: queue-queue, queue-delay, and delay-delay correlations. We derive the closed-form expression for each category separately before combining them.

\subsubsection{Queue-Queue Correlations ($\mathbb{E}[\Xc_i \Xc_j]$)}

We begin by computing the joint expectation of the queue lengths $\mathbb{E}[\Xc_i \Xc_j]$ for any pair of clients $i,j$. Consider first the case where the clients are distinct ($i \neq j$). Using the identity $\xc_i \xc_j = \sum_{k=1}^{\xc_i} \sum_{\ell=1}^{\xc_j} 1$, we can rewrite the expectation as a sum of tail probabilities:
\begin{align*}
	\mathbb{E}[\Xc_i \Xc_j]
	&= \sum_{x \in \mathcal{X}_{n, m-1}} \xc_i \xc_j \pi_{n, m-1}(x) \\
	&= \sum_{x \in \mathcal{X}_{n, m-1}} \sum_{k = 1}^{\xc_i} \sum_{\ell = 1}^{\xc_j} \pi_{n, m-1}(x) 
	= \sum_{\substack{k, \ell = 1 \\ k + \ell \le m-1}}^{m-1} \sum_{\substack{x \in \mathcal{X}_{n, m-1} \\ \xc_i \ge k, \xc_j \ge \ell}} \pi_{n, m-1}(x).
\end{align*}
Substituting the product-form solution $\pi_{n, m-1}(x)$ from \Cref{eq:pi} into this summation yields:
\begin{align}
	\mathbb{E}[\Xc_i \Xc_j]
	&= \frac{1}{Z_{n, m-1}}
	\sum_{\substack{k, \ell = 1 \\ k + \ell \le m-1}}^{m-2}
	\sum_{\substack{x \in \mathcal{X}_{n, m-1} \\ \xc_i \ge k, \xc_j \ge \ell}}
	\prod_{r=1}^n \left( \frac{p_r}{\muc_r} \right)^{\xc_r} \frac{1}{\xd_r!}\left(\frac{p_r}{\mud_r}\right)^{\xd_r} \frac{1}{\xu_r!}\left(\frac{p_r}{\muu_r}\right)^{\xu_r} \nonumber \\
	&\overset{(\text{a})}= \frac{1}{Z_{n, m-1}}
	\sum_{\substack{k, \ell = 1 \\ k + \ell \le m-1}}^{m-2}
	\left( \frac{p_i}{\muc_i} \right)^{k}
	\left( \frac{p_j}{\muc_j} \right)^{\ell}
	\underbrace{\sum_{y \in \mathcal{X}_{n, m - 1 - k - \ell}}
		\prod_{r = 1}^n \left( \frac{p_r}{\muc_r} \right)^{\yc_r} \frac{1}{\yd_r!}\left(\frac{p_r}{\mud_r}\right)^{\yd_r} \frac{1}{\yu_r!}\left(\frac{p_r}{\muu_r}\right)^{\yu_r}}_{= Z_{n, m-1-k-\ell}} \nonumber \\
	&= \sum_{k=1}^{m-2} \sum_{\ell = 1}^{m-1-k} \left(\frac{p_i}{\muc_i}\right)^{k} \left(\frac{p_j}{\muc_j}\right)^{\ell} \frac{Z_{n,m-1-k-\ell}}{Z_{n,m-1}}
	= \alpha_{i,j}. \label{eq:EXcc}
\end{align}
Step (a) follows by performing the change of variables $y = x - k \ec_i - \ell \ec_j$, where $\ec_i$ denotes the $3n$-dimensional canonical unit vector corresponding to the component $\xc_i$ within the state vector $x$. The inner sum
resolves to the normalizing constant of a system with $m-1-k-\ell$ tasks.

Conversely, for the second moment of a single queue ($i=j$), we utilize the identity $(\xc_i)^2 = \sum_{k=1}^{\xc_i} (2k-1)$. This allows us to rewrite the expectation as a weighted sum of tail probabilities:
\begin{align*}
	\mathbb{E}[(\Xc_i)^2]
	&= \sum_{x \in \mathcal{X}_{n, m-1}} (\xc_i)^2 \pi_{n, m-1}(x) \\
	&= \sum_{x \in \mathcal{X}_{n, m-1}} \sum_{k = 1}^{\xc_i} (2k-1) \pi_{n, m-1}(x) 
	= \sum_{k=1}^{m-1} (2k-1) \sum_{\substack{x \in \mathcal{X}_{n, m-1} \\ \xc_i \ge k}} \pi_{n, m-1}(x).
\end{align*}
Applying the product-form solution from \Cref{eq:pi} once more, we obtain:
\begin{align}
	\mathbb{E}[(\Xc_i)^2]
	&= \frac{1}{Z_{n, m-1}}
	\sum_{k=1}^{m-1} (2k-1)
	\sum_{\substack{x \in \mathcal{X}_{n, m-1} \\ \xc_i \ge k}}
	\prod_{r=1}^n \left( \frac{p_r}{\muc_r} \right)^{\xc_r} \frac{1}{\xd_r!}\left(\frac{p_r}{\mud_r}\right)^{\xd_r} \frac{1}{\xu_r!}\left(\frac{p_r}{\muu_r}\right)^{\xu_r} \nonumber \\
	&\overset{(\text{b})}= \frac{1}{Z_{n, m-1}}
	\sum_{k=1}^{m-1} (2k-1)
	\left( \frac{p_i}{\muc_i} \right)^{k}
	\underbrace{\sum_{y \in \mathcal{X}_{n, m - 1 - k}}
		\prod_{r = 1}^n \left( \frac{p_r}{\muc_r} \right)^{\yc_r} \frac{1}{\yd_r!}\left(\frac{p_r}{\mud_r}\right)^{\yd_r} \frac{1}{\yu_r!}\left(\frac{p_r}{\muu_r}\right)^{\yu_r}}_{= Z_{n, m-1-k}} \nonumber \\
	&= \sum_{k=1}^{m-1} (2k-1)
	\left( \frac{p_i}{\muc_i} \right)^{k}
	\frac{Z_{n, m-1-k}}{Z_{n, m-1}} = \alpha_{i,i}. \label{eq:EXc2}
\end{align}
Step (b) follows from the change of variables $y = x - k \ec_i$.

\subsubsection{Queue-Delay Correlations ($\mathbb{E}[\Xc_i \Xd_j]$ and $\mathbb{E}[\Xc_i \Xu_j]$)}

Next, we evaluate the interaction between the computation queue of client $i$ and the communication delays of client $j$. We focus on the term $\mathbb{E}[\Xc_i \Xd_j]$:
\begin{align*}
	\mathbb{E}[\Xc_i \Xd_j]
	&= \sum_{x \in \mathcal{X}_{n, m-1}} \sum_{k = 1}^{\xc_i} \xd_j \pi_{n, m-1}(x) 
	= \sum_{k = 1}^{m-2} \sum_{\substack{x \in \mathcal{X}_{n, m-1} \\ \xc_i \ge k}} \xd_j \pi_{n, m-1}(x) \\
	&= \frac1{Z_{n, m-1}}
	\sum_{k = 1}^{m-2}
	\sum_{\substack{x \in \cX_{n, m-1} \\ \xc_i \ge k, \xd_j \ge 1}}
	\prod_{r=1}^n \left( \frac{p_r}{\muc_r} \right)^{\xc_r} \frac{1}{(\xd_r-\mathbf{1}\{r=j\})!}\left(\frac{p_r}{\mud_r}\right)^{\xd_r} \frac{1}{\xu_r!}\left(\frac{p_r}{\muu_r}\right)^{\xu_r} \\
	&\overset{(\text{c})}= \frac{1}{Z_{n, m-1}} \sum_{k = 1}^{m-2} \left( \frac{p_i}{\muc_i} \right)^{k} \frac{p_j}{\mud_j} \underbrace{\sum_{y \in \mathcal{X}_{n, m - 2 - k}} \prod_{r=1}^n \left( \frac{p_r}{\muc_r} \right)^{\yc_r} \frac{1}{\yd_r!}\left(\frac{p_r}{\mud_r}\right)^{\yd_r} \frac{1}{\yu_r!}\left(\frac{p_r}{\muu_r}\right)^{\yu_r}}_{= Z_{n, m-2-k}} \\
	&= \frac{p_j}{\mud_j} \sum_{k = 1}^{m-2} \left( \frac{p_i}{\muc_i} \right)^{k} \frac{Z_{n, m-2-k}}{Z_{n, m-1}} = \beta_{i,2} \frac{p_j}{\mud_j}.
\end{align*}
Here, step (c) uses the change of variables $y = x - k \ec_i - \ed_j$ along with the fact that $\xd_j \ge 1$. By symmetry, the uplink correlation is $\mathbb{E}[\Xc_i \Xu_j] = \beta_{i,2} \frac{p_j}{\muu_j}$. Summing these components gives the total interaction between client $i$'s computation and client $j$'s communication:
\begin{align} \label{eq:EXcd}
	\mathbb{E}[\Xc_i (\Xd_j + \Xu_j)] = \beta_{i,2} p_j \left(\frac{1}{\mud_j}+\frac{1}{\muu_j}\right) = \beta_{i,2} \gamma_j, \quad i,j \in \{1,\ldots,n\}.
\end{align}
Similarly, swapping the roles of $i$ and $j$ gives the symmetric interaction:
\begin{align} \label{eq:EXcd_2}
	\mathbb{E}[(\Xd_i + \Xu_i) \Xc_j] = \beta_{j,2} \gamma_i, \quad i,j \in \{1,\ldots,n\}.
\end{align}

\subsubsection{Delay-Delay Correlations} \label{sec:delay-delay}

Finally, we turn to the correlations between delay nodes. We first calculate the expectation for pairs of distinct queues. For the cross-interaction $\mathbb{E}[\Xd_i \Xu_j]$, the downlink and uplink queues are distinct entities for all $i,j \in \{1,\ldots,n\}$. The calculation proceeds as follows:
\begin{align*}
	\mathbb{E}[\Xd_i \Xu_j]
	&= \frac{1}{Z_{n, m-1}} \sum_{\substack{x \in \mathcal{X}_{n, m-1} \\ \xd_i \ge 1, \xu_j \ge 1}} \xd_i \xu_j \pi_{n, m-1}(x) \\
	&\overset{(\text{d})}= \frac{1}{Z_{n, m-1}} \frac{p_i}{\mud_i} \frac{p_j}{\muu_j} \underbrace{\sum_{y \in \mathcal{X}_{n, m - 3}} \prod_{r=1}^n \left( \frac{p_r}{\muc_r} \right)^{\yc_r} \frac{1}{\yd_r!}\left(\frac{p_r}{\mud_r}\right)^{\yd_r} \frac{1}{\yu_r!}\left(\frac{p_r}{\muu_r}\right)^{\yu_r}}_{= Z_{n, m-3}}  \\
	&= \frac{Z_{n, m-3}}{Z_{n, m-1}} \frac{p_i p_j}{\mud_i \muu_j},
\end{align*}
where (d) follows from the change of variable $y = x - \ed_i - \eu_j$. Similarly, for distinct clients ($i \neq j$), the correlations between delay nodes of the same type follow the same pattern:
\begin{align*}
	\mathbb{E}[\Xd_i \Xd_j] = \frac{Z_{n, m-3}}{Z_{n, m-1}} \frac{p_i p_j}{\mud_i \mud_j}, \quad \text{ and }
	\mathbb{E}[\Xu_i \Xu_j] = \frac{Z_{n, m-3}}{Z_{n, m-1}} \frac{p_i p_j}{\muu_i \muu_j}.
\end{align*}
However, when computing the second moment of a single queue (e.g., $\Xd_i$), the factorial moment expansion applies:
\begin{align*}
	\mathbb{E}[(\Xd_i)^2] &= \frac{1}{Z_{n, m-1}} \sum_{\substack{x \in \mathcal{X}_{n, m-1} \xd_i \ge 1}} (\xd_i)^2 \pi_{n, m-1}(x) \\
	&= \frac{1}{Z_{n, m-1}} \sum_{\substack{x \in \mathcal{X}_{n, m-1} \\ \xd_i \ge 2}} \xd_i (\xd_i-1) \pi_{n, m-1}(x) + \frac{1}{Z_{n, m-1}} \sum_{\substack{x \in \mathcal{X}_{n, m-1} \\ \xd_i \ge 1}} \xd_i \pi_{n, m-1}(x) \\
	&\overset{(\text{e})}= \frac{1}{Z_{n, m-1}} (\frac{p_i}{\mud_i})^2 \underbrace{\sum_{y \in \mathcal{X}_{n, m - 3}} \prod_{r=1}^n \left( \frac{p_r}{\muc_r} \right)^{\yc_r} \frac{1}{\yd_r!}\left(\frac{p_r}{\mud_r}\right)^{\yd_r} \frac{1}{\yu_r!}\left(\frac{p_r}{\muu_r}\right)^{\yu_r}}_{= Z_{n, m-3}}  +\, \mathbb{E}[\Xd_i] \\
	&\overset{(\text{f})}= \frac{Z_{n, m-3}}{Z_{n, m-1}} (\frac{p_i}{\mud_i})^2 + \frac{Z_{n, m-2}}{Z_{n, m-1}} \frac{p_i}{\mud_i} \\
	&= \frac{1}{Z_{n, m-1}} \frac{p_i}{\mud_i} ( Z_{n, m-3} \frac{p_i}{\mud_i} + Z_{n, m-2}),
\end{align*}
where (e) follows from the change of variable $y = x - 2 \ed_i$ and (f) from the expression for $\mathbb{E}[\Xd_i]$ in \eqref{eq:EXd}.

Generalizing this to all delay pairs, the total delay-delay contribution is:
\begin{align*}
	\sum_{s \in \{\mathrm{d,u}\}} \sum_{r \in \{\mathrm{d,u}\}} \mathbb{E}[X_i^s X_j^r] 
	= \frac{Z_{n, m-3}}{Z_{n, m-1}} p_i p_j \left( \frac{1}{\mud_i} + \frac{1}{\muu_i} \right) \left( \frac{1}{\mud_j} + \frac{1}{\muu_j} \right)
	= \frac{Z_{n, m-3}}{Z_{n, m-1}} \gamma_i \gamma_j,
\end{align*}
if $i \neq j$. Otherwise, if $i=j$, we have:
\begin{align*}
	\sum_{s \in \{\mathrm{d,u}\}} \sum_{r \in \{\mathrm{d,u}\}} \mathbb{E}[X_i^s X_i^r] 
	&= \frac{Z_{n, m-3}}{Z_{n, m-1}} (p_i)^2 \left( \frac{1}{\mud_i} + \frac{1}{\muu_i} \right)^2 +  \frac{Z_{n, m-2}}{Z_{n, m-1}} \left( \frac{1}{\mud_j} + \frac{1}{\muu_i} \right) \\
	&= \frac{\gamma_i}{Z_{n, m-1}} (\gamma_i Z_{n, m-3} + Z_{n, m-2}).
\end{align*}
In both scenarios, this simplifies to the unified term:
\begin{align} \label{eq:EXdd}
	\sum_{s \in \{\mathrm{d,u}\}} \sum_{r \in \{\mathrm{d,u}\}} \mathbb{E}[X_i^s X_j^r] 
	= \frac{Z_{n, m-3}}{Z_{n, m-1}} \gamma_i \gamma_j + \mathbf{1}_{\{i=j\}} \frac{Z_{n, m-2}}{Z_{n, m-1}} \gamma_i
	= \psi_{i,j}.
\end{align}

Summing the contributions from the three components: the queue-queue term $\alpha_{i,j}$ from \eqref{eq:EXcc} and \eqref{eq:EXc2}, the queue-delay interactions $\beta_{i,2}\gamma_j + \beta_{j,2}\gamma_i$ from \eqref{eq:EXcd} and \eqref{eq:EXcd_2}, and the delay-delay term $\psi_{i,j}$ from \eqref{eq:EXdd}, yields the final result in \Cref{eq:EXX}.

\section{Proof of Theorem~\ref{th:opt_steps}} \label{proof_th:opt_steps}

To prove Theorem~\ref{th:opt_steps}, we adopt the virtual iterates framework introduced in~\cite{koloskova2022sharper} and further developed in~\cite{leconte2024queueing}. While we follow the general reasoning of \cite[Theorem 1]{leconte2024queueing} to establish the expression for round complexity $K_\epsilon$, we introduce a critical modification. Specifically, we address a technical flaw in the original proof, where the authors implicitly assume that the relative delay $D_{i,k}$ is independent of $\nabla f(w_k)$ for all $k$. This assumption does not necessarily hold under the closed Jackson network dynamics described in \Cref{sec:model_description} (and \Cref{sec:model_description2}). Our proof rigorously accounts for this dependency.

Let $S_0$ be the multiset of initially selected clients, where each client appears as many times as the number of tasks it gets initially assigned. By construction, the cardinality of $S_0$ is $m$.
To simplify the notation, we denote the stochastic gradient evaluated on a local data sample $\zeta_i$ as $g_i(w)$ rather than $g_i(w, \zeta_i)$.
For $k \in \mathbb{N}_{>0}$, the virtual iterates $v_k$ are defined recursively as:
\begin{align*}
	\begin{cases}
		v_0 = \param_0, \\
		v_1 = v_0 - \eta \sum_{i \in S_0} \frac{1}{n p_i} g_i(\param_0), \\
		v_{k+1} = v_k - \frac{\eta}{n p_{A_{k}}} g_{A_{k}}(\param_{k}), & k \geq 1.
	\end{cases}
\end{align*}

The difference between $v_k$ and $\param_k$ captures the in-flight computation tasks dispatched by the \gls{CS}. These are tasks for which the resulting gradients have not yet been received at the end of step $k-1$ (i.e., immediately after receiving the gradient from client $C_{k-1}$ and before dispatching a new task to client $A_k$).

We begin by recalling \cite[Lemma 4]{leconte2024queueing}, which holds under our exact problem setting:
\begin{lemma} \label{lem:bound1}
	In the framework of \Cref{sec:model_description}, for any learning rate satisfying $\eta \leq \frac{n^2}{8L \sum_{i=1}^n p_i^{-1}}$, the following inequality holds:
	\begin{align*}
		\frac{1}{K} \sum_{k=0}^{K-1} \bEp\left[\|\nabla f(\param_k)\|^2 \right] \leq
		\frac{4 \Delta}{\eta K} +
		\frac{4 \eta L (2M^2 + \sigma^2)}{n^2} \sum_{i=1}^n \frac{1}{p_i}
		+ \frac{2 L^2}{K} \sum_{k=0}^{K-1} \bEp\left[\|v_k - \param_k\|^2 \right].
	\end{align*}
\end{lemma}

With this lemma established, we proceed to derive the precise bound for the staleness error $\|v_k - \param_k\|^2$ to uncover the expression for $K_\epsilon$.

Let $\mathcal{Q}_k$ denote the multiset of tuples $(i, j)$, where $i$ is the index of a client whose gradient computation task has not yet been applied by \gls{CS} at the beginning of round~$k$ (i.e., it is currently computing, queued, or in transit), and $j$ is the round in which this task was initially dispatched.

The multiset $\mathcal{Q}_k$ can be defined recursively as follows:
\begin{align*}
	\mathcal{Q}_1 &= \{(i, 0) \mid i \in S_0, i \neq C_0\}, \\
	\mathcal{Q}_{k+1} &= \left( \mathcal{Q}_k \setminus \{(C_k, I_k)\} \right) \cup \{(A_{k}, k)\} \quad \text{for all } k \ge 1.
\end{align*}
Recall that $I_k$ denotes the round during which the gradient received at the end of round~$k$ was initially dispatched to client~$C_k$.

Using \cite[Lemma 9]{leconte2024queueing} to express the difference $v_k - \param_k$, and subsequently applying the Cauchy-Schwarz inequality, we have for each $k \in \{1, \ldots, K\}$:
\begin{align*}
	\|v_k - \param_k\|^2 = \left\| - \eta \sum_{(i,j) \in \mathcal{Q}_k} \frac{1}{n p_i} g_i(\param_j) \right\|^2 
	\leq \eta^2 |\mathcal{Q}_k| \sum_{(i,j) \in \mathcal{Q}_k} \frac{1}{n^2 p_i^2} \|g_i(\param_j)\|^2,
\end{align*}
where $|\mathcal{Q}_k|$ denotes the cardinality of the multiset $\mathcal{Q}_k$. Notably, for each $k \in \mathbb{N}_{>0}$, the number of in-flight tasks immediately after the $k$-th model-parameter update is strictly $|\mathcal{Q}_k| = m-1$.

Averaging this error over all $K$ rounds and rearranging the summation indices, we obtain:
\begin{align*}
	\frac{1}{K} \sum_{k=0}^{K-1} \|v_k - \param_k\|^2
	&\leq \frac{\eta^2 (m-1)}{K} \sum_{k=1}^{K-1} \sum_{(i,j) \in \mathcal{Q}_k} \frac{1}{n^2 p_i^2} \|g_i(\param_j)\|^2, \\
	&= \frac{\eta^2 (m-1)}{K} \sum_{k=1}^{K-1} \sum_{i=1}^{n} \frac{1}{n^2 p_i^2} \sum_{j=0}^{k-1} \|g_i(\param_j)\|^2 \one{(i,j) \in \mathcal{Q}_k}, \\
	&= \frac{\eta^2 (m-1)}{K} \sum_{i=1}^{n} \frac{1}{n^2 p_i^2} \sum_{j=0}^{K-2} \|g_i(\param_j)\|^2 \sum_{k=j+1}^{K-1} \one{(i,j) \in \mathcal{Q}_k},
\end{align*}
where the first equality follows by rearranging the sum symbols,
and the second equality by exchanging the sums over~$k$ and~$j$.

We observe that for all $j \in \{1,2,\ldots,K-1\}$ and $i \in \{1,2,\ldots,n\}$, the inner summation over the indicator function is bounded by the relative delay:
\begin{align*}
	\sum_{k=j+1}^{K-1} \one{(i,j) \in \mathcal{Q}_k} \leq \sum_{k=j+1}^{\infty} \one{(i,j) \in \mathcal{Q}_k} = D_{i,j},
\end{align*}
where $D_{i,j}$ is defined as the number of model updates performed by the \gls{CS} between (i) the time a task is assigned to client~$i$ at the beginning of round~$j$, and (ii) the time the resulting gradient is applied by the \gls{CS} (provided that $A_j=i$; otherwise, $D_{i,j} = 0$).

Similarly, for the initial tasks at $j=0$, we have:
\begin{align*}
	\sum_{k=1}^{K-1} \one{(i,0) \in \mathcal{Q}_k} \leq \sum_{k=1}^{\infty} \one{(i,0) \in \mathcal{Q}_k} = D_{i,0}.
\end{align*}
With a slight abuse of notation, $D_{i,0}$ captures the number of updates performed by the \gls{CS} between time~$0$ and the time the resulting gradients of all tasks initially assigned to client~$i \in S_0$ are applied (otherwise, $D_{i,0} = 0$).

Taking the expectation with respect to the Palm measure, we arrive at:
\begin{align} \label{ineq:giSij}
	\bEp \left[ \frac{1}{K} \sum_{k=0}^{K-1}\|v_k - \param_k\|^2 \right]
	\leq \frac{\eta^2 (m-1)}{K} \sum_{i=1}^{n} \frac{1}{n^2 p_i^2} \sum_{j=0}^{K-1} \bEp \left[ \|g_i(\param_j)\|^2 D_{i,j} \right].
\end{align}

We now examine the term $\bEp [ \|g_i(\param_j)\|^2 D_{i,j} ]$. Crucially, given $\param_j$, the stochastic gradient $g_i(\param_j)$ depends solely on client~$i$'s local data, while the relative delay $D_{i,j}$ depends purely on network dynamics, which operate independently of the client's local data sampling. Therefore, $D_{i,j}$ and $g_i(\param_j)$ are conditionally independent given $\param_j$.
By applying the law of total expectation, we can safely split the terms\footnote{While \cite{leconte2024queueing} implicitly assumes that $w_j$ and $D_{i, j}$ are independent, this dependency generally holds; therefore, we adopt a different approach to rigorously bound the expectation $\bEp \left[ \|g_i(w_j)\|^2 D_{i,j} \right]$.}:
\begin{align}
	\nonumber
	\bEp \left[ \|g_i(\param_j)\|^2 D_{i,j} \right] 
	&= \bEp \left( \bEp \left[ \|g_i(\param_j)\|^2 D_{i,j} \mid \param_j \right] \right) \\
	\label{eq:giSij}
	&= \bEp \left( \bEp \left[ \|g_i(\param_j)\|^2 \mid \param_j \right] \bEp \left[ D_{i,j} \mid \param_j \right] \right).
\end{align}

To bound the gradient norm, we utilize the inequality $\|a+b\|^2 \leq 2\|a\|^2 + 2\|b\|^2$ alongside the model assumptions:
\begin{align*}
	\bEp \left[ \|g_i(\param_j)\|^2 \mid \param_j \right] &= \bEp \left[ \|g_i(\param_j) - \nabla f_i(\param_j) + \nabla f_i(\param_j)\|^2 \mid \param_j \right] \\
	&\leq 2 \bEp \left[ \|g_i(\param_j) - \nabla f_i(\param_j)\|^2 \mid \param_j \right] + 2 \bEp \left[ \|\nabla f_i(\param_j)\|^2 \mid \param_j \right] \\
	&\leq 2 \sigma^2 + 2 \bEp \left[ \|\nabla f_i(\param_j)\|^2 \mid \param_j \right] \quad \text{(using Assumption~\textbf{A3})} \\
	&\leq 2 \sigma^2 + 2 G^2 \quad \text{(using Assumption~\textbf{A5})}.
\end{align*}

Substituting this bound back into \eqref{eq:giSij}, we find:
\begin{align*}
	\bEp \left[ \|g_i(\param_j)\|^2 D_{i,j} \right] \leq 2 (\sigma^2 + G^2) \bEp \left[ D_{i,j} \right].
\end{align*}

Plugging this result into inequality~\eqref{ineq:giSij} yields the explicit bound for the in-flight staleness error:
\begin{align} \label{ineq:inFlight}
	\bEp \left[ \frac{1}{K} \sum_{k=0}^{K-1}\|v_k - \param_k\|^2 \right]
	\leq \frac{2 \eta^2 (m-1)}{K} \sum_{i=1}^{n} \frac{\sigma^2 + G^2}{n^2 p_i^2} \sum_{j=0}^{K-1} \bEp \left[ D_{i,j} \right].
\end{align}

Incorporating inequality~\eqref{ineq:inFlight} into the initial bound established in Lemma~\ref{lem:bound1}, we obtain:
\begin{multline*}
	\frac{1}{K}\sum_{k=0}^{K-1} \bEp \left[ \|\nabla f(\param_k)\|^2 \right]
	\leq \frac{4 \Delta}{\eta K}
	+ \frac{4 \eta L (2M^2 + \sigma^2)}{n^2} \sum_{i=1}^n \frac{1}{p_i} \\
	+ \frac{4 \eta^2 L^2 (m-1)(\sigma^2 + G^2)}{n^2} \sum_{i=1}^{n} \frac{1}{p_i^2} \left( \frac{1}{K}\sum_{k=0}^{K-1} \bEp \left[ D_{i,k} \right] \right).
\end{multline*}

Due to the stationarity of the sequence $D_{i,k}$, we have $\bEp \left[ D_{i,k} \right] = \bEp \left[ D_i \right]$ for all $k \in \{1,2,\ldots,K-1\}$.  and the inequality above simplifies to:
\begin{multline*}
	\frac{1}{K}\sum_{k=0}^{K-1} \bEp \left[ \|\nabla f(\param_k)\|^2 \right]
	\leq \frac{4 \Delta}{\eta K}
	+ \frac{4 \eta L (2M^2 + \sigma^2)}{n^2} \sum_{i=1}^n \frac{1}{p_i} \\
	+ \frac{4 \eta^2 L^2 (m-1)(\sigma^2 + G^2)}{n^2} \sum_{i=1}^{n} \frac{1}{p_i^2} \left( \frac{\bEp[D_{i,0}]}{K} + \bEp[D_i]\right).
\end{multline*}
In the remainder, we assume that the learning rate~$\eta$ is such that
\begin{align} \label{eta_bound1}
	\eta \le \frac{n^2}{8L \sum_{i=1}^n p_i^{-1}},
\end{align}
so that the assumption of \Cref{lem:bound1} is satisfied,
and the last inequality is also true.

To guarantee $\epsilon$-accuracy, we bound each of the three terms on the right-hand side by $\epsilon/3$. If we provisionally assume that the total number of rounds $K$ is sufficiently large to ensure the initial delay is dominated by the steady-state delay (i.e., $\frac{\bEp[D_{i,0}]}{K} \le \bEp[D_i]$), we can upper bound the expected delay term $\left(\frac{\bEp[D_{i,0}]}{K} + \bEp[D_i]\right)$ by $2\bEp[D_i]$. Applying this simplification and bounding the second and third terms by $\epsilon/3$ yields the following constraints on the learning rate $\eta$:
\begin{align*}
	\eta \le \frac{n^2 \epsilon}{12 L (2M^2 + \sigma^2) \sum_{i=1}^{n}\frac{1}{p_i}},
	\quad \text{and} \quad
	\eta \le \frac{n}{2 L}
	\sqrt{\frac{\epsilon}{6 (m-1)(\sigma^2 + G^2)
			\sum_{i=1}^{n} \frac{\bEp[D_i]}{p_i^2}}}.
\end{align*}

Returning to the first term, the condition $\frac{4 \Delta}{\eta K} \le \frac{\epsilon}{3}$ dictates that $K \ge \frac{12 \Delta}{\eta \epsilon}$. Because the established bounds require $\eta$ to shrink at least as $\mathcal{O}(\sqrt{\epsilon})$, the required number of rounds $K$ grows unboundedly as $\mathcal{O}(1/\epsilon^{3/2})$ when $\epsilon \to 0$. This confirms our provisional assumption: there rigorously exists a threshold accuracy $\epsilon_0 > 0$ such that for all $0 < \epsilon \le \epsilon_0$, $K$ is large enough to strictly satisfy $K \ge \max_{i} \frac{\bEp[D_{i,0}]}{\bEp[D_i]}$.

Now, by substituting the derived upper bounds on $\eta$, alongside the initial learning rate condition in \eqref{eta_bound1}, into the requirement $K \ge \frac{12 \Delta}{\eta \epsilon}$, we establish the explicit lower bound for the total number of rounds:
\begin{multline*}
	K \ge \frac{12 \Delta}{\epsilon}\,\frac{1}{\eta}
	\ge \frac{12 \Delta}{\epsilon}
	\max \Biggl\{
	\frac{8L}{n^2}\sum_{i=1}^n \frac{1}{p_i},
	\frac{12 L (2M^2 + \sigma^2)}{n^2 \epsilon}\sum_{i=1}^n \frac{1}{p_i},
	\\
	\frac{2L\sqrt{6 (m-1)(\sigma^2 + G^2)}}{n\sqrt{\epsilon}}
	\sqrt{\sum_{i=1}^n \frac{\bEp[D_i]}{p_i^2}}
	\Biggr\}.
\end{multline*}
Hence, the convergence criterion $\frac{1}{K}\sum_{t=0}^{K-1} \mathbb{E}[\|\nabla f(w_t)\|^2] \le \epsilon$ is strictly satisfied for any $0 < \epsilon \le \epsilon_0$ by requiring $K \ge K_\epsilon(p,m)$, where:
\begin{equation*}
	K_{\epsilon}(p,m) = \frac{24 L\Delta}{n\epsilon}
	\Biggl[
	\left(4 + \frac{B}{\epsilon}\right) \sum_{i=1}^n \frac{1}{n p_i}
	+ \left( \frac{C (m-1)}{\epsilon} \sum_{i=1}^n \frac{\bEp[D_i]}{p_i^2} \right)^{1/2}
	\Biggr].
\end{equation*}
The expression for $K_\epsilon(p,m)$ is obtained by upper-bounding the maximum in the preceding inequality with the sum of its terms, simplifying, and defining the constants $B = 6(\sigma^2 + 2M^2)$ and $C = 6(\sigma^2 + G^2)$. This concludes the proof.

\section{Relaxing the Bounded Gradient Assumption (A5) for Theorem~\ref{th:opt_steps}} \label{sec:K_epsilon_unbounded}

Even when the bounded gradient assumption (\textbf{A5}) is removed, the queueing network model detailed in \Cref{sec:model_description} still provides rigorous convergence guarantees for round complexity. We formalize this in the following theorem.

\begin{theorem} \label{th:K_epsilon_unbounded}
	Define the system-wide staleness factor $S_{\mathrm{sys}}$ as:
	\begin{equation}
		S_{\mathrm{sys}} = (m-1)|\muu|\, \sum_{i=1}^n \left( \frac{1}{\mud_i} + \frac{1}{\muu_i} + \frac{m}{\muc_i} \right) p_i^{-2}, \quad \text{where} \quad |\muu| = \sum_{i=1}^n \muu_i.
	\end{equation}
	Then, under Assumptions~\textbf{A1}--\textbf{A4} and the model presented in \Cref{sec:model_description}, there exists $\epsilon_0 > 0$ such that for any target accuracy $\epsilon \in (0, \epsilon_0]$ and any learning rate satisfying
	\[
	\eta \le
	\min \Biggl\{ \frac{n^2}{8L \sum_{i=1}^n p_i^{-1}},\,
	\frac{n^2 \epsilon}{4 L B \sum_{i=1}^n p_i^{-1}},\,
	\frac{n \sqrt{\epsilon}}{2 L} \left(2 B (m-1) \sum_{i=1}^n \frac{\bEp[D_i]}{p_i^2}\right)^{-1/2} ,\,
	\frac{n}{ 4L \sqrt{(m-1)S_{\mathrm{sys}}}}
	\Biggr\},
	\]
	the expected gradient norm satisfies
	\[
	\frac{1}{K} \sum_{t=0}^{K-1} \bEp\!\left[\|\nabla f(w_t)\|^2\right] \le \epsilon
	\]
	whenever $K \ge K_\epsilon(p,m)$, where the required round complexity $K_{\epsilon}(p,m)$ is given by:
	\begin{equation*}
		K_{\epsilon}(p,m) = \frac{96 L\Delta}{n\epsilon}
		\Biggl[
		\left(2 + \frac{B}{\epsilon}\right) \sum_{i=1}^n \frac{1}{n p_i}
		+ \sqrt{(m-1) S_{\mathrm{sys}}}
		+ \left( \frac{B (m-1)}{2 \epsilon} \sum_{i=1}^n \frac{\bEp[D_i]}{p_i^2} \right)^{1/2}
		\Biggr].
	\end{equation*}
	Here, $B = 6(\sigma^2+2M^2)$.
\end{theorem}

This theorem's significance is threefold. First, it establishes rigorous convergence guarantees without the restrictive bounded gradient assumption. Second, it introduces $S_{\mathrm{sys}}$ to physically capture the aggregate staleness penalty across the network. Finally, it formalizes hardware-algorithm co-design by bounding the maximum learning rate $\eta$ inversely to $S_{\mathrm{sys}}$, proving that severe network congestion mathematically necessitates a smaller learning rate for optimization stability.

\begin{proof}
	The proof of \Cref{th:K_epsilon_unbounded} follows the same technique as the proof of \Cref{th:opt_steps}. The derivations remain valid until Equation~\eqref{eq:giSij}. We must examine the expectation term $\bEp \left[ \|g_i(\param_j)\|^2 D_{i,j} \right]$ differently. Specifically, we do not rely on the bounded gradient assumption (\textbf{A5}), and we must explicitly account for the dependence between the relative delay and the model parameters. 
	
	We utilize the inequality $\|a+b\|^2 \leq 2\|a\|^2 + 2\|b\|^2$ alongside the model assumptions:
	\begin{align*}
		\bEp \left[ \|g_i(\param_j)\|^2 \mid \param_j \right] 
		&= \bEp \left[ \|g_i(\param_j) - \nabla f_i(\param_j) + \nabla f_i(\param_j)\|^2 \mid \param_j \right] \\
		&\leq 2 \bEp \left[ \|g_i(\param_j) - \nabla f_i(\param_j)\|^2 \mid \param_j \right] + 2 \bEp \left[ \|\nabla f_i(\param_j)\|^2 \mid \param_j \right] \\
		&\leq 2 \sigma^2 + 2 \bEp \left[ \|\nabla f_i(\param_j)\|^2 \mid \param_j \right] \\
		&\leq 2 \sigma^2 + 4 \bEp \left[ \|\nabla f_i(\param_j) - \nabla f(\param_j)\|^2 \mid \param_j \right] + 4 \bEp \left[ \|\nabla f(\param_j)\|^2 \mid \param_j \right]  \\
		&\leq 2 \sigma^2 + 4 M^2 +  4 \bEp \left[ \|\nabla f(\param_j)\|^2 \mid \param_j \right].
	\end{align*}
	where the second inequality follows from \textbf{A3},
	the third from the classical inequality
	$\|a+b\|^2 \leq 2\|a\|^2 + 2\|b\|^2$,
	and the fourth inequality from \textbf{A4}.
	
	Plugging this back into Equation~\eqref{eq:giSij}, we obtain:
	\begin{equation} \label{eq:g_ijD_ij-unbounded}
		\bEp \left[ \|g_i(\param_j)\|^2 D_{i,j} \right] \leq 2 (\sigma^2 + 2 M^2) \bEp\left[D_{i,j} \right] +  4 \bEp\left[ \|\nabla f(\param_j)\|^2  D_{i,j} \right].
	\end{equation}
	
	Note that the relative delay $D_{i,j}$ is not independent of $\nabla f(\param_j)$. Consequently, we stochastically upper-bound the relative delay $D_{i,j}$ using a new random variable that is independent of $\nabla f(\param_j)$. 
	Specifically, we consider the worst-case sojourn time scenario for the task sent to client~$i$ at round~$j$, which occurs when this task finds $m-1$ tasks already queued at the computation queue of client~$i$ upon being completely downloaded. Due to the memoryless property of exponential service times, the sojourn time $S_i$ of this task at client~$i$ (including servers $\mathrm{d}_i$, $\mathrm{c}_i$, and $\mathrm{u}_i$) is distributed as a sum of independent exponential random variables:
	$$
	S_i = E^{\mathrm{d}}_i + \sum_{k=1}^m E^{\mathrm{c}}_{i,k} + E^{\mathrm{u}}_i,
	$$
	where $E^{\mathrm{d}}_i$ is an exponential random variable with parameter $\mud_i$, $(E^{\mathrm{c}}_{i,k})_{k \geq 1}$ are i.i.d. exponential random variables with parameter $\muc_i$, and $E^{\mathrm{u}}_i$ is an exponential random variable with parameter $\muu_i$. The quantity $D_{i,j}$, representing the number of parameter updates during the sojourn of this task, can be stochastically bounded by the number of events generated by an independent Poisson process~$N$ with intensity~$(m-1) |\muu| = (m-1) \sum_{k=1}^n \muu_k$ over a time interval distributed as $S_i$.
	
	Therefore, for all $j \in \{1,\ldots,K\}$:
	\begin{align*}
		\bEp\left[D_{i,j} \|\nabla f(\param_j)\|^2\right]
		&\leq \bEp\left[N(S_i) \|\nabla f(\param_j)\|^2\right] \\
		&= \bEp\left[N(S_i)\right] \cdot \bEp\left[\|\nabla f(\param_j)\|^2\right] \quad \text{(by independence)} \\
		&= \bEp\left[\|\nabla f(\param_j)\|^2\right] \cdot \left(\frac{1}{\mud_i} + \frac{m}{\muc_i} + \frac{1}{\muu_i}\right) (m-1) \sum_{k=1}^n \muu_k.
	\end{align*}
	
	Substituting this bound into \eqref{eq:g_ijD_ij-unbounded}, we find:
	$$
	\bEp \left[ \|g_i(\param_j)\|^2 D_{i,j} \right] \leq 2 (\sigma^2 + 2 M^2) \bEp\left[D_{i,j} \right] +  4 (m-1) |\muu| \left(\frac{1}{\mud_i} + \frac{m}{\muc_i} + \frac{1}{\muu_i}\right) \bEp\left[\|\nabla f(\param_j)\|^2\right].
	$$
	
	Plugging this result into inequality~\eqref{ineq:giSij} yields the explicit bound for the in-flight staleness error:
	\begin{multline*}
		\bEp \left[ \frac{1}{K} \sum_{k=0}^{K-1}\|v_k - \param_k\|^2 \right]
		\leq \frac{2 \eta^2 (m-1)}{K} \sum_{i=1}^{n} \frac{1}{n^2 p_i^2} \sum_{j=0}^{K-1} \Biggl( (\sigma^2 + 2 M^2) \bEp\left[D_{i,j} \right] \\
		+ 2 (m-1) |\muu| \left(\frac{1}{\mud_i} + \frac{m}{\muc_i} + \frac{1}{\muu_i}\right) \bEp\left[\|\nabla f(\param_j)\|^2\right] \Biggr).
	\end{multline*}
	
	Incorporating the latter inequality into the initial bound established in Lemma~\ref{lem:bound1}, and defining the system-wide staleness factor as $S_{\mathrm{sys}} = (m-1) |\muu|\, \sum_{i=1}^n \left( \frac{1}{\mud_i} + \frac{1}{\muu_i} + \frac{m}{\muc_i} \right) p_i^{-2}$, we obtain:
	\begin{align*}
		&\frac{1}{K}\sum_{k=0}^{K-1} \left( 1 - \frac{8 \eta^2 L^2 (m-1)}{n^2} S_{\mathrm{sys}}  \right) \bEp \left[ \|\nabla f(\param_k)\|^2 \right] \\
		&\leq \frac{4 \Delta}{\eta K} + \frac{4 \eta L (2M^2 + \sigma^2)}{n^2} \sum_{i=1}^n \frac{1}{p_i}
		+ \frac{4 \eta^2 L^2 (m-1)(\sigma^2 + 2 M^2)}{n^2} \sum_{i=1}^{n} \frac{1}{p_i^2} \left( \frac{1}{K}\sum_{k=0}^{K-1} \bEp \left[ D_{i,k} \right] \right).
	\end{align*}
	
	By enforcing the following constraint on the learning rate:
	\begin{align} \label{lr:bound_staleness}
		\eta \le \frac{n}{4L} \frac{1}{\sqrt{(m-1) S_{\mathrm{sys}}}},
	\end{align}
	the factor multiplying the expected gradient norm on the left-hand side is lower-bounded by $1/2$. Multiplying the entire inequality by $2$ yields: 
	\begin{multline*}
		\frac{1}{K}\sum_{k=0}^{K-1} \bEp \left[ \|\nabla f(\param_k)\|^2 \right] 
		\leq \frac{8 \Delta}{\eta K} + \frac{8 \eta L (2M^2 + \sigma^2)}{n^2} \sum_{i=1}^n \frac{1}{p_i} \\
		+ \frac{8 \eta^2 L^2 (m-1)(\sigma^2 + 2 M^2)}{n^2} \sum_{i=1}^{n} \frac{1}{p_i^2} \left( \frac{1}{K}\sum_{k=0}^{K-1} \bEp \left[ D_{i,k} \right] \right).
	\end{multline*}
	
	Due to the stationarity of the sequence $D_{i,k}$, we have $\bEp \left[ D_{i,k} \right] = \bEp \left[ D_i \right]$ for all $k \ge 1$. Thus, the inequality simplifies to:
	\begin{multline*}
		\frac{1}{K}\sum_{k=0}^{K-1} \bEp \left[ \|\nabla f(\param_k)\|^2 \right]
		\leq \frac{8 \Delta}{\eta K} + \frac{8 \eta L (2M^2 + \sigma^2)}{n^2} \sum_{i=1}^n \frac{1}{p_i} \\
		+ \frac{8 \eta^2 L^2 (m-1)(\sigma^2 + 2 M^2)}{n^2} \sum_{i=1}^{n} \frac{1}{p_i^2} \left( \frac{\bEp[D_{i,0}]}{K} + \bEp[D_i]\right).
	\end{multline*}
	
	To guarantee $\epsilon$-accuracy, we bound each of the three terms on the right-hand side by $\epsilon/3$. If we provisionally assume that the total number of rounds $K$ is sufficiently large to ensure the initial delay is dominated by the steady-state delay (i.e., $\frac{\bEp[D_{i,0}]}{K} \le \bEp[D_i]$), we can upper-bound the expected delay term $\left(\frac{\bEp[D_{i,0}]}{K} + \bEp[D_i]\right)$ by $2\bEp[D_i]$. Applying this simplification and bounding the second and third terms by $\epsilon/3$ yields the following constraints on the learning rate $\eta$:
	\begin{align*}
		\eta \le \frac{n^2 \epsilon}{24 L (2M^2 + \sigma^2) \sum_{i=1}^{n}\frac{1}{p_i}},
		\quad \text{and} \quad
		\eta \le \frac{n}{4 L}
		\sqrt{\frac{\epsilon}{3 (m-1)(\sigma^2 + 2M^2) \sum_{i=1}^{n} \frac{\bEp[D_i]}{p_i^2}}}.
	\end{align*}
	
	Returning to the first term, the condition $\frac{8 \Delta}{\eta K} \le \frac{\epsilon}{3}$ dictates that $K \ge \frac{24 \Delta}{\eta \epsilon}$. Because the established bounds require $\eta$ to shrink at least as $\mathcal{O}(\sqrt{\epsilon})$, the required number of rounds $K$ grows unboundedly as $\mathcal{O}(1/\epsilon^{3/2})$ when $\epsilon \to 0$. This confirms our provisional assumption: there rigorously exists a threshold accuracy $\epsilon_0 > 0$ such that for all $0 < \epsilon \le \epsilon_0$, $K$ is large enough to strictly satisfy $K \ge \max_{i} \frac{\bEp[D_{i,0}]}{\bEp[D_i]}$.
	
	Now, by substituting the derived upper bounds on $\eta$, alongside the initial constraints $\eta \le \frac{n^2}{8L \sum p_i^{-1}}$ and \eqref{lr:bound_staleness}, into the requirement $K \ge \frac{24 \Delta}{\eta \epsilon}$, we establish the explicit lower bound for the total number of rounds:
	\begin{multline*}
		K \ge \frac{24 \Delta}{\epsilon\eta}
		\ge \frac{24 \Delta}{\epsilon}
		\max \Biggl\{
		\frac{8L}{n^2}\sum_{i=1}^n \frac{1}{p_i}, \quad
		\frac{4L}{n} \sqrt{(m-1) S_{\mathrm{sys}}}, \quad
		\frac{24 L (2M^2 + \sigma^2)}{n^2 \epsilon}\sum_{i=1}^n \frac{1}{p_i}, \\
		\frac{4L\sqrt{3 (m-1)(\sigma^2 + 2M^2)}}{n\sqrt{\epsilon}}
		\sqrt{\sum_{i=1}^n \frac{\bEp[D_i]}{p_i^2}}
		\Biggr\}.
	\end{multline*}
	
	Hence, the convergence criterion $\frac{1}{K}\sum_{t=0}^{K-1} \mathbb{E}[\|\nabla f(\param_t)\|^2] \le \epsilon$ is strictly satisfied for any $0 < \epsilon \le \epsilon_0$ by requiring $K \ge K_\epsilon(p,m)$, where:
	\begin{equation*}
		K_{\epsilon}(p,m) = \frac{96 L\Delta}{n\epsilon}
		\Biggl[
		\left(2 + \frac{B}{\epsilon}\right) \sum_{i=1}^n \frac{1}{n p_i}
		+ \sqrt{(m-1) S_{\mathrm{sys}}}
		+ \left( \frac{B (m-1)}{2 \epsilon} \sum_{i=1}^n \frac{\bEp[D_i]}{p_i^2} \right)^{1/2}
		\Biggr].
	\end{equation*}
	The expression for $K_\epsilon(p,m)$ is obtained by upper-bounding the maximum in the preceding inequality with the sum of its terms, simplifying, and defining the constant $B = 6(\sigma^2 + 2M^2)$. This concludes the proof.
\end{proof}

\section{Optimizing Round Complexity: Experimental Results} \label{num:optimize-round}

In this section, we empirically validate our theoretical findings by demonstrating that optimizing the routing probability vector $p$ yields significant improvements in round complexity. Adhering to the fully concurrent setting ($m=n$) established in \cite{koloskova2022sharper}, we benchmark our proposed strategy against two standard baselines:

\begin{enumerate}
	\item \textbf{Round-Optimized \texttt{Generalized AsyncSGD}:} Our proposed method, which utilizes the optimal routing vector $p^{\ast K}$ designed to minimize the theoretical round complexity bound $K_\epsilon$ under full concurrency ($m=n$).
	
	\item \textbf{Standard Baseline (\texttt{AsyncSGD})} \cite[Algorithm 2]{koloskova2022sharper}: The conventional asynchronous approach employing uniform routing probabilities ($p^{\text{uni}}$) across all clients and full concurrency ($m=n$).
	
	\item \textbf{Max-Throughput \texttt{Generalized AsyncSGD}:} A heuristic strategy that prioritizes system speed by maximizing the global update frequency ($p^{\ast \lambda}$) under full concurrency ($m=n$), disregarding the impact of data heterogeneity on convergence.
\end{enumerate}

\subsection{Experimental Setup} \label{num:exp_scenario_round}
To evaluate the relative efficiency of our proposed strategy against the baselines, we simulate a throughput-diverse regime comprising $n = 100$ clients, partitioned into five distinct clusters (Types A--E). This configuration captures the diverse hardware and network conditions typical of realistic heterogeneous edge environments, ranging from high-performance super-clients to resource-constrained stragglers. The specific service rates for computation ($\mu^{\text{c}}$), uplink ($\mu^{\text{u}}$), and downlink ($\mu^{\text{d}}$) are detailed in Table~\ref{tab:client_profiles_2}.

\begin{table}[htbp]
	\centering
	\caption{Client clusters, representative hardware profiles, and service rates (Rate $\mu$ in tasks/sec).}
	\label{tab:client_profiles_2}
	\setlength{\tabcolsep}{3pt}
	\begin{tabular}{cl ccc c}
		\toprule
		\textbf{Type} & \textbf{Profile} & $\mu^{\text{c}}$ & $\mu^{\text{u}}$ & $\mu^{\text{d}}$ & \textbf{Count} \\
		\midrule
		A &  Fast compute, slow network & 10.0 & 2.0 & 2.5 & 15 \\
		B &  slow compute, Fast network  & 2.5 & 8.0 & 9.0 & 35 \\
		C & Balanced  & 5.0 & 5.0 & 6.0 & 30 \\
		D & Straggler & 0.5 & 0.8 & 1.1 & 15 \\
		E & High-Performance  & 15.0 & 10.0 & 11.0 & 5 \\
		\bottomrule
	\end{tabular}
\end{table}

We evaluate performance on the EMNIST~\cite{cohen2017emnist} dataset under three distinct data distribution scenarios:
\begin{itemize}
	\item \textbf{Homogeneous (IID):} The training data is uniformly shuffled and partitioned among clients, ensuring that each client possesses an identical class distribution and an equal number of samples.
	
	\item \textbf{Heterogeneous (Non-IID):} We simulate realistic feature and label heterogeneity using a Dirichlet distribution. For each class~$k$, the proportion of samples allocated to client~$j$ is drawn from a vector $q_k \sim \text{Dir}_n(\alpha)$. We set the concentration parameter to $\alpha = 0.2$, following~\cite{yurochkin2019bayesian,li2021federatedlearningnoniiddata}.
	
	\item \textbf{Highly Heterogeneous (Pathological Non-IID):} We enforce extreme label skew where each client holds data from only 3 distinct classes, selected uniformly at random from the total label space. The total number of samples is balanced across all clients.
	
\end{itemize}

\subsection{Optimization: Strategy and Results}
We determine the optimal routing vector $p^{\ast K}$ by minimizing the round complexity $K_{\epsilon}$ via gradient descent. The optimization is performed using the Adam optimizer~\cite{kingma2014adam}, where the gradients of $K_{\epsilon}$ are computed exactly using its closed-form expression \eqref{eq:K_epsilon_bounded} combined with the delay gradient \eqref{eq:gradD} derived in Theorem~\ref{theo:little_law}.
We note that the objective function $K_\epsilon(p,m)$ is non-convex with respect to $p$; hence, to mitigate the risk of converging to suboptimal local minima, we employ multiple random initializations. If one is interested only in finding a solution that outperforms uniform routing, the uniform distribution can serve as an effective starting point for the optimization.

Similarly, for the baseline comparison, we compute the max-throughput routing vector $p^{\ast \lambda}$ by maximizing the system update frequency $\lambda$. This is achieved by performing gradient ascent with Adam, utilizing the closed-form gradient of the update frequency provided in Equation~\eqref{eq:grad_throughput}.
The data-dependent constants $\sigma$, $M$, and $G$ (introduced in Section~\ref{data_model}) are estimated empirically from the training dataset, and we set the target gradient norm bound to $\epsilon = 1$.

The optimized routing probabilities and corresponding performance metrics are reported in \Cref{tab:opt_results_round}. To interpret the optimization behavior, we explicitly analyze the term $\bEp[D_i(p,m=n)] \, p_i^{-2}$, which we define as the \emph{staleness impact factor}. This metric quantifies each client's individual contribution to the staleness component of $K_\epsilon$ (see Equation~\eqref{eq:K_epsilon_bounded}), allowing us to identify specific clusters that disproportionately compromise algorithmic stability.

\begin{table}[htbp]
	\centering
	\caption{Comparison of optimized routing probabilities and staleness impact factors across different client clusters (Types A--E).}
	\label{tab:opt_results_round}
	\setlength{\tabcolsep}{10pt}
	\begin{tabular}{l c c c c c}
		\toprule
		& \multicolumn{2}{c}{\textbf{Routing Probabilities}} 
		& \multicolumn{3}{c}{\textbf{Staleness Impact Factor}} \\
		& \multicolumn{2}{c}{$p \times 100$} 
		& \multicolumn{3}{c}{$\bEp[D_i] \, p_i^{-2} \times 10^{-2} $} \\
		\cmidrule(lr){2-3} \cmidrule(lr){4-6}
		\textbf{Type} 
		& $p^{\ast \lambda}$ 
		& $p^{\ast K}$ 
		& $p^{\ast \lambda}$ 
		& $p^{\text{uni}}$ 
		& $p^{\ast K}$ \\ 
		\midrule
		A & 0.405 & 0.788 & 373.9 & 40.9 & 3.0 \\
		B & 0.422 & 0.754 & 274.8 & 29.1 & 2.0 \\
		C & 1.421 & 0.748 & 75.5 & 23.8 & 1.8 \\
		D & 0.045 & 1.033 & 14,869.5 & 500.2 & 9.7 \\
		E & 7.173 & 0.718 & 8.7 & 10.6 & 0.8 \\
		\bottomrule
	\end{tabular}
\end{table}

The throughput-optimized routing ($p^{\ast \lambda}$) exhibits extreme bias toward fast clients, heavily favoring the super-clients (Type~E) while aggressively down-weighting stragglers (Type~D). Although this strategy successfully maximizes the system update frequency ($\lambda = 151$ updates/sec), it yields an extreme staleness impact for Type~D clients ($\bEp[D_i] \, p_i^{-2} \approx 1.5 \times 10^{6}$). Consequently, when a straggler finally contributes an update, its gradient is excessively stale. This severity degrades convergence and completely offsets the benefits gained from the frequent updates of fast clients.

Standard uniform routing ($p^{\text{uni}}$) fails to account for this system heterogeneity. By treating Type~D stragglers and Type~E super-clients identically, it yields a staleness factor for the stragglers (approximately $5 \times 10^{4}$) that is moderate compared to the extremes of $p^{\ast \lambda}$, yet still highly suboptimal. Our proposed method, $p^{\ast K}$, improves upon this baseline by a factor of $50\times$ (reducing the staleness impact to $\approx 10^{3}$), demonstrating that the equal treatment of highly heterogeneous clients is severely inefficient.

In contrast, the round-optimized routing ($p^{\ast K}$) adopts a counter-intuitive but highly effective strategy: it prioritizes the slowest clients. It assigns Type~D stragglers the highest routing probability (nearly 1.4$\times$ that of Type~E), thereby sacrificing raw update frequency ($\lambda = 2.4$ vs. 41 for $p^{\text{uni}}$) to homogenize the staleness impact across the network. By forcing slow clients to update more frequently, the optimizer ensures that the staleness term associated with stragglers does not dominate the convergence bound.

\subsection{Learning Performance}
To validate the insights obtained from the routing-optimization analysis and corroborate the theoretical results, we simulate the training process on image classification tasks.  
We adopt the experimental setup described in Section~\ref{num:exp_scenario_round} to model the learning dynamics.

To assess robustness to initial transients, the system is initialized out of equilibrium: at $t=0$, the $m=n$ tasks are assigned uniformly at random to the clients’ downlink servers, rather than being drawn from the stationary distribution.  
Moreover, to verify that our conclusions are not artifacts of the exponential service-time assumption, we evaluate performance under three distinct distributions for both computation and communication times:
\begin{enumerate}
	\item[(i)] \textbf{Exponential:} Service times are exponentially distributed, as required by the theoretical analysis.
	\item[(ii)] \textbf{Deterministic:} Service times are fixed and equal to $1/\mu$, corresponding to zero variance.
	\item[(iii)] \textbf{Lognormal:} Service times follow a heavy-tailed distribution with mean $1/\mu$.  
	We set the variance of the underlying normal distribution to $\sigma_N^2 = 1$, reflecting the high variability observed in real-world edge systems while enforcing a fixed coefficient of variation across clients.
\end{enumerate}

Models are trained using a standard multi-class cross-entropy loss, and performance is evaluated on an unseen, label-balanced test set.  
Learning rates are selected via grid search, and additional implementation details are provided in Section~\ref{exp:details}.

The experimental results presented in \Cref{fig:emnist_round} show that \texttt{Generalized AsyncSGD} equipped with the optimized routing vector $p^{\ast K}$ consistently outperforms all baseline methods across the considered scenarios.  
Throughout training, $p^{\ast K}$ achieves higher accuracy and greater stability, even under highly heterogeneous data distributions, leading to a substantial reduction in the number of communication rounds required for convergence.

Notably, although assigning higher routing probabilities to slower clients might intuitively appear to bias the model toward their local data, our simulations indicate the opposite effect.  
The optimized routing effectively mitigates gradient staleness, resulting in smooth and stable convergence.  
In contrast, both the uniform routing strategy and, in particular, the max-throughput strategy exhibit significantly higher loss variance, confirming that blindly maximizing update frequency can induce severe learning instabilities.  
These observations validate that our framework accurately captures and counteracts the adverse effects of staleness arising from heterogeneity in both data and service times.

Moreover, the performance gains achieved by $p^{\ast K}$ are consistent across all evaluated service-time distributions and are especially pronounced in the deterministic case.  
This demonstrates that the robustness of the proposed approach extends well beyond the exponential service-time assumptions underlying the theoretical analysis.

Finally, it is important to contextualize these results with respect to the chosen performance metric.  
Looking ahead to \Cref{sec:discussion}, the advantages of $p^{\ast K}$ are specific to the round-complexity metric $K_\epsilon$.  
By design, minimizing $K_\epsilon$ reduces staleness by prioritizing slower clients, which inherently decreases system throughput.  
Consequently, when performance is evaluated in terms of wall-clock time rather than communication rounds, the relative ranking of the strategies can differ.

\begin{figure}[htbp]
	\centering
	\includegraphics[width=0.49\textwidth]{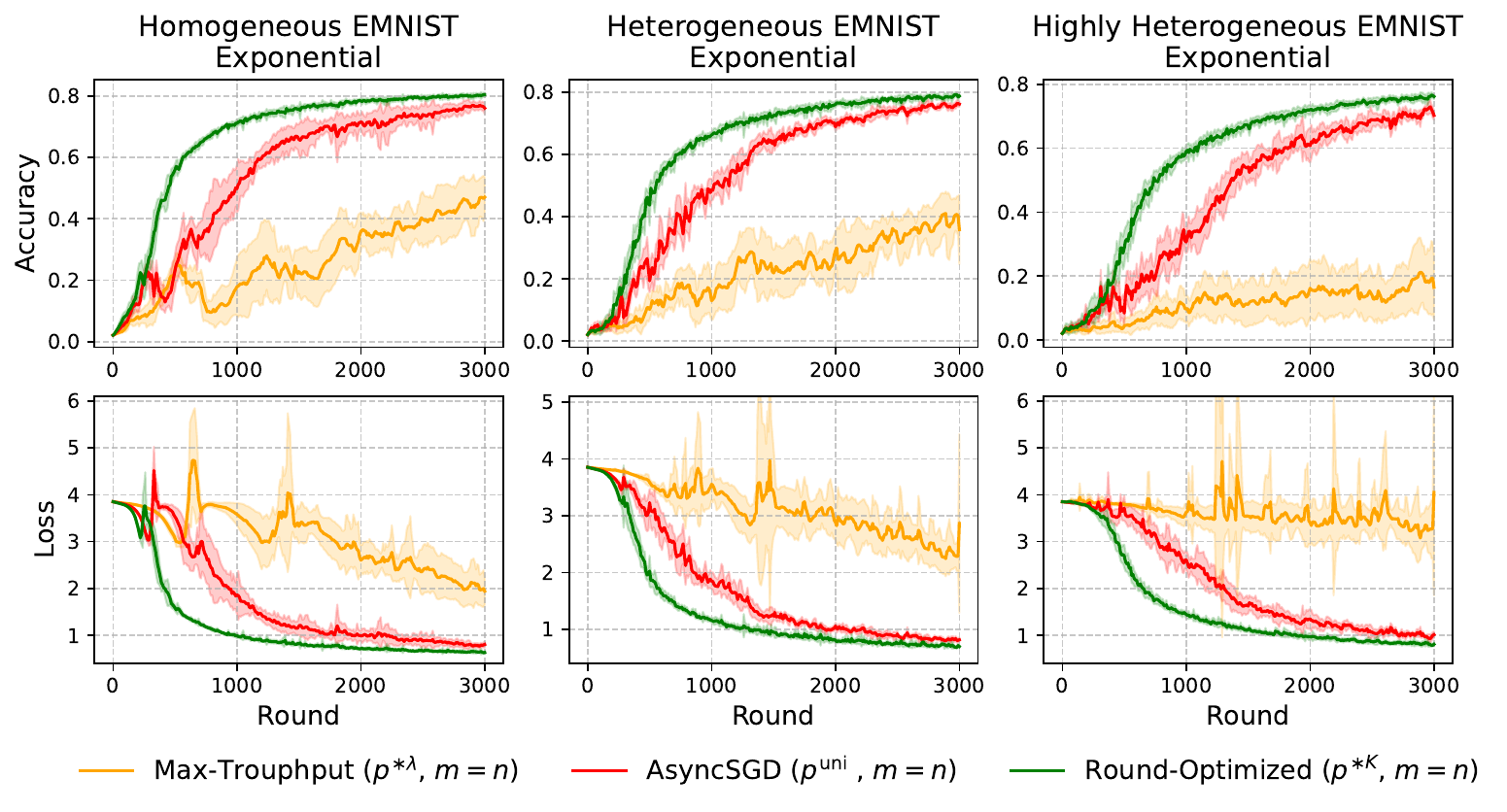}
	\hfill
	\includegraphics[width=0.49\textwidth]{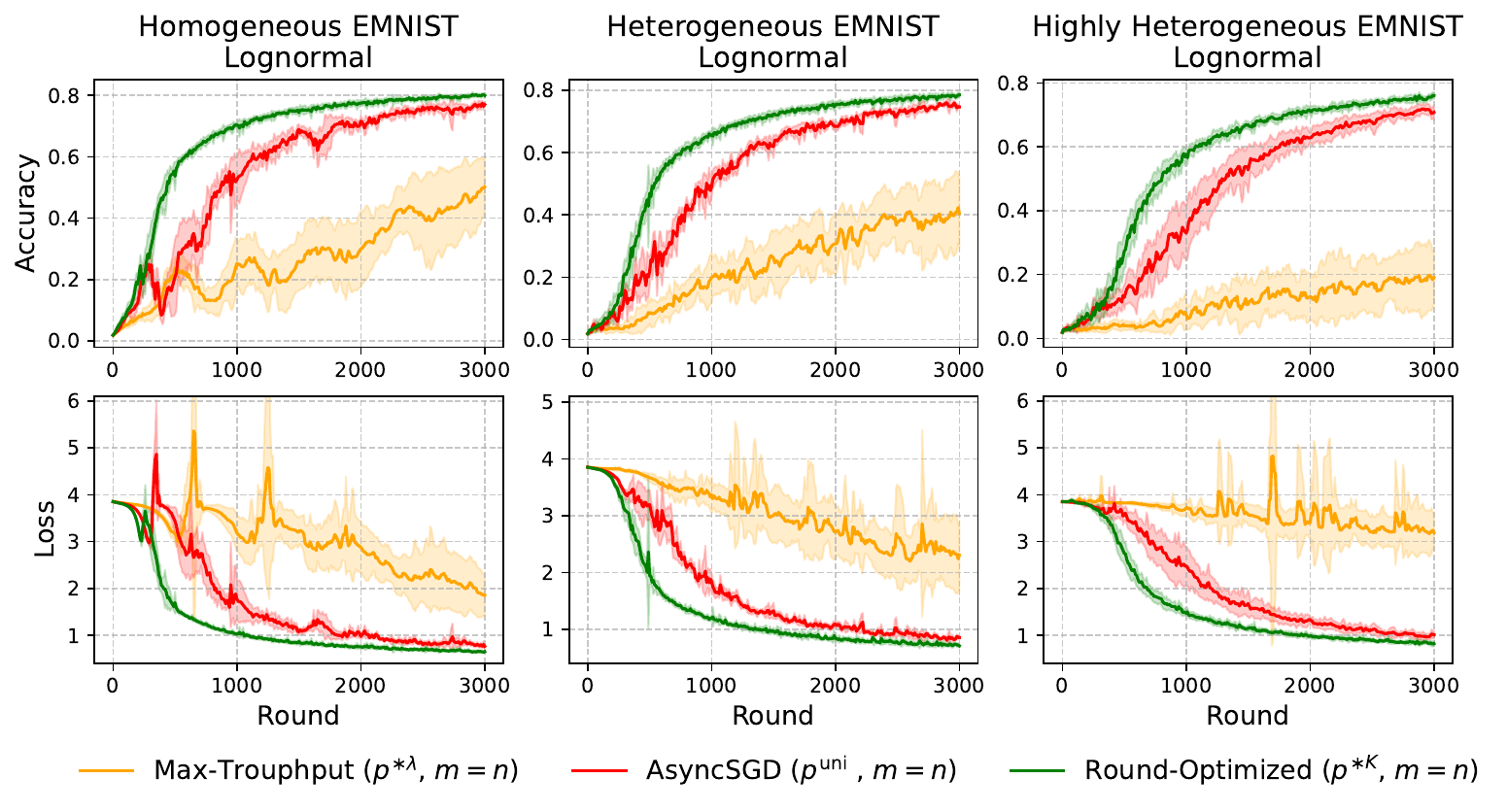}
	\\
	\includegraphics[width=0.49\textwidth]{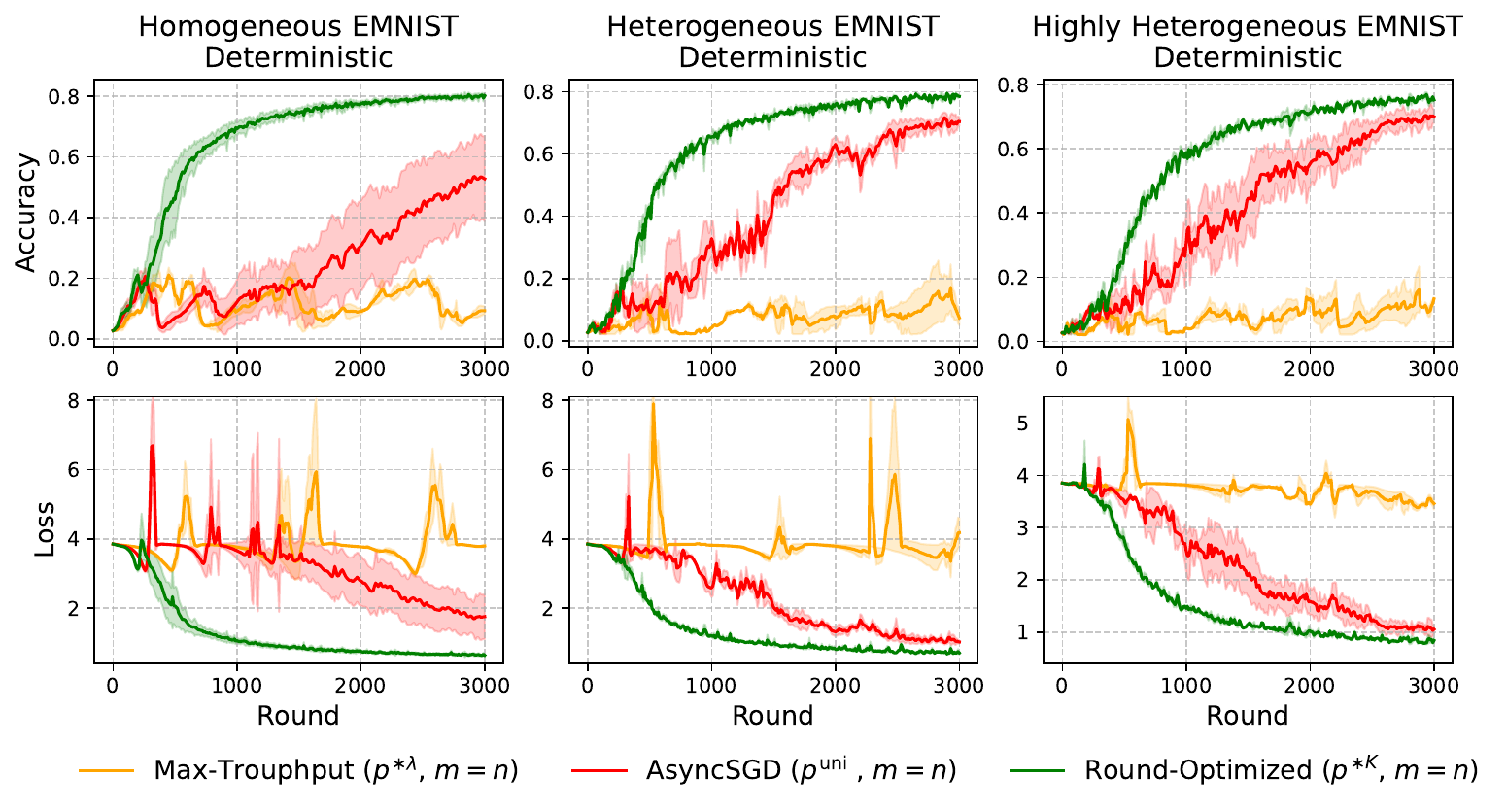}
	\caption{Test set accuracy and loss vs. rounds for the EMNIST dataset ($n=100$, $m=n$) using the client profiles defined in Table~\ref{tab:client_profiles_2}. The subplots evaluate performance under three service time distributions: \textbf{(Top Left)} Exponential, \textbf{(Top Right)} Lognormal, and \textbf{(Bottom)} Deterministic. Each plot compares convergence across three data settings: Homogeneous (IID), Heterogeneous (Dirichlet $\alpha=0.2$), and Highly Heterogeneous (3 classes/client).}
	\label{fig:emnist_round}
\end{figure}

\section{Proof of Proposition~\ref{prop:time-eps}} \label{proof_prop:time-eps}

The proof proceeds in three distinct stages. First, we derive the expected elapsed time for $K_{\epsilon}$ global rounds \eqref{eq:tau}. Next, we establish the explicit formulation for the system throughput \eqref{eq:throughput}. Finally, we compute the exact analytical gradient of this throughput \eqref{eq:grad_throughput}.

\subsection{Proof of Equation~\eqref{eq:tau}}
Recall that $(T_k)_{k \in \mathbb{N}}$ denotes the sequence of service completion times at servers $({\mathrm{u}}_i)_{i=1}^n$, with $T_0 = 0$. In this context, $T_k$ denotes the beginning of the $k$-th round, which consequently has a duration of $T_{k+1} - T_k$.
By definition, the total time for $K_{\epsilon}$ rounds is:
\begin{align} \label{tau_def}
	\tau_{\epsilon} = \sum_{k=1}^{K_{\epsilon}(p,m)} (T_k - T_{k-1}).
\end{align}

Let $N = (N(t))_{t \ge 0}$ be the counting process associated with the point process $(T_k)_{k \in \mathbb{N}}$. This process tracks the total number of service completions and can be expressed as the sum of positive decrements of the uplink queue sizes:
\[
N(t) = \sum_{0 < s \le t} \sum_{i=1}^n \left( \xiu_i(s^-) - \xiu_i(s) \right)^+.
\]
Since $\xi$ is stationary, $N$ is a stationary point process. Consequently, under the Palm probability measure $\bPp$, the inter-event times are identically distributed. Thus, $\bEp[T_k - T_{k-1}] = \bEp[T_1]$ for all $k \ge 1$.

Applying the inversion formula for stationary point processes (see \cite[Corollary~6.16]{serfozo1999}, taking $X_t = 1$), we have:
\[
\lambda \, \bEp[T_1] = 1,
\]
where $\lambda$ denotes the mean intensity of $(N(t))_{t \ge 0}$ under $\bP$.
Therefore, taking the expectation under $\bPp$ in Equation~\eqref{tau_def} yields:
\[
\bEp[\tau_{\epsilon}] = \sum_{k=1}^{K_{\epsilon}(p,m)} \bEp[T_k - T_{k-1}]
= \frac{K_{\epsilon}(p,m)}{\lambda(p,m)}.
\]
This proves Equation~\eqref{eq:tau}. 

\subsection{Proof of Equation~\eqref{eq:throughput}}
In the network model of \Cref{sec:model_description}, the intensity $\lambda(p,m)$ is defined as the expected sum of the instantaneous service completion rates of all uplink servers $({\mathrm{u}}_i)_{i=1}^n$. Based on the generator given in Equation~\eqref{eq:generator}, the instantaneous rate for server $u_i$ is given by $\muu_i \xiu_i$. Summing over all $i$ and taking the expectation under the stationary measure $\bP$ yields:
\[
\lambda(p,m) = \sum_{i=1}^n \muu_i\, \bE[\xiu_i] 
= \frac{Z_{n,m-1}}{Z_{n,m}}.
\]
The second equality is established in Section~\ref{proof_prop:jackson} and follows directly from Equation~\eqref{nume}. This proves Equation~\eqref{eq:throughput}.

\subsection{Proof of Equation~\eqref{eq:grad_throughput}}
Using the first part of Equation~\eqref{eq:grad-log}, we have:
\[
\frac{\partial Z_{n, m-1}}{\partial p_j}
= \frac{Z_{n, m-1}}{p_j} \mathbb{E}[\Xd_j + \Xc_j + \Xu_j].
\]
Following the same logic as in Section~\ref{proof_eq:gradD} (specifically applying the derivation for the stationary distribution $\pi_{n,m}$), we analogously obtain:
\[
\frac{\partial Z_{n, m}}{\partial p_j}
= \frac{Z_{n, m}}{p_j} \mathbb{E}[\xid_j + \xic_j + \xiu_j].
\]
Differentiating Equation~\eqref{eq:throughput} for the throughput $\lambda(p,m) = Z_{n,m-1}/Z_{n,m}$ yields:
\begin{align*}
	\frac{\partial}{\partial p_j} \lambda(p,m)
	&= \frac{\partial}{\partial p_j} \left( \frac{Z_{n, m-1}}{Z_{n, m}} \right) \\
	&= \frac{Z_{n, m-1}}{Z_{n, m}} \left( \frac{1}{Z_{n, m-1}}\frac{\partial Z_{n, m-1}}{\partial p_j} - \frac{1}{Z_{n, m}}\frac{\partial Z_{n, m}}{\partial p_j} \right) \\
	&= \frac{\lambda(p,m)}{p_j} \left( \mathbb{E}[\Xd_j + \Xc_j + \Xu_j] - \mathbb{E}[\xid_j + \xic_j + \xiu_j] \right),
\end{align*}
which completes the proof.

\section{Optimizing Time Complexity: Optimization Results} \label{num:opt_time_strat}

As established in Section~\ref{num:opt_results}, relying on Theorem~\ref{theo:little_law} and Proposition~\ref{prop:time-eps}, we derive a closed-form expression for the expected time to $\epsilon$-accuracy, $\bEp[\tau_{\epsilon}]$, as a function of the routing vector $p$ and the concurrency level $m$. Furthermore, we can compute its exact gradient with respect to the routing probabilities. To find the optimal parameter pair $(p^{\ast \tau}, m^{\ast \tau})$, we minimize $\bEp[\tau_{\epsilon}]$ using a sequential optimization approach to account for the discrete nature of $m$.

Starting from $m=2$, we determine the optimal routing vector $p^\ast$ for each fixed $m$ using the Adam optimizer \cite{kingma2014adam}. To accelerate convergence, we employ a \textit{warm-start strategy}: the optimization for level $m+1$ is initialized using the optimal vector $p^\ast$ found at level $m$. The sequential search terminates when the objective function $\bEp[\tau_{\epsilon}]$ ceases to decrease, indicating that the optimal concurrency level $m^{\ast \tau}$ has been identified.

\begin{figure*}[htbp]
	\centering
	\includegraphics[width=0.6\textwidth]{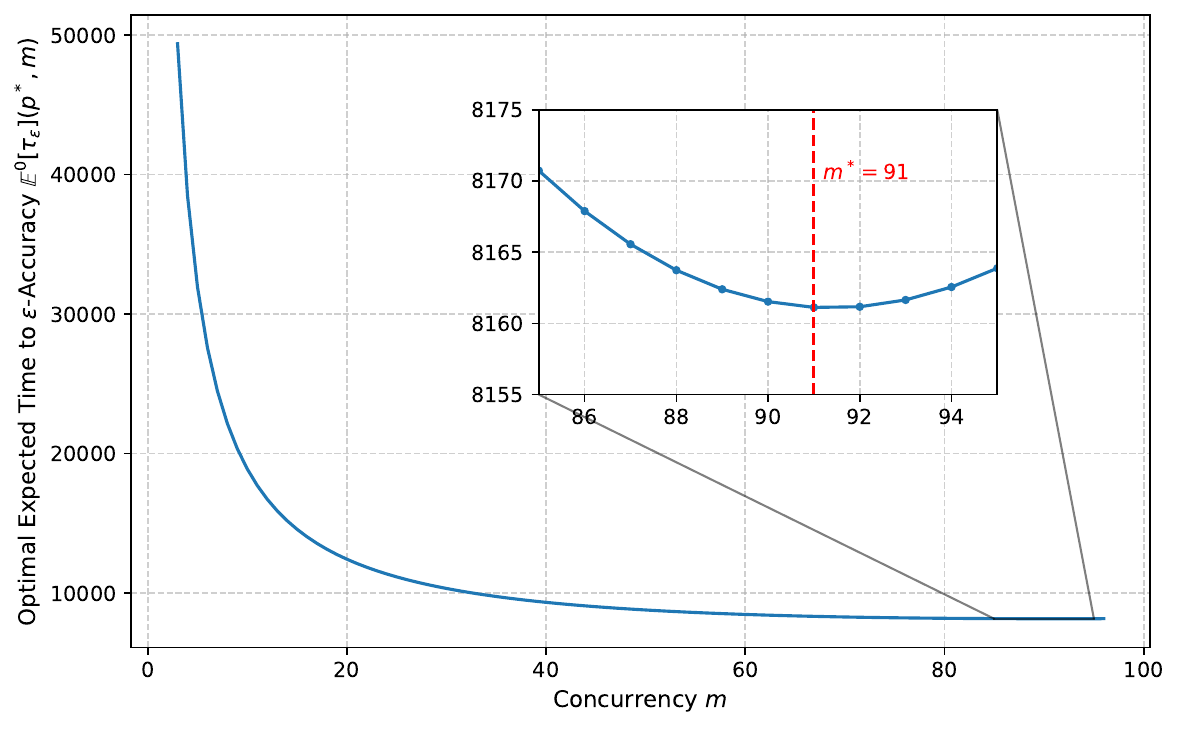}
	\caption{Optimized expected time to $\epsilon$-accuracy, $\bEp[\tau_{\epsilon}](p^\ast,m)$, as a function of the concurrency level $m$. The curve illustrates the fundamental trade-off between parallel speedup and gradient staleness, with the inset highlighting the minimum achieved at the optimal concurrency $m^{\ast} = 91$.}
	\label{fig:tau_star_m}
\end{figure*}

\Cref{fig:tau_star_m} illustrates the evolution of the optimized expected time, $\bEp[\tau_{\epsilon}](p^\ast,m)$, as a function of $m$ during the sequential optimization process. The macro-level curve exhibits a sharp initial decline, demonstrating that introducing moderate concurrency drastically reduces the time required to reach an $\epsilon$-stationary point due to the benefits of parallel processing. However, as $m$ continues to increase, the system experiences diminishing marginal returns and the curve appears to flatten. The zoomed-in inset reveals the precise behavior in this asymptotic region: the expected time reaches a strict absolute minimum at $m^{\ast \tau} = 91$. Beyond this optimal concurrency level, the curve begins to increase, confirming that the computational benefits of further parallelism are finally outweighed by staleness effects. The search halts exactly at this minimum, yielding the optimal configuration $(m^{\ast \tau}, p^{\ast \tau})$ for maximum system efficiency.

\section{Optimizing Time Complexity: Additional Experiments} \label{num:time_c100}

To further validate the insights derived in Section~\ref{num:optimize-time}, we extend our experimental evaluation to the more challenging CIFAR-100 dataset.  
We retain the experimental setup described in Section~\ref{num:exp_scenario}, using the same network dynamics and client profiles from \Cref{tab:client_profiles}.  
As in Section~\ref{num:optimize-time}, we consider both IID and non-IID data distributions (modeled using a Dirichlet distribution with $\alpha=0.2$), and evaluate performance under three service-time distributions: exponential, log-normal, and deterministic.

We compare three strategies: (i) the \textbf{time-optimized} configuration $(p^{\ast \tau},\,m^{\ast \tau})$, (ii) the \textbf{round-optimized} configuration $(p^{\ast K},\,m=n)$, and (iii) the baseline \texttt{AsyncSGD} with uniform routing $(p^{\text{uni}},\,m=n)$, consistent with the setup in \Cref{tab:opt_results}.

\Cref{fig:time_results_c100} shows the evolution of test loss and accuracy as a function of wall-clock time for all configurations.  
In addition, \Cref{tab:time_target_acc_c100} reports the percentage reduction in wall-clock time required to reach a target accuracy of $0.5$ relative to the baselines.

Overall, the results closely mirror those observed in Section~\ref{num:optimize-time}.  
The time-optimized strategy consistently achieves the best wall-clock performance, confirming its robustness and effectiveness when minimizing training time is the primary objective.

\begin{figure*}[htbp]
	\centering
	\includegraphics[width=\textwidth]{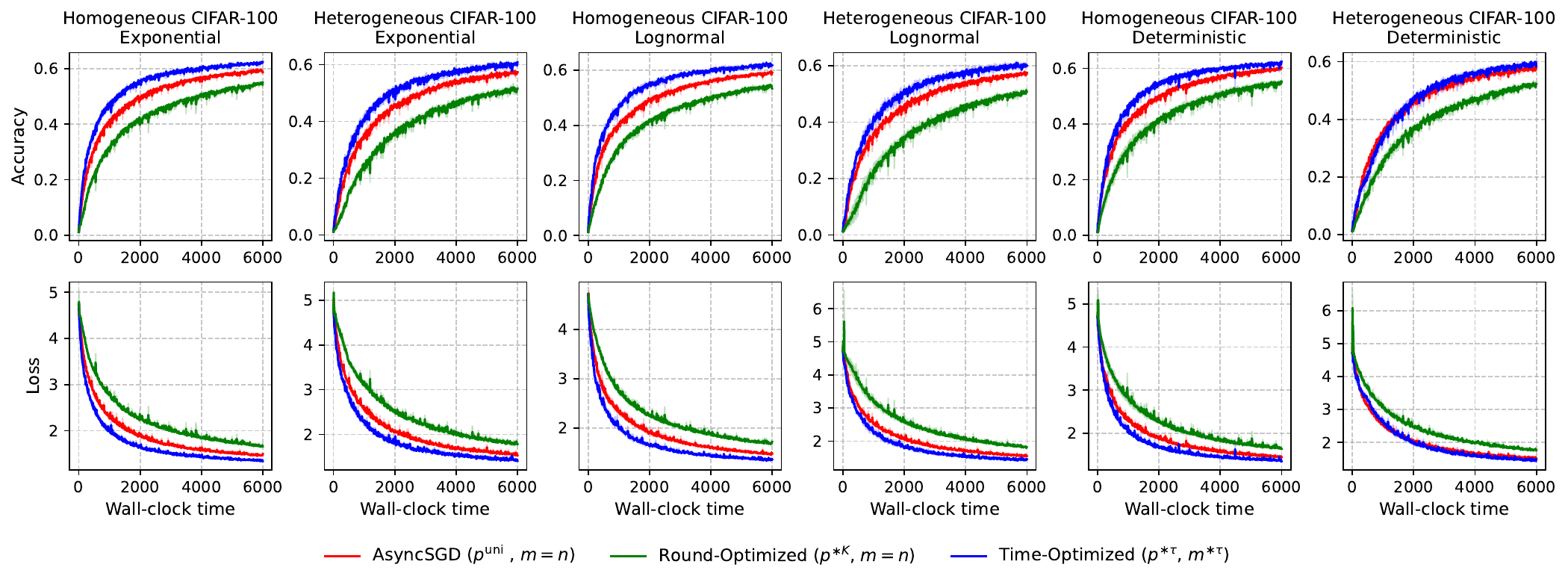}
	\caption{Test set performance for the CIFAR-100 scenario described in Section~\ref{num:exp_scenario}. The top row displays test accuracy and the bottom row shows loss, both plotted against wall-clock time. Columns correspond to different service time distributions under homogeneous (IID) and heterogeneous (Non-IID) data settings. Each subplot compares the three strategies. Simulations were repeated 3 times for 6000 wall-clock time units. Solid lines indicate means; shaded areas represent standard deviations.}
	\label{fig:time_results_c100}
\end{figure*}

\begin{table}[htbp]
	\centering
	\caption{Percentage time reduction of the Time-optimized strategy relative to baselines for a target test accuracy of $0.5$ on CIFAR-100.}
	\label{tab:time_target_acc_c100}
	\begin{tabular}{llcc}
		\toprule
		\multicolumn{2}{c}{\textbf{Scenario}} & \multicolumn{2}{c}{\textbf{Time Reduction (\%) vs.}} \\
		\cmidrule(lr){1-2} \cmidrule(lr){3-4}
		\textbf{Dist.} & \textbf{Data}  & \textbf{Round-Optimized} & \textbf{AsyncSGD} \\
		\midrule
		\multirow{2}{*}{Exp.} 
		& IID      & 67.95 & 40.04 \\
		& Non-IID  & 60.52 & 31.26 \\
		\midrule
		\multirow{2}{*}{LogN.}   
		& IID      & 67.31 & 40.01 \\
		& Non-IID  & 65.46 & 28.86 \\
		\midrule
		\multirow{2}{*}{Det.} 
		& IID      & 66.23 & 38.05 \\
		& Non-IID  & 51.79 & 8.83 \\
		\bottomrule
	\end{tabular}
\end{table}

\section{Proof of Proposition~\ref{prop:energy-eps}} \label{proof_prop:energy-eps}

Recall that $P(t)$ denotes the instantaneous power consumption of the system at wall-clock time~$t$.  
Consistent with the energy model introduced in \Cref{sec:nrg_model}, the instantaneous power is defined as follows:
\begin{align} \label{eq:inst_power}
	P(t)
	= \sum_{i=1}^n \left(
	\Pc_i \,\mathbf{1}\{\xic_i(t) > 0\}
	+ \Pu_i \,\xiu_i(t)
	+ \Pd_i \,\xid_i(t)
	\right).
\end{align}

Let $\tau_{\epsilon}$ and $K_{\epsilon}$ denote, respectively, the wall-clock time and the number of global rounds required to reach $\epsilon$-accuracy.  
Note that $K_{\epsilon}$ is a deterministic integer determined by the routing vector~$p$, the concurrency level~$m$, and the service rates~$\mu$, whereas $\tau_{\epsilon}$ is a random variable.  
The total energy consumption up to $\epsilon$-accuracy is given by
\[
E_{\epsilon}
= \int_0^{\tau_{\epsilon}} P(t)\,dt
= \sum_{k=1}^{K_{\epsilon}(p,m)} \int_{T_{k-1}}^{T_k} P(t)\,dt,
\]
where $(T_k)_{k \in \mathbb{N}}$ denotes the sequence of model-parameter update times (i.e., service completion times at the uplink servers $(\mathrm{u}_i)_{i=1}^n$), with $T_0=0$.

Since the system state process $(\xi(t))_{t \ge 0}$ is stationary and ergodic, the sequence of energy increments over the renewal intervals $[T_{k-1}, T_k)$ forms a stationary sequence.  
Consequently, under the Palm probability measure associated with $\{T_k\}$, the expected energy consumed per round is constant:
\[
\mathbb{E}^0\!\left[\int_{T_{k-1}}^{T_k} P(t)\,dt\right]
= \mathbb{E}^0\!\left[\int_0^{T_1} P(t)\,dt\right].
\]

Applying the inversion formula for stationary point processes (see~\cite[Corollary~6.16]{serfozo1999}, taking $X_t=P(t)$) yields
\[
\mathbb{E}^0\!\left[\int_0^{T_1} P(t)\,dt\right]
= \frac{\mathbb{E}[P(0)]}{\lambda(p,m)},
\]
where $\lambda$ denotes the system throughput (update frequency) as in \Cref{eq:throughput}.  
Substituting this expression into the summation for $E_{\epsilon}$ gives
\begin{align} \label{eq:energy_exp}
	\mathbb{E}^0[E_{\epsilon}]
	= \sum_{k=1}^{K_{\epsilon}(p,m)} \frac{\mathbb{E}[P(0)]}{\lambda}
	= \frac{K_{\epsilon}(p,m)}{\lambda(p,m)}\,\mathbb{E}[P(0)].
\end{align}

It remains to compute the steady-state expected power $\mathbb{E}[P(0)]$.  
Taking expectations in~\eqref{eq:inst_power} with respect to the stationary distribution $\pi_{n,m}$ and using linearity yields
\[
\mathbb{E}[P(0)]
= \sum_{i=1}^n \left(
\Pc_i \,\mathbb{P}(\xic_i > 0)
+ \Pd_i \,\mathbb{E}[\xid_i]
+ \Pu_i \,\mathbb{E}[\xiu_i]
\right).
\]

We now evaluate these terms using results established earlier.  
For the computation queues, applying \eqref{eq:prob_Xc} with $k=1$ to a system of population~$m$ (recall that \eqref{eq:prob_Xc} holds for $X\sim\pi_{n,m-1}$, whereas here $\xi\sim\pi_{n,m}$) gives
\[
\mathbb{P}(\xic_i > 0)
= \frac{Z_{n,m-1}}{Z_{n,m}}\,\frac{p_i}{\muc_i}.
\]
Similarly, using the stationary expressions derived in~\eqref{eq:EXd} and its uplink counterpart yields
\[
\mathbb{E}[\xid_i]
= \frac{Z_{n,m-1}}{Z_{n,m}}\,\frac{p_i}{\mud_i},
\qquad
\mathbb{E}[\xiu_i]
= \frac{Z_{n,m-1}}{Z_{n,m}}\,\frac{p_i}{\muu_i}.
\]

Substituting these expressions into $\mathbb{E}[P(0)]$ and recalling from~\eqref{eq:throughput} that $\lambda(p,m) = \frac{Z_{n,m-1}}{Z_{n,m}}$, the normalization constants cancel with the factor $1/\lambda(p,m)$ in~\eqref{eq:energy_exp}.  
We thus obtain
\[
\mathbb{E}^0[E_{\epsilon}]
= K_{\epsilon}(p,m)
\sum_{i=1}^n p_i
\left(
\frac{\Pc_i}{\muc_i}
+ \frac{\Pd_i}{\mud_i}
+ \frac{\Pu_i}{\muu_i}
\right),
\]
which completes the proof.

\section{Proof of Equation~\ref{opt_E_routing}} \label{proof_opt_E_routing}

We seek the optimal probability distribution $p^{\ast E}$ that minimizes the objective function:
\begin{equation*}
	\left( \sum_{j=1}^n \frac{1}{p_j} \right) \left( \sum_{i=1}^n p_i \mathcal{E}_i \right).
\end{equation*}

Recall the Cauchy-Schwarz inequality, which states that for any real sequences $(a_j)$ and $(b_j)$:
\begin{equation*}
	\left( \sum_{j=1}^n a_j^2 \right) \left( \sum_{j=1}^n b_j^2 \right) \ge \left( \sum_{j=1}^n a_j b_j \right)^2.
\end{equation*}
By defining the sequences $a_j = \frac{1}{\sqrt{p_j}}$ and $b_j = \sqrt{p_j \mathcal{E}_j}$, we can lower-bound the objective function as follows:
\begin{equation*}
	\left( \sum_{j=1}^n \frac{1}{p_j} \right) \left( \sum_{j=1}^n p_j \mathcal{E}_j \right) \ge \left( \sum_{j=1}^n \frac{1}{\sqrt{p_j}} \sqrt{p_j \mathcal{E}_j} \right)^2 = \left( \sum_{j=1}^n \sqrt{\mathcal{E}_j} \right)^2.
\end{equation*}
Thus, the objective function is lower-bounded by the constant $\left( \sum_{j=1}^n \sqrt{\mathcal{E}_j} \right)^2$.

The minimum is achieved when the Cauchy-Schwarz equality holds, which occurs if and only if the sequences $(a_j)$ and $(b_j)$ are linearly proportional ($a_j \propto b_j$). Substituting our definitions, we require:
\begin{equation*}
	\frac{1}{\sqrt{p_j}} \propto \sqrt{p_j \mathcal{E}_j}.
\end{equation*}
Isolating $p_j$ yields the optimal routing probabilities up to a proportionality constant:
\begin{equation*}
	p^{\ast E}_i \propto \frac{1}{\sqrt{\mathcal{E}_i}}.
\end{equation*}
Since $p^{\ast E}$ must be a valid probability distribution ($\sum_{i=1}^n p^{\ast E}_i = 1$), we normalize to obtain the exact closed-form solution:
\begin{equation*}
	p^{\ast E}_i = \frac{\mathcal{E}_i^{-1/2}}{\sum_{j=1}^n \mathcal{E}_j^{-1/2}}, \quad i \in \{1,\ldots,n\},
\end{equation*}
which concludes the proof.

\section{Optimizing Energy Complexity: Additional Experiments} \label{num:optimize-energy}

Building upon the results presented in Section~\ref{num:energy_learning}, we further illustrate the time-energy trade-off in \texttt{Generalized AsyncSGD} by providing detailed learning trajectories for the EMNIST dataset. We evaluate both homogeneous and heterogeneous (Dirichlet with $\alpha=0.2$) data partitions across varying service time distributions (exponential, lognormal, and deterministic). Client speeds and power profiles are configured according to \Cref{tab:client_profiles} and \Cref{tab:power_coefficients}, respectively. Retaining the overarching experimental setup from Section~\ref{num:optimize-time}, we evaluate the convergence dynamics across the following five distinct strategies:
\begin{itemize}
	\item \textbf{Max-Throughput:} Maximizes the system update frequency ($p^{\ast \lambda}$) under full concurrency ($m=n$), without accounting for the impact of gradient staleness on convergence quality.
	
	\item \textbf{Time-Optimized:} Combines gradient staleness with update frequency using the optimal parameters $(p^{\ast \tau}, m^{\ast \tau})$. As characterized in Proposition~\ref{prop:time-eps}, this strategy strictly minimizes the expected wall-clock time $\bEp[\tau_{\epsilon}]$ required to achieve $\epsilon$-accuracy, regardless of the resulting energy consumption.
	
	\item \textbf{Energy-Optimized:} Employs the optimal parameters $(p^{\ast E}, m=1)$ to strictly minimize the expected energy $\bEp[E_{\epsilon}]$ required to achieve $\epsilon$-accuracy (Proposition~\ref{prop:energy-eps}), with no consideration for convergence speed.
	
	\item \textbf{Time-Energy Co-Optimized:} Utilizes the optimal configuration $(p^{\ast \rho}, m^{\ast \rho})$ to minimize the joint objective defined in \eqref{eq:joint_opt}. We set $\rho=0.1$, which, as demonstrated by the Pareto frontier in \Cref{fig:pareto_opt}, strategically balances the trade-off by securing substantial energy savings with minimal degradation to convergence speed.
	
	\item \textbf{Standard Baseline (\texttt{AsyncSGD})~\cite[Algorithm 2]{koloskova2022sharper}:} Represents the conventional approach, employing full concurrency ($m=n$) and uniform routing probabilities ($p^{\text{uni}}$).
\end{itemize}

\Cref{fig:time_nrg_results} illustrates the evolution of test accuracy with respect to both cumulative energy consumption (top row) and wall-clock time (bottom row). We analyze the convergence dynamics of each routing strategy, highlighting how their distinct mathematical objectives govern the fundamental trade-off between learning speed and energy efficiency.

While the \textbf{Max-Throughput} strategy achieves the highest theoretical update frequency, it performs poorly across both metrics. By disproportionately favoring the fastest, most energy-intensive client (Type~E), it incurs massive energy costs. Furthermore, ignoring gradient staleness severely hinders convergence. This is particularly detrimental in heterogeneous (Non-IID) settings, where delayed updates from slower clients induce severe client drift, pushing the optimization in conflicting directions. Consequently, Max-Throughput is highly unstable, exhibiting severe performance spikes and high variance (orange regions in \Cref{fig:time_nrg_results}).

The \textbf{Time-Optimized} strategy mitigates this instability by jointly optimizing update frequency and gradient staleness. As shown in the bottom row of \Cref{fig:time_nrg_results}, it consistently achieves the fastest convergence across all evaluated scenarios. However, because its objective function is entirely agnostic to power consumption, it incurs a substantial energy penalty, trailing the energy-aware strategies significantly.

Conversely, the \textbf{Energy-Optimized} strategy strictly minimizes total energy. It vastly outperforms all others on the energy scale (top row), achieving high accuracy with minimal power. However, to avoid the energy overhead of concurrent communication, it enforces strict sequential processing ($m=1$). This imposes a debilitating temporal penalty, rendering its wall-clock convergence prohibitively slow (appearing nearly flat in the bottom row).

The standard \textbf{\texttt{AsyncSGD}} baseline represents the conventional, unoptimized approach. Relying on uniform routing probabilities and fixed full concurrency without adaptive tuning, it yields mediocre performance in both domains. It is consistently slower than the Time-Optimized strategy while simultaneously consuming significantly more energy than the energy-aware approaches.

Finally, the \textbf{Time-Energy Co-Optimized} strategy effectively bridges these extremes. By navigating the Pareto frontier ($\rho=0.1$), it secures a highly efficient operating point: its energy efficiency is second only to the strictly Energy-Optimized approach, and its temporal convergence is second only to the strictly Time-Optimized approach. Crucially, this joint strategy strictly dominates the \texttt{AsyncSGD} baseline, achieving higher accuracy faster \textit{and} with significantly less energy across all data and service time distributions.

\begin{figure*}[htbp] 
	\centering
	\includegraphics[width=\textwidth]{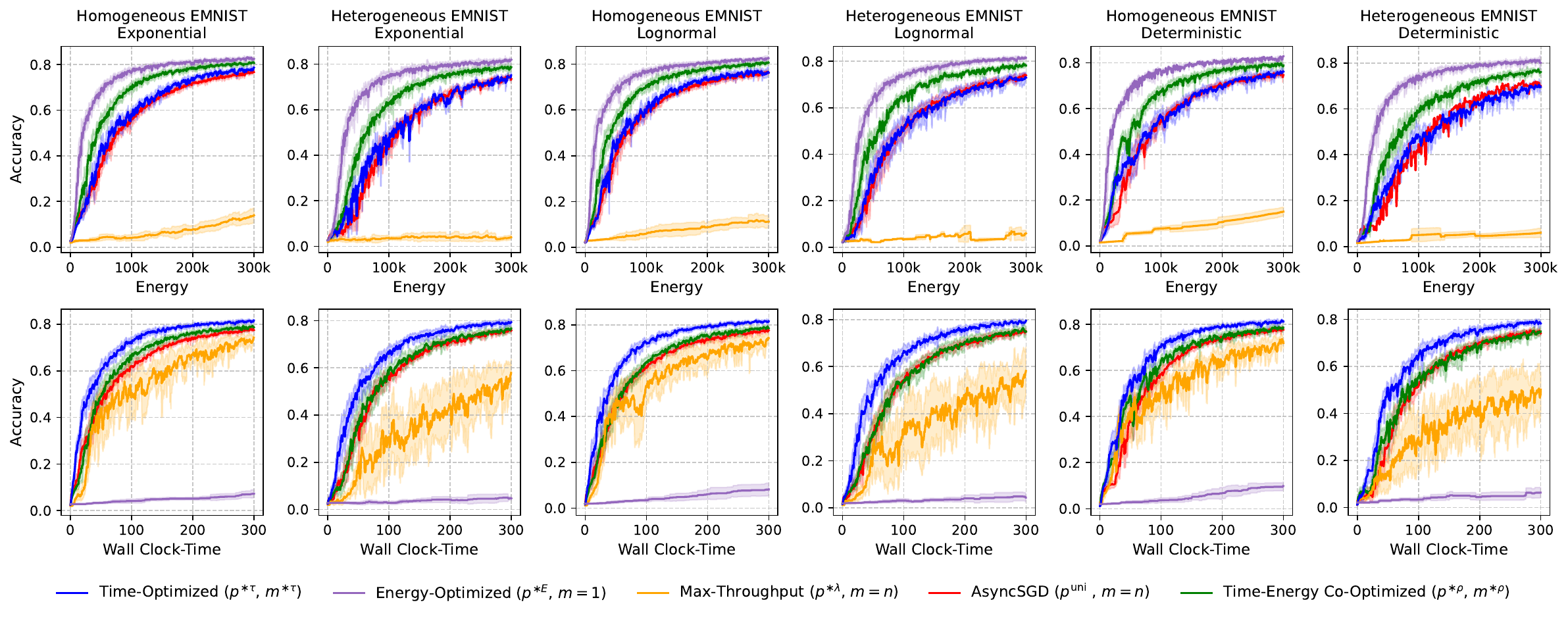}
	\caption{Learning trajectories on the EMNIST dataset under the experimental scenario detailed in Section~\ref{num:exp_scenario} and the power profile from Table~\ref{tab:power_coefficients}. The top and bottom rows display test accuracy as a function of total energy consumption and wall-clock time, respectively. Columns represent different service time distributions across homogeneous (IID) and heterogeneous (Non-IID) data partitions. Each panel compares the convergence dynamics of the five routing strategies. Results are averaged over 10 independent runs spanning 300,000 energy units and 300 time units, with solid lines indicating the mean and shaded regions representing the standard deviation.}
	\label{fig:time_nrg_results}
\end{figure*}

\section{Proof of Proposition~\ref{prop:jackson2}} \label{proof_prop:jackson2}

Although $(Y(t))_{t \ge 0}$ and $(\tilde{Y}_k)_{k \in \bN}$ track task history via class labels, they do not constitute strict Markov chains. By recording only the aggregate number of tasks per class, these state vectors discard the specific arrival order required to characterize the \gls{FIFO} dynamics at the \gls{CS}.

To establish the stationary distributions and facilitate the application of Little’s Law, we must first extend the state description to recover the Markov property.
This proof proceeds in three steps. First, we introduce a detailed Markovian state space that explicitly incorporates the order of tasks in the \gls{CS} queue and derive its stationary distribution. Second, we marginalize this detailed distribution to prove that the continuous-time process $(Y(t))_{t \ge 0}$ follows $\phi_{n,m}$. Finally, we demonstrate that the discrete-time embedded chain $(\tilde{Y}_k)_{k \in \bN}$ admits $\phi_{n,m-1}$ as its stationary distribution.

\subsection{Markovian Description} \label{sec:markovian_description}

To recover a fully Markovian description of the system dynamics, it is necessary to explicitly track the ordering of tasks at the \gls{CS}.  
Indeed, aggregated state representations discard the \gls{FIFO} order at the \gls{CS}, which is essential for correctly characterizing future transitions.  
We therefore introduce an augmented state process that records both task counts and ordering information.
Specifically, we define the right-continuous (càdlàg) stochastic process
\[
Z(t) = \bigl(\Zcs(t), \,\Zd_1(t),\, \Zc_1(t),\,\Zu_1(t), \ldots, \Zd_n(t), \, \Zc_n(t), \,\Zu_n(t)\bigr).
\]
For each $t \ge 0$ and $i \in \{1,\ldots,n\}$, the components $\Zd_i(t)$, $\Zc_i(t)$, and $\Zu_i(t)$ denote the number of tasks at servers $\mathrm{d}_i$, $\mathrm{c}_i$, and $\mathrm{u}_i$, respectively.
The component $\Zcs(t)$ records the ordered sequence of task classes in the \gls{CS} queue, maintained under a \gls{FIFO} discipline. Formally,
\[
\Zcs(t) = \bigl(\Zcs_1(t), \Zcs_2(t), \ldots, \Zcs_{\ell(\Zcs(t))}(t)\bigr),
\]
where $\Zcs_j(t) \in \{1,\ldots,n\}$ denotes the class of the task at position~$j$ (with $j=1$ corresponding to the task in service), and $\ell(\Zcs(t))$ is the total number of tasks currently at the \gls{CS} (i.e., the length of the sequence).

The process $Z(t)$ evolves on a state space that enforces conservation of the total number of tasks.  
Specifically, it takes values in
\begin{multline} \label{Space_Z}
	\mathcal{S}_m = \Bigg\{
	z = \bigl(\zCS, \,\zd_1,\, \zc_1,\,\zu_1, \ldots, \zd_n, \, \zc_n, \,\zu_n\bigr) \in \mathcal{Z}_m \times \mathcal{X}_{3n,\leq m}
	\;:\; \\[-.15cm]
	\ell(\zCS) + \sum_{i=1}^n (\zd_i + \zc_i + \zu_i) = m
	\Bigg\}.    
\end{multline}
Here, $\mathcal{Z}_m$ denotes the set of finite class sequences of length at most~$m$,
\[
\mathcal{Z}_m = \{ (z_1,\ldots,z_\ell) : 0 \leq \ell \leq m, \, z_j \in \{1,\ldots,n\} \text{ for all } j\},
\]
and $\mathcal{X}_{3n,\leq m}$ denotes the set of $3n$-dimensional vectors with non-negative integer entries summing to at most $m$:
\[
\mathcal{X}_{3n,\leq m} = \{x \in \mathbb{N}^{3n} : |x| \leq m\}.
\]

With this augmented state space, the system admits a Markovian description with tractable stationary behavior, as formalized in the following lemma.

\begin{lemma} \label{lemma:class_seq}
	In the model of \Cref{sec:model_description2}, the process $(Z(t))_{t \ge 0}$ is an irreducible, positive recurrent continuous-time Markov chain. Its stationary distribution is given by
	\begin{align} \label{eq:zeta}
		\zeta_{n,m}(z)
		&= \frac{1}{W_{n,m}}
		\prod_{i=1}^n
		\left(\frac{p_i}{\mu^{\mathrm{CS}}}\right)^{|\zCS|_i}
		\left(\frac{p_i}{\muc_i}\right)^{\zc_i}
		\frac{1}{\zd_i!}\left(\frac{p_i}{\mud_i}\right)^{\zd_i}
		\frac{1}{\zu_i!}\left(\frac{p_i}{\muu_i}\right)^{\zu_i},
	\end{align}
	for all $z \in \mathcal{S}_m$, where $W_{n,m}$ is a normalizing constant.
	Here, $|\zCS|_i$ denotes the count of class-$i$ tasks within the sequence $\zCS$, defined as
	$|\zCS|_i = \sum_{j=1}^{\ell(\zCS)} \mathbf{1}\{\zCS_j = i\}$.
\end{lemma}

\begin{proof}
	We first establish irreducibility and positive recurrence.  
	Since all routing probabilities $p_i$ are strictly positive and the network topology is strongly connected, any state in $\mathcal{S}_m$ can be reached from any other state.  
	Because $\mathcal{S}_m$ is finite, the chain is irreducible and positive recurrent.
	
	Let $r$ denote the transition rates of $(Z(t))_{t \ge 0}$.
	For any state $z$ and class index $i \in \{1,\ldots,n\}$, let $\ell = \ell(\zCS)$ denote the current number of tasks in the \gls{CS} queue. We define the following state transformations:
	\begin{itemize}
		\item $z^{+i}$: append a task of class $i$ to the end of the sequence $\zCS$ (i.e., new sequence $(\zCS_1, \dots, \zCS_\ell, i)$). Valid if $\ell < m$.
		\item $z_{+i}$: prepend a task of class $i$ to the beginning of $\zCS$ (i.e., new sequence $(i, \zCS_1, \dots, \zCS_\ell)$). Valid if $\ell < m$.
		\item $z^{-i}$: remove the last task from $\zCS$. Valid if $\ell > 0$ and $\zCS_\ell = i$.
		\item $z_{-i}$: remove the first task from $\zCS$. Valid if $\ell > 0$ and $\zCS_1 = i$.
	\end{itemize}
	
	The non-zero transition rates of $(Z(t))_{t \ge 0}$ are given by:
	\begin{align} \label{rates_z}
		\begin{cases}
			r(z,\,z - \ed_i + \ec_i) = \mud_i \, \zd_i, \\[0.5ex]
			r(z,\,z - \ec_i + \eu_i) = \muc_i \,\mathbf{1}\{\zc_i > 0\}, \\[0.5ex]
			r(z,\,z^{+i} - \eu_i) = \muu_i \, \zu_i, \\[0.5ex]
			r(z,\,z_{-i} + \ed_j) = \mu^{\mathrm{CS}} \,\mathbf{1}\{\zCS_1 = i\} \, p_j,
		\end{cases}
		\qquad i,j \in \{1,\ldots,n\}.
	\end{align}
	The total exit rate from state $z$ is:
	\begin{align*}
		r(z) &= \sum_{y \in \mathcal{S}_m } r(z,y)
		= \sum_{i=1}^n \mud_i \zd_i + \muc_i \mathbf{1}\{\zc_i > 0\} + \muu_i \zu_i + \muCS \mathbf{1}\{\zCS_1 = i\} \underbrace{\sum_{j=1}^n p_j}_{=1} \\
		&= \mu^{\mathrm{CS}} \mathbf{1}\{\ell(\zCS) \neq 0\} + \sum_{i=1}^n \left( \mud_i \zd_i + \muc_i \mathbf{1}\{\zc_i > 0\} + \muu_i \zu_i \right).
	\end{align*}
	
	To prove \Cref{lemma:class_seq}, we invoke the \emph{Reversal Test}~\cite[Theorem~13.4.13]{Bremaud_mc_2020}. We construct the reversed transition rates $\tilde{r}$ with respect to the proposed measure $\zeta_{n,m}$ and verify that the total rate equation $r(z) = \sum_y \tilde{r}(z,y)$ holds. If satisfied, $\zeta_{n,m}$ is the invariant measure.
	
	We derive the reversed rates $\tilde{r}(z,y) = \frac{\zeta_{n,m}(y)}{\zeta_{n,m}(z)} r(y,z)$ for each transition type:
	
	1. \textbf{Computation $\to$ Uplink:}
	\begin{align*}
		\tilde{r}(z,\,z - \eu_i + \ec_i)
		&= \frac{\zeta_{n,m}(z - \eu_i + \ec_i)}{\zeta_{n,m}(z)} \, r(z - \eu_i + \ec_i,\, z) \\
		&= \frac{p_i/\muc_i}{(1/\zu_i)(p_i/\muu_i)} \, \muc_i \mathbf{1}\{\zu_i \ge 1\} \\
		&= \muu_i \zu_i.
	\end{align*}
	
	2. \textbf{Downlink $\to$ Computation:}
	\begin{align*}
		\tilde{r}(z,\,z - \ec_i + \ed_i)
		&= \frac{\zeta_{n,m}(z - \ec_i + \ed_i)}{\zeta_{n,m}(z)} \, r(z - \ec_i + \ed_i,\, z) \\
		&= \frac{(1/(\zd_i+1))(p_i/\mud_i)}{p_i/\muc_i} \, \mud_i (\zd_i+1) \mathbf{1}\{\zc_i \ge 1\} \\
		&= \muc_i \mathbf{1}\{\zc_i > 0\}.
	\end{align*}
	
	3. \textbf{Arrival to CS:}
	\begin{align*}
		\tilde{r}(z,\,z^{-i} + \eu_i)
		&= \frac{\zeta_{n,m}(z^{-i} + \eu_i)}{\zeta_{n,m}(z)} \, r(z^{-i} + \eu_i,\, z) \\
		&= \frac{(1/(\zu_i+1))(p_i/\muu_i)}{(p_i/\mu^{\mathrm{CS}}} \, \muu_i (\zu_i+1) \mathbf{1}\{\zCS_{|\zCS|} = i\} \\
		&= \mu^{\mathrm{CS}} \mathbf{1}\{\zCS_{\ell(\zCS)} = i\}.
	\end{align*}
	
	4. \textbf{Departure from CS:}
	\begin{align*}
		\tilde{r}(z,\,z_{+i} - \ed_j)
		&= \frac{\zeta_{n,m}(z_{+i} - \ed_j)}{\zeta_{n,m}(z)} \, r(z_{+i} - \ed_j,\, z) \\
		&= \frac{(p_i/\mu^{\mathrm{CS}})}{(1/\zd_j)(p_j/\mud_j)} \, \mu^{\mathrm{CS}} \, p_j \,\mathbf{1}\{\zd_j \ge 1\} \\
		&= p_i \, \mud_j \zd_j.
	\end{align*}
	
	Next, summing over all possible transitions yields the total reversed rate:
	\begin{align*}
		\tilde{r}(z) &= \sum_{y \in \mathcal{S}_m } \tilde{r}(z,y)
		= \sum_{i=1}^n \Big(\muc_i \,\mathbf{1}\{\zc_i > 0\} + \muu_i \zu_i + \muCS \,\mathbf{1}\{\zCS_{\ell(\zCS)} = i\} + p_i \sum_{j=1}^n \mud_j \zd_j\Big) \\[0.5ex]
		&= \muCS \, \sum_{i=1}^n \mathbf{1}\{\zCS_{\ell(\zCS)} = i\} +\sum_{i=1}^n \Big(\muc_i \,\mathbf{1}\{\zc_i > 0\} + \muu_i \zu_i + \mud_i \zd_i \sum_{j=1}^n p_j\Big) \\[0.5ex]
		&= \muCS \,\mathbf{1}\{\ell(\zCS) \neq 0\} \;+\; \sum_{i=1}^n \mud_i \zd_i + \muc_i \,\mathbf{1}\{\zc_i > 0\} + \muu_i \zu_i = r(z).
	\end{align*}
	
	Since $r(z) = \tilde{r}(z)$ for all $z \in \mathcal{S}_m$, the reversal test theorem \cite[Theorem~13.4.13]{Bremaud_mc_2020} implies that $\zeta_{n,m}$ is indeed the stationary distribution of the Markov chain $(Z(t))_{t \ge 0}$.
	
\end{proof}

\subsection[Stationary Distribution of the Continuous-Time Process Y]{Stationary Distribution of the Continuous-Time Process $(Y(t))_{t \ge 0}$}

Using the stationary distribution of the detailed Markov process $(Z(t))_{t \ge 0}$ established in Lemma~\ref{lemma:class_seq}, we derive the stationary distribution of the aggregated process $(Y(t))_{t \ge 0}$ by marginalization.

For any admissible state
\[
x = \bigl(\xCS_j, \,\xd_j, \,\xc_j, \,\xu_j \; ; \; j \in \{1,\ldots,n\}\bigr) \in \mathcal{X}_{4n,m},
\]
the stationary probability $\mathbb{P}(Y=x)$ is obtained by summing $\zeta_{n,m}(z)$ over all detailed states $z \in \mathcal{S}_m$ that project onto $x$.  
Equivalently, this corresponds to summing over all class sequences $\zCS$ containing exactly $\xCS_i$ tasks of class~$i$ for each $i \in \{1,\ldots,n\}$. Thus:

\begin{align*}
	\bP(Y=x)
	&= \sum_{\substack{\zCS \in \mathcal{Z}_m \\ |\zCS|_i = \xCS_i,\, \forall i}}
	\zeta_{n,m}\bigl(\zCS, \,\xd_j, \,\xc_j, \,\xu_j \; ; \; j \in \{1,\ldots,n\}\bigr) \\[1ex]
	&= \frac{1}{W_{n,m}}
	\left[ \prod_{j=1}^n
	\left(\frac{p_j}{\mu^{\mathrm{CS}}}\right)^{\xCS_j}
	\left(\frac{p_j}{\muc_j}\right)^{\xc_j}
	\frac{1}{\xd_j!}\left(\frac{p_j}{\mud_j}\right)^{\xd_j}
	\frac{1}{\xu_j!}\left(\frac{p_j}{\muu_j}\right)^{\xu_j} \right]
	\underbrace{\sum_{\substack{\zCS \in \mathcal{Z}_m \\ |\zCS|_i = \xCS_i,\, \forall i}} 1}_{\text{number of compatible sequences}} \\[1ex]
	&= \frac{1}{W_{n,m}}
	\left[ \prod_{j=1}^n
	\left(\frac{p_j}{\mu^{\mathrm{CS}}}\right)^{\xCS_j}
	\left(\frac{p_j}{\muc_j}\right)^{\xc_j}
	\frac{1}{\xd_j!}\left(\frac{p_j}{\mud_j}\right)^{\xd_j}
	\frac{1}{\xu_j!}\left(\frac{p_j}{\muu_j}\right)^{\xu_j} \right]
	\frac{\bigl(\sum_{i=1}^n \xCS_i\bigr)!}{\prod_{i=1}^n \xCS_i!} \\[1ex]
	&= \phi_{n,m}(x).
\end{align*}

The second equality follows from the definition of $\zeta_{n,m}$ in \eqref{eq:zeta}, observing that the probability mass depends on the sequence $\zCS$ only through the counts of each class, not their specific order. Consequently, the term in brackets factors out of the summation.
The third equality evaluates the underbraced sum, which counts the number of distinct sequences $\zCS$ of length $\sum_i \xCS_i$ containing exactly $\xCS_i$ elements of each class $i$. This count is given precisely by the multinomial coefficient
\[
\binom{\sum_{i=1}^n \xCS_i}{\xCS_1, \xCS_2, \ldots, \xCS_n} = \frac{\bigl(\sum_{i=1}^n \xCS_i\bigr)!}{\prod_{i=1}^n \xCS_i!}.
\]
This concludes the derivation of the product form for $(Y(t))_{t \ge 0}$ as stated in Proposition~\ref{prop:jackson2}.

\subsection[Stationary Distribution of the Discrete-Time Process $\tilde{Y}$]{Stationary Distribution of the Discrete-Time Process $(\tilde{Y}_k)_{k \in \mathbb{N}}$}

We now establish the stationary distribution of the system state observed at parameter-update instants, namely the discrete-time process $(\tilde{Y}_k)_{k \in \mathbb{N}}$, and show that it admits $\phi_{n,m-1}$ as its invariant distribution.

Let $(\hat{Z}_k)_{k \in \mathbb{N}}$ denote the jump chain associated with the ergodic continuous-time Markov process $(Z(t))_{t \ge 0}$.  
By Lemma~\ref{lem:2-pi-cont}, the sequence of transitions
\[
\bigl\{(\hat{Z}_k, \hat{Z}_{k+1}) : k \in \mathbb{N}\bigr\}
\]
forms an irreducible, discrete-time, homogeneous Markov chain with invariant measure
\[
\hat{\zeta}(u,v) = \zeta_{n,m}(u)\,r(u,v),
\qquad u,v \in \mathcal{S}_m,
\]
where $\mathcal{S}_m$ is the state space of the process $(Z(t))_{t \ge 0}$ as defined in \eqref{Space_Z}, $r$ denotes its transition rates given in \eqref{rates_z}, and $\zeta_{n,m}$ represents its stationary distribution derived in \eqref{eq:zeta}.

We characterize the stationary distribution of $(\tilde{Y}_k)$ by conditioning on transitions of the jump chain corresponding to service completions at the \gls{CS}.  
For any state
\[
x = \bigl(\xCS_i, \,\xd_i, \,\xc_i, \,\xu_i \; ; \; i \in \{1,\ldots,n\}\bigr) \in \mathcal{X}_{4n,m-1},
\]
the stationary probability $\mathbb{P}(\tilde{Y}=x)$ can be expressed as
\begin{align} \label{stat_def_tildeY}
	\mathbb{P}(\tilde{Y} = x)
	= \mathbb{E}\!\left[
	\sum_{\substack{\zCS \in \mathcal{Z}_{m-1}\\ |\zCS|_l = \xCS_l,\, \forall l}}
	\sum_{i=1}^n \sum_{j=1}^n
	\mathbf{1}\{\hat{Z}_k = z_{+i}, \,\hat{Z}_{k+1} = z + \ed_j\}
	\;\middle|\; (\hat{Z}_k, \hat{Z}_{k+1}) \in H
	\right],
\end{align}
where $z$ is the detailed state corresponding to $x$ (augmented with the sequence $\zCS$), i.e.,
\[
z = \bigl(\zCS, \,\xd_i, \,\xc_i, \,\xu_i \; ; \; i \in \{1,\ldots,n\}\bigr) \in \mathcal{S}_{m-1},
\]
and $H$ denotes the set of jump-chain transitions corresponding to service completions at the \gls{CS}:
\[
H = \bigl\{(u,v) \in \mathcal{S}_m \times \mathcal{S}_m : v = u_{-i} + \ed_j \text{ for some } i,j \in \{1,\ldots,n\}\bigr\}.
\]

Applying Corollary~\ref{coro:2-pi-cont} to Equation~\eqref{stat_def_tildeY} with
\[
g(u,v) = \sum_{\substack{\zCS \in \mathcal{Z}_{m-1}\\ |\zCS|_l = \xCS_l,\, \forall l}}
\sum_{i=1}^n \sum_{j=1}^n
\mathbf{1}\{u = z_{+i}, \, v = z + \ed_j\},
\]
we obtain, for each $x \in \mathcal{X}_{4n,m-1}$,
\begin{align}
	\mathbb{P}(\tilde{Y} = x)
	&= \sum_{\substack{\zCS \in \mathcal{Z}_{m-1}\\ |\zCS|_l = \xCS_l,\, \forall l}}
	\frac{\displaystyle \sum_{i=1}^n \sum_{j=1}^n \zeta_{n,m}(z_{+i}) \,\muCS p_j}
	{\displaystyle \sum_{v \in \mathcal{S}_m}
		\sum_{i=1}^n \sum_{j=1}^n
		\zeta_{n,m}(v)\,\muCS\,\mathbf{1}\{\vCS_1 = i\}\,p_j} \nonumber \\[1ex]
	&= \frac{\bigl(\sum_{j=1}^n \xCS_j\bigr)!}{\prod_{j=1}^n \xCS_j!}
	\frac{\displaystyle \sum_{i=1}^n \zeta_{n,m}(z_{+i}) \sum_{j=1}^n p_j}
	{\displaystyle \sum_{v \in \mathcal{S}_m}
		\sum_{i=1}^n \zeta_{n,m}(v)\,\mathbf{1}\{\vCS_1 = i\}\,\sum_{j=1}^n p_j} \nonumber \\[1ex]
	\label{eq:y_tilde}
	&= \frac{\bigl(\sum_{j=1}^n \xCS_j\bigr)!}{\prod_{j=1}^n \xCS_j!}
	\frac{\displaystyle \sum_{i=1}^n \zeta_{n,m}(z_{+i})}
	{\displaystyle \sum_{v \in \mathcal{S}_m}
		\zeta_{n,m}(v)\,\sum_{i=1}^n \mathbf{1}\{\vCS_1 = i\}}.
\end{align}

The second equality follows from the definition of $\zeta_{n,m}$ in \eqref{eq:zeta}. Since the term in the numerator depends only on the class counts and not on the specific sequence $\zCS$, the summation over $\zCS$ reduces to counting the number of distinct orderings, which is given by the multinomial coefficient
\[
\frac{\bigl(\sum_{i=1}^n \xCS_i\bigr)!}{\prod_{i=1}^n \xCS_i!}.
\]
Additionally, the factor $\muCS$ cancels out between the numerator and denominator.
The third equality follows by rearranging the sums and using the fact that $\sum_{j=1}^n p_j = 1$. Note that we use $\vCS_1$ to denote the head of the queue (the task in service).

We now simplify the remaining terms.  
First, we analyze the denominator. Summing the indicator over all classes $i$ is equivalent to checking if the \gls{CS} queue is non-empty:
\begin{align} \label{den2}
	\sum_{v \in \mathcal{S}_m}
	\zeta_{n,m}(v)\, \sum_{i=1}^n \mathbf{1}\{\vCS_1 = i\}
	&= \sum_{v \in \mathcal{S}_m}
	\zeta_{n,m}(v)\,\mathbf{1}\{\ell(\vCS) \neq 0\}
	= \mathbb{P}\Bigl(\sum_{i=1}^n \YCS_i > 0\Bigr),
\end{align}
where $Y \sim \phi_{n,m}$. Using the explicit form of $\phi_{n,m}$ in \eqref{eq:phi}, the probability that the \gls{CS} is busy is given by:
\begin{align} \label{little_cs}
	\mathbb{P}\Bigl(\sum_{i=1}^n \YCS_i > 0\Bigr)
	= \sum_{x \in \mathcal{X}_{4n,m}}
	\phi_{n,m}(x)\,\mathbf{1}\Bigl\{\sum_{i=1}^n \xCS_i > 0\Bigr\}
	= \frac{1}{\muCS} \frac{W_{n,m-1}}{W_{n,m}}.
\end{align}

On the other hand, for the numerator term involving $\zeta_{n,m}(z_{+i})$, we establish a relation between the stationary distributions for populations $m$ and $m-1$:
\begin{align} \label{num2}
	\sum_{i=1}^n \zeta_{n,m}(z_{+i})
	= \sum_{i=1}^n \frac{p_i}{\muCS} \frac{W_{n,m-1}}{W_{n,m}} \zeta_{n,m-1}(z)
	= \frac{1}{\muCS} \frac{W_{n,m-1}}{W_{n,m}} \zeta_{n,m-1}(z).
\end{align}
The first equality follows by factoring out $\frac{p_i}{\muCS}$ and multiplying and dividing by $W_{n,m-1}$ to recover the expression for $\zeta_{n,m-1}(z)$. The second equality follows from the fact that $\sum_{i=1}^n p_i = 1$.

Substituting \eqref{den2}, \eqref{little_cs}, and \eqref{num2} back into \eqref{eq:y_tilde}, the terms $\frac{1}{\muCS} \frac{W_{n,m-1}}{W_{n,m}}$ cancel out, yielding:
\[
\mathbb{P}(\tilde{Y} = x)
= \frac{\bigl(\sum_{j=1}^n \xCS_j\bigr)!}{\prod_{j=1}^n \xCS_j!}
\zeta_{n,m-1}(z)
= \phi_{n,m-1}(x),
\quad x \in \mathcal{X}_{4n,m-1}.
\]
The final equality follows by substituting the definition of $\zeta_{n,m-1}$. This concludes the proof.

\section[Buzen’s Recursive Algorithm with a CS-Side Queue]{Buzen’s Recursive Algorithm with a \gls{CS}-Side Queue} \label{sec:buzen_cs}

We extend \textit{Buzen}'s recursive algorithm to incorporate the \gls{CS} service rate, enabling $\mathcal{O}(nm^2)$ computation of the normalization constants for the model of \Cref{sec:model_description2}.

\begin{proposition} \label{prop:buzen_cs}
	The normalization constants $W_{n,\um}$, for all $\um \in \{0,1,\ldots,m\}$, can be computed in $\mathcal{O}(nm^2)$ time and $\mathcal{O}(m)$ memory using Buzen’s recursive algorithm. In particular, $W_{n,\um} = V_{3n,\um}$ for all $\um \in \{0,1,\ldots,m\}$, where the quantities $V_{\un,\um}$ are defined recursively by
	\begin{alignat*}{2}
		\bullet \quad & V_{\un,0} = 1,
		&\quad& \text{for } \un \in \{1,\ldots,3n+1\} \text{.} \\
		\bullet \quad & V_{1,\um} = \left(\frac{1}{\muCS}\right)^{\um},
		&& \text{for } \um \in \{0,\ldots,m\} \text{.} \\
		\bullet \quad & V_{\un+1,\um} = V_{\un,\um} + \frac{p_\un}{\muc_\un}\,V_{\un+1,\um-1},
		&& \text{for } \un \in \{1,\ldots,n\}, \um \in \{1,\ldots,m\} \text{.} \\
		\bullet \quad & V_{\un+1,\um} = \sum_{k=0}^{\um} \frac{1}{k!} \left(\frac{p_{\un-n}}{\mud_{\un-n}}\right)^{k} V_{\un,\um-k},
		&& \text{for } \un \in \{n+1,\ldots,2n\}, \um \in \{1,\ldots,m\} \text{.} \\
		\bullet \quad & V_{\un+1,\um} = \sum_{k=0}^{\um} \frac{1}{k!} \left(\frac{p_{\un-2n}}{\muu_{\un-2n}}\right)^{k} V_{\un,\um-k},
		&& \text{for } \un \in \{2n+1,\ldots,3n\}, \um \in \{1,\ldots,m\} \text{.}
	\end{alignat*}
\end{proposition}

\section{Proof of Theorem~\ref{theo:little_law2}} \label{proof_theo:little_law2}

The proof of Theorem~\ref{theo:little_law2} parallels the methodology established for the first model (Theorem~\ref{theo:little_law}). We employ the same \textit{uniformization} technique to bridge the gap between the discrete-time operational semantics of the system and the continuous-time requirements of Little's Law. We proceed to establish each claim of the theorem in turn.

\subsection{Proof of Equation~\eqref{eq:D_cs}}

Fix a client index $i \in \{1,\ldots,n\}$ and a round index $k \in \mathbb{N}$. In the stationary regime, the expected delay is given by
\[
\bEp[D_i] = \bEp[D_{i,k}]
= \sum_{j=1}^n \bEp[D_{i,k} \mid A_k=j] \,\bPp(A_k=j)
= p_i \,\bEp[D_{i,k} \mid A_k=i]
= p_i \,\bEp[R_i],
\]
where $R_{i,k}$ denotes the number of global model updates occurring between the dispatch of the $k$-th task to client~$i$ and the application of its resulting gradient. This quantity represents the sojourn time (measured in "number of tasks") of a job within the subsystem formed by the server path $\mathrm{d}_i \to \mathrm{c}_i \to \mathrm{u}_i$ and the \gls{CS} queue. We denote by $\bEp[R_i]$ the Palm stationary expectation of this discrete sojourn time.

Consequently, proving \Cref{eq:D_cs} reduces to showing that, for each $i \in \{1,\ldots,n\}$,
\begin{align} \label{eq:discret_little_cs}
	\bEp[R_i] = \frac{\mathbb{E}[\tYCS_i + \tYd_i + \tYc_i + \tYu_i]}{p_i}.
\end{align}
This relation can be interpreted as Little’s Law applied to the discrete instants of service completion at the \gls{CS}.

\textbf{Step 1: The Uniformization Construct.}

To apply standard queueing theorems, we introduce an auxiliary \emph{uniform continuous-time system}. In this proxy model, the duration of every round (time between parameter updates) is \gls{iid} and exponentially distributed with mean $1$. While this creates a fictitious time scale, it satisfies the premises of the classical continuous-time Little’s Law, allowing us to derive results that we subsequently map back to the original discrete-time process.

Recall that $(T_k)_{k \in \mathbb{N}}$ denotes the sequence of service completion times at the \gls{CS}. Since each $T_k$ is a stopping time with respect to the continuous-time Markov process $(Z(t))_{t \geq 0}$, the sampled sequence $(Z(T_k))_{k \in \mathbb{N}}$ forms an ergodic discrete-time homogeneous Markov chain. We define the \textit{uniform continuous-time Markov chain} $(\Bar{Z}(t))_{t \ge 0}$ by subordinating this discrete chain to a Poisson process $(N(t))_{t \ge 0}$ with unit rate:
\[
\Bar{Z}(t) = Z(T_{N(t)}), \quad t \ge 0.
\]

\textbf{Step 2: Arrival Processes and Sojourn Times.}

Let $\bar{T}_{i, k}$ denote the arrival time of the $k$-th task assigned to client~$i$ in this \textit{uniform system}, and let $N_i(t)$ be the associated counting process.

In this transformed setting, the arrival stream $(\bar{T}_{i, k})_{k \in \mathbb{N}}$ constitutes a homogeneous Poisson point process with intensity $p_i$. This follows from Poisson thinning: the global stream of tasks leaving the \gls{CS} is a Poisson process of rate $1$ (due to the exponential round durations), and each task is independently routed to client $i$ with probability $p_i$.

The sojourn time of the $k$-th task assigned to client~$i$ corresponds to the sum of $R_{i,k}+1$ inter-update intervals. Specifically,
\[
\bar{R}_{i,k} = \sum_{l=1}^{R_{i,k}+1} E_l,
\]
where $(E_l)_{l \in \mathbb{N}}$ is a sequence of \gls{iid} exponential random variables with mean $1$.
Accordingly, the total occupancy of the subsystem associated with client~$i$ (servers $\mathrm{d}_i, \mathrm{c}_i, \mathrm{u}_i$ and the \gls{CS} queue) at time $t$ is given by:
\begin{align} \label{eq:cs_load}
	h_i(\Bar{Z}(t)) 
	= \sum_{k \in \mathbb{N}} \mathbf{1}\{\bar{T}_{i,k} \le t < \bar{T}_{i,k} + \bar{R}_{i,k}\}
	= \bZd_i(t) + \bZc_i(t) + \bZu_i(t) + \sum_{k=1}^{\ell(\bZcs(t))} \mathbf{1}\{\bZcs_k(t) = i\}.
\end{align}

\textbf{Step 3: Application of Little’s Law.}

The uniform system satisfies the requisite conditions for the stationary version of Little's Law as stated in \cite[Theorem~5.2]{serfozo1999}:
(i) The system state $(\Bar{Z}(t))_{t \ge 0}$ is an ergodic Markov chain;
(ii) The arrival process is Poisson (thus stationary and simple);
(iii) The sojourn time $\bar{R}_{i,k}$ is determined solely by the system evolution after arrival $(\Bar{Z}(t), t \ge \bar{T}_{i, k})$;
(iv) The first moments are finite: $\mathbb{E}[h_i(\Bar{Z})] < \infty$ and $\mathbb{E}[N_i(1)] = p_i$.

Applying the theorem to \eqref{eq:cs_load} yields:
\begin{align} \label{little_uniform_cs}
	\mathbb{E}[h_i(\Bar{Z})] = p_i\,\bEp[\bar{R}_i].
\end{align}

\textbf{Step 4: Mapping Back to the Discrete Model.}

We now translate \eqref{little_uniform_cs} back to the operational parameters of the original system.
By the properties of uniformization, the stationary distribution of $\Bar{Z}$ is identical to that of the embedded chain $(Z(T_k))_{k \in \mathbb{N}}$. Thus, $\mathbb{E}[f_i(\Bar{Z})] = \mathbb{E}[h_i(Z(T_k))]$.
Expanding the state function $h_i$, we have:
\[
h_i(Z(T_k)) = \Yd_i(T_k) + \Yc_i(T_k) + \Yu_i(T_k) + \YCS_i(T_k).
\]
Furthermore, since the fictitious intervals $(E_l)$ are independent of the discrete count $R_i$, Wald's identity implies $\mathbb{E}[\bar{R}_i] = \mathbb{E}[R_i + 1]$.
Substituting these relations into \eqref{little_uniform_cs} gives:
\[
\mathbb{E}[\YCS_i(T_k) + \Yd_i(T_k) + \Yc_i(T_k) + \Yu_i(T_k)] = p_i (\bEp[R_i] + 1) = p_i\,\bEp[R_i] + p_i.
\]
Using the fact that the routing probability satisfies $\mathbb{E}[\mathbf{1}\{A_k=i\}] = p_i$, we can absorb the constant term $p_i$ into the expectation on the left side:
\[
\mathbb{E}[\YCS_i(T_k) + \Yc_i(T_k) + \Yu_i(T_k) + (\Yd_i(T_k) - \mathbf{1}\{A_k=i\})] = p_i\,\bEp[R_i].
\]
Finally, we recognize the term in the brackets as the state definition of $\tilde{Y}$ (the system state at update instants \textit{excluding} the newly routed task, as defined in \eqref{def:tilde_Y}). specifically, $\tYd_{i,k} = \Yd_i(T_k) - \mathbf{1}\{A_k=i\}$.
Therefore:
\[
\mathbb{E}[\tYCS_i + \tYd_i + \tYc_i + \tYu_i] = p_i\,\bEp[R_i].
\]
This confirms \eqref{eq:discret_little_cs} and concludes the proof of \Cref{eq:D_cs}.

\subsection{Proof of Equation~\eqref{eq:gradD_cs}} \label{proof_eq:gradD_cs}

Let $i, j \in \{1, \ldots, n\}$.
Our objective is to derive the gradient of the expected delay with respect to the routing parameters. By applying \Cref{eq:D_cs} and exploiting the bilinearity of the covariance operator, proving \Cref{eq:gradD_cs} is equivalent to establishing the following identity:
\begin{align*}
	\frac{\partial \mathbb{E}[\tYCS_i + \tYd_i + \tYc_i + \tYu_i]}{\partial (\log p_j)}
	&= \mathrm{Cov}\Bigl[\tYCS_i + \tYd_i + \tYc_i + \tYu_i,\;\tYCS_j + \tYd_j + \tYc_j + \tYu_j\Bigr].
\end{align*}

For brevity, let $S_k(x) = \xCS_k + \xd_k + \xc_k + \xu_k$ denote the total number of tasks of class $k$ in state $x$.
Recall that the random vector $\tilde{Y}$ follows the stationary distribution $\phi_{n, m-1}$ given in \eqref{eq:phi}. Taking the logarithm of this probability mass function reveals that it belongs to an exponential family with respect to the canonical parameters $\log p_k$:
\begin{align} \label{eq:phi-log}
	\log \phi_{n, m-1}(x)
	&= - \log W_{n, m-1}
	+ \sum_{k = 1}^n S_k(x) \log p_k \nonumber \\
	&\quad + \underbrace{\log \left( \left(\sum\limits_{k=1}^n \xCS_k\right)! \right)
		- \log \left( (\mu^{\mathrm{CS}})^{\sum_k \xCS_k} \prod_{k=1}^n \xCS_k! (\muc_k)^{\xc_k} (\mud_k)^{\xd_k} \xd_k! (\muu_k)^{\xu_k} \xu_k! \right)}_{=H(x)},
\end{align}
for all $x \in \mathcal{X}_{4n, m-1}$.
The normalization constant $W_{n, m-1}$ is defined by summing the exponential terms over the state space:
\begin{align} \label{eq:W-log}
	W_{n, m-1} = \sum_{x \in \mathcal{X}_{4n, m-1}} \exp\bigl( H(x) + \sum_{k=1}^n S_k(x) \log p_k \bigr),
\end{align}
where $H(x)$ collects all terms independent of $\{p_k\}_{k=1}^n$.

We proceed by deriving two intermediate differentiation results.
First, we compute the gradient of the function $\log W_{n, m-1}$.
Differentiating \eqref{eq:W-log} with respect to $\log p_j$ yields:
\begin{align*}
	\frac{\partial W_{n, m-1}}{\partial (\log p_j)}
	&= \sum_{x \in \mathcal{X}_{4n, m-1}} S_j(x) \, \exp\bigl( H(x) + \sum_{k=1}^n S_k(x) \log p_k \bigr) \\
	&= \sum_{x \in \mathcal{X}_{4n, m-1}} S_j(x) \, W_{n, m-1} \, \phi_{n, m-1}(x).
\end{align*}
Dividing by $W_{n, m-1}$, we recover the expectation of $S_j(\tilde{Y})$:
\begin{align} \label{eq:grad-logW}
	\frac{\partial \log W_{n, m-1}}{\partial (\log p_j)}
	= \frac{1}{W_{n, m-1}} \frac{\partial W_{n, m-1}}{\partial (\log p_j)}
	= \mathbb{E}[S_j(\tilde{Y})]
	= \mathbb{E}[\tYCS_j + \tYd_j + \tYc_j + \tYu_j].
\end{align}

Second, we differentiate the log-probability itself. Using \eqref{eq:phi-log} and substituting \eqref{eq:grad-logW}:
\begin{align} \label{eq:grad-log_phi}
	\frac{\partial \log \phi_{n, m-1}(x)}{\partial (\log p_j)}
	&= S_j(x) - \frac{\partial \log W_{n, m-1}}{\partial (\log p_j)} \nonumber \\
	&= (\xCS_j + \xd_j+\xc_j+\xu_j) - \mathbb{E}[\tYCS_j + \tYd_j + \tYc_j + \tYu_j].
\end{align}

Finally, to obtain the gradient of the expectation, we differentiate the definition of the expected value:
\begin{align*}
	\frac{\partial \mathbb{E}[S_i(\tilde{Y})]}{\partial (\log p_j)}
	&= \sum_{x \in \mathcal{X}_{4n, m-1}} S_i(x) \frac{\partial \phi_{n, m-1}(x)}{\partial (\log p_j)} \\
	&= \sum_{x \in \mathcal{X}_{4n, m-1}} S_i(x) \, \phi_{n, m-1}(x) \frac{\partial \log \phi_{n, m-1}(x)}{\partial (\log p_j)}.
\end{align*}
Substituting \eqref{eq:grad-log_phi} into this expression leads to:
\begin{align*}
	\frac{\partial \mathbb{E}[S_i(\tilde{Y})]}{\partial (\log p_j)}
	&= \sum_{x \in \mathcal{X}_{4n, m-1}} \phi_{n, m-1}(x) \, S_i(x) \, \Bigl( S_j(x) - \mathbb{E}[S_j(\tilde{Y})] \Bigr) \\
	&= \mathbb{E}\Bigl[ S_i(\tilde{Y}) \bigl( S_j(\tilde{Y}) - \mathbb{E}[S_j(\tilde{Y})] \bigr) \Bigr] \\
	&= \mathrm{Cov}\Bigl[ S_i(\tilde{Y}), S_j(\tilde{Y}) \Bigr].
\end{align*}
Expanding $S_i$ and $S_j$ back to their components concludes the proof.

\subsection{Proof of Equation~\eqref{exp:delay_cs}} \label{proof_exp:delay_cs}

To establish the closed-form expression in \eqref{exp:delay_cs}, we first characterize the stationary distribution of the system when the task counts at the \gls{CS} are aggregated.

\begin{lemma} \label{lem:aggregated_dist}
	Consider the aggregated process 
	\[
	\Bigl(\sum_{j=1}^n \tYCS_j, \,\tYd_i, \,\tYc_i, \,\tYu_i ;\ i=1,\ldots,n\Bigr),
	\]
	which resides in the state space $\mathcal{X}_{3n+1,m-1}$.
	This process has the following stationary distribution:
	\begin{align} \label{eq:nu}
		\nu_{n,m-1}(x)
		&= \frac{1}{W_{n,m-1}}
		\left(\frac{1}{\muCS}\right)^{\xCS}
		\prod_{i=1}^n
		\left[ \left( \frac{p_i}{\muc_i} \right)^{\xc_i}
		\frac{1}{\xd_i!}\left( \frac{p_i}{\mud_i} \right)^{\xd_i}
		\frac{1}{\xu_i!}\left( \frac{p_i}{\muu_i} \right)^{\xu_i} \right],
	\end{align}
	for $x = (\xCS,\xd_1,\dots,\xd_n,\xc_1,\dots,\xc_n,\xu_1,\dots,\xu_n) \in \mathcal{X}_{3n+1,m-1}$, where $W_{n,m-1}$ is the same normalizing constant as in Proposition~\ref{prop:jackson2}.
\end{lemma}

\begin{proof}
	The stationary probability of the detailed system state $y = \bigl(\yCS_1, \ldots, \yCS_n, \,\yd_1, \ldots, \yd_n,$ $\,\yc_1, \ldots, \yc_n, \,\yu_1, \ldots, \yu_n\bigr)$ is given by the product-form solution $\phi_{n,m-1}(y)$. To obtain the stationary probability of the aggregated state $x$, in which the individual class counts at the \gls{CS} are combined into a single total $\xCS = \sum_{j=1}^n \yCS_j$, we sum $\phi_{n,m-1}(y)$ over all detailed configurations of the \gls{CS} queue that satisfy this summation constraint.
	
	Consequently, for $x = (\xCS, \xd_1, \ldots, \xd_n, \xc_1, \ldots, \xc_n, \xu_1, \ldots, \xu_n) \in \mathcal{X}_{3n+1,m-1}$, the stationary probability is given by:
	\begin{align*}
		&\mathbb{P}\!\left(\sum_{j=1}^n \tYCS_j=\xCS,\,\tYd_i=\xd_i,\,\tYc_i=\xc_i,\,\tYu_i=\xu_i\right) \\
		&=\sum_{\substack{y\in\mathbb{N}^n\\ \sum_{k=1}^n y_k=\xCS}}
		\phi_{n,m-1}\bigl(y_i,\,\xd_i,\,\xc_i,\,\xu_i; i=1,\ldots,n\bigr)\\
		&=\frac{1}{W_{n,m-1}}\prod_{i=1}^n\!\left(\frac{p_i}{\muc_i}\right)^{\xc_i}\frac{1}{\xd_i!}\left(\frac{p_i}{\mud_i}\right)^{\xd_i}\frac{1}{\xu_i!}\left(\frac{p_i}{\muu_i}\right)^{\xu_i}
		\times \underbrace{ \sum_{\substack{y \in \mathbb{N}^n \\ \sum y_k = \xCS}} \frac{\xCS!}{\prod_{k=1}^n y_k!} \prod_{i=1}^n \left(\frac{p_i}{\muCS}\right)^{y_i} }_{\text{Multinomial Expansion}}.
	\end{align*}
	
	Applying the multinomial theorem to the underbraced term, we obtain:
	\[
	\sum_{\substack{y \in \mathbb{N}^n \\ \sum y_k = \xCS}} \frac{\xCS!}{\prod_{k=1}^n y_k!} \prod_{i=1}^n \left(\frac{p_i}{\muCS}\right)^{y_i}
	= \left( \sum_{i=1}^n \frac{p_i}{\muCS} \right)^{\xCS}
	= \left( \frac{1}{\muCS} \underbrace{\sum_{i=1}^n p_i}_{=1} \right)^{\xCS}
	= \left( \frac{1}{\muCS} \right)^{\xCS}.
	\]
	Substituting this back yields the expression of $\nu_{n,m-1}$ in \eqref{eq:nu}.
\end{proof}

Equipped with the stationary distribution $\nu_{n,m-1}$, we proceed to compute the expected queue lengths, starting with the infinite server nodes. For the downlink server $\mathrm{d}_j$, the expectation is defined as:
\begin{align*}
	\mathbb{E}[\tYd_j] 
	&= \sum_{x \in \mathcal{X}_{3n+1,m-1}} \xd_j\,\nu_{n,m-1}(x) \\
	&= \frac{p_j}{\mud_j} \frac{1}{W_{n,m-1}}
	\sum_{\substack{x \in \mathcal{X}_{3n+1,m-1}\\ \xd_j > 0}}
	\left(\frac{1}{\muCS}\right)^{\xCS}
	\prod_{i=1}^n 
	\left(\frac{p_i}{\muc_i}\right)^{\xc_i}
	\frac{1}{(\xd_i-\mathbf{1}_{\{i=j\}})!}
	\left(\frac{p_i}{\mud_i}\right)^{\xd_i-\mathbf{1}_{\{i=j\}}}
	\frac{1}{\xu_i!}\left(\frac{p_i}{\muu_i}\right)^{\xu_i}.
\end{align*}
Letting $x' = x - \ed_j$, the summation is over $\mathcal{X}_{3n+1,m-2}$, which sums to $W_{n,m-2}$. Thus, we obtain $\mathbb{E}[\tYd_j] = \frac{p_j}{\mud_j} \frac{W_{n,m-2}}{W_{n,m-1}}$. 

By symmetry, the same derivation applies to the uplink server $\mathrm{u}_j$, yielding $\mathbb{E}[\tYu_j] = \frac{p_j}{\muu_j} \frac{W_{n,m-2}}{W_{n,m-1}}$. Summing these two components provides the total communication delay contribution:
\begin{align} \label{eq:proof_comm_sum}
	\mathbb{E}[\tYd_j] + \mathbb{E}[\tYu_j] = p_j \left( \frac{1}{\mud_j} + \frac{1}{\muu_j} \right) \frac{W_{n,m-2}}{W_{n,m-1}} = \frac{W_{n,m-2}}{W_{n,m-1}} \gamma_j.
\end{align}

Next, we analyze the single-server nodes, beginning with the \gls{CS}. We first compute the expected \emph{total} number of tasks $\MCS = \sum_{i=1}^n \tYCS_i$ using the tail-sum formula $\mathbb{E}[\tYCS] = \sum_{l=1}^{m-1} \mathbb{P}(\tYCS \ge l)$. The probability that the queue length exceeds $l$ is:
\begin{align*}
	\mathbb{P}(\MCS \ge l)
	&= \sum_{\substack{x \in \mathcal{X}_{3n+1,m-1}\\ \xCS \ge l}} \nu_{n,m-1}(x) \\
	&= \left(\frac{1}{\muCS}\right)^{l} \frac{1}{W_{n,m-1}}
	\sum_{\substack{x \in \mathcal{X}_{3n+1,m-1}\\ \xCS \ge l}}
	\left(\frac{1}{\muCS}\right)^{\xCS-l}
	\prod_{i=1}^n 
	\left(\frac{p_i}{\muc_i}\right)^{\xc_i}
	\frac{1}{\xd_i!}\left(\frac{p_i}{\mud_i}\right)^{\xd_i}
	\frac{1}{\xu_i!}\left(\frac{p_i}{\muu_i}\right)^{\xu_i}.
\end{align*}
Using the change of variable $x' = x - l \cdot \eCS$, the sum resolves to $W_{n,m-1-l}$. Therefore, the expected total tasks are
\begin{align} \label{eq:EXcs_total}
	\mathbb{E}[\MCS] = \sum_{l=1}^{m-1} \left(\frac{1}{\muCS}\right)^{l} \frac{W_{n,m-1-l}}{W_{n,m-1}} = \tilde{\beta}_{\mathrm{CS},1}.
\end{align} 

Conditioned on this total $\MCS$, the distribution of class counts $(\tYCS_1, \dots, \tYCS_n)$ is multinomial with probabilities $(p_1, \dots, p_n)$. Thus, $\mathbb{E}[\tYCS_j \mid \MCS] = p_j \MCS$, and taking the expectation yields:
\begin{align} \label{eq:proof_cs_exp}
	\mathbb{E}[\tYCS_j] = p_j \mathbb{E}[\MCS] = p_j \tilde{\beta}_{\mathrm{CS},1}.
\end{align}

A similar logic applies to the local computation node $\mathrm{c}_j$, which behaves as a single server with relative load $p_j/\mu^{\mathrm{c}}_j$. Applying the same tail-sum argument results in:
\begin{align} \label{eq:proof_comp_exp}
	\mathbb{E}[\tYc_j] = \sum_{l=1}^{m-1} \left(\frac{p_j}{\muc_j}\right)^{l} \frac{W_{n,m-1-l}}{W_{n,m-1}} = \tilde{\beta}_{j,1}.
\end{align}

Finally, combining the components derived in \eqref{eq:proof_comm_sum}, \eqref{eq:proof_cs_exp}, and \eqref{eq:proof_comp_exp}, we arrive at the desired closed-form expression:
\[
\sum_{s \in \{\mathrm{CS},\mathrm{c},\mathrm{d},\mathrm{u}\}} \mathbb{E}[\tilde{Y}_j^s]
= p_j \tilde{\beta}_{\mathrm{CS},1} + \tilde{\beta}_{j,1} + \frac{W_{n,m-2}}{W_{n,m-1}}\gamma_j.
\]

\subsection{Proof of Equation~\eqref{exp:grad_cs}} \label{proof_exp:grad_cs}
To prove \Cref{exp:grad_cs}, we decompose the sum on the left-hand side into three distinct categories of interaction: queue-queue, queue-delay, and delay-delay correlations. We derive the closed-form expression for each category separately before combining them.

\noindent
\subsubsection*{\textbf{1. Queue-Queue Correlations ($\mathbb{E}[\tYCS_i \tYc_j]$, $\mathbb{E}[\tYCS_i \tYCS_j]$, and $\mathbb{E}[\tYc_i \tYc_j]$)}}

We begin by computing the joint expectation $\mathbb{E}[\tYCS_i \tYc_j]$ between the \gls{CS} queue and the computation queue for any pair of clients $i,j$. Let $\MCS = \sum_{l=1}^n \tYCS_l$ denote the total number of tasks at the \gls{CS}. Using the law of iterated expectations, we have:
\begin{align*}
	\mathbb{E}[\tYCS_i \tYc_j] 
	= \mathbb{E}\bigl[\mathbb{E}[\tYCS_i \tYc_j \mid \MCS, \tYc_j]\bigr] 
	= \mathbb{E}\bigl[\mathbb{E}[\tYCS_i \mid \MCS] \tYc_j\bigr].
\end{align*}
Conditioned on the total load $\MCS$, the vector $(\tYCS_1, \ldots, \tYCS_n)$ follows a multinomial distribution with probability parameters $p_1, \ldots, p_n$ and $\MCS$ trials. Thus, $\mathbb{E}[\tYCS_i \mid \MCS] = p_i \MCS$. Substituting this into the previous equation yields:
\begin{align} \label{eq:EXcsc}
	\mathbb{E}[\tYCS_i \tYc_j] = p_i \mathbb{E}[\MCS \tYc_j].
\end{align}
To compute $\mathbb{E}[\MCS \tYc_i]$, we use the identity $\xCS \xc_i = \sum_{k=1}^{\xCS} \sum_{\ell=1}^{\xc_i} 1$ to rewrite the expectation as a sum of tail probabilities:
\begin{align*} 
	\mathbb{E}[\MCS \tYc_j]
	= \sum_{x \in \mathcal{X}_{n, m-1}} \sum_{k = 1}^{\xCS} \sum_{\ell = 1}^{\xc_j} \nu_{n, m-1}(x) 
	= \sum_{\substack{k, \ell = 1 \\ k + \ell \le m-1}}^{m-2} \sum_{\substack{x \in \mathcal{X}_{n, m-1} \\ \xCS \ge k, \xc_j \ge \ell}} \nu_{n, m-1}(x).
\end{align*}
Substituting the product-form solution $\nu_{n, m-1}(x)$ from \Cref{eq:nu} into this summation gives:
\begin{align*}
	&\mathbb{E}[\MCS \tYc_j] \\
	&= \frac{1}{W_{n, m-1}}
	\sum_{\substack{k, \ell = 1 \\ k + \ell \le m-1}}^{m-2}
	\sum_{\substack{x \in \mathcal{X}_{n, m-1} \\ \xCS \ge k, \xc_j \ge \ell}}
	\prod_{r=1}^n \left( \frac{1}{\muCS} \right)^{\xCS} \left( \frac{p_r}{\muc_r} \right)^{\xc_r} \frac{1}{\xd_r!}\left(\frac{p_r}{\mud_r}\right)^{\xd_r} \frac{1}{\xu_r!}\left(\frac{p_r}{\muu_r}\right)^{\xu_r} \nonumber \\
	&\overset{(\text{a})}= \frac{1}{W_{n, m-1}}
	\sum_{\substack{k, \ell = 1 \\ k + \ell \le m-1}}^{m-2}
	\left( \frac{1}{\muCS} \right)^{k}
	\left( \frac{p_j}{\muc_j} \right)^{\ell}
	\underbrace{\sum_{y \in \mathcal{X}_{n, m - 1 - k - \ell}}
		\prod_{r = 1}^n \left( \frac{1}{\muCS} \right)^{\yCS} \left( \frac{p_r}{\muc_r} \right)^{\yc_r} \frac{1}{\yd_r!}\left(\frac{p_r}{\mud_r}\right)^{\yd_r} \frac{1}{\yu_r!}\left(\frac{p_r}{\muu_r}\right)^{\yu_r}}_{= W_{n, m-1-k-\ell}} \nonumber \\
	&= \sum_{\substack{k, \ell = 1 \\ k + \ell \le m-1}}^{m-2}
	\left( \frac{1}{\muCS} \right)^{k}
	\left( \frac{p_j}{\muc_j} \right)^{\ell}
	\frac{W_{n, m-1-k-\ell}}{W_{n, m-1}} = \tilde{\alpha}_{\mathrm{CS},j}.
\end{align*}
Step (a) follows by performing the change of variables $y = x - k \eCS - \ell \ec_i$, where $\eCS$ and $\ec_i$ denote the $3n$-dimensional canonical unit vectors corresponding to the components $\xCS$ and $\xc_i$, respectively. The inner sum resolves exactly to the normalizing constant of a system with $m-1-k-\ell$ tasks.
Substituting this result back into our initial expression \eqref{eq:EXcsc} yields:
\begin{align} \label{eq:EXcsEXc}
	\mathbb{E}[\tYCS_i \tYc_j] 
	= p_i \tilde{\alpha}_{\mathrm{CS},j} \quad i,j \in \{1,\ldots,n\}.
\end{align}

Next, we evaluate the internal \gls{CS} correlations $\mathbb{E}[\tYCS_i \tYCS_j]$. Consider first the case where the classes are distinct ($i \neq j$). Applying the same conditioning on $\MCS$, we obtain:
\begin{align} \label{eq:EXcsEXcs_inter}
	\mathbb{E}[\tYCS_i \tYCS_j] 
	= \mathbb{E}\bigl[\mathbb{E}[\tYCS_i \tYCS_j \mid \MCS]\bigr] 
	= p_i p_j \mathbb{E}[\MCS(\MCS-1)],
\end{align}
where the second equality follows from the cross-moment property of the multinomial distribution. Using the identity $\xCS (\xCS-1) = \sum_{k=1}^{\xCS-1} 2k$, we can express the factorial moment $\mathbb{E}[\MCS(\MCS-1)]$ as a sum of tail probabilities:
\begin{align}
	&\mathbb{E}[\MCS(\MCS-1)] \nonumber \\
	&= \sum_{x \in \mathcal{X}_{n, m-1}} \xCS (\xCS-1) \nu_{n, m-1}(x) 
	= \sum_{k=1}^{m-2} 2k \sum_{\substack{x \in \mathcal{X}_{n, m-1} \\ \xCS \ge k+1}}  \nu_{n, m-1}(x) \nonumber \\
	&\overset{(\text{b})}= \frac{1}{W_{n, m-1}}
	\sum_{k=1}^{m-2} 2k
	\left( \frac{1}{\muCS} \right)^{k+1}
	\underbrace{\sum_{y \in \mathcal{X}_{n, m - 2 - k}}
		\prod_{r = 1}^n \left( \frac{1}{\muCS} \right)^{\yCS} \left( \frac{p_r}{\muc_r} \right)^{\yc_r} \frac{1}{\yd_r!}\left(\frac{p_r}{\mud_r}\right)^{\yd_r} \frac{1}{\yu_r!}\left(\frac{p_r}{\muu_r}\right)^{\yu_r}}_{= W_{n, m-2-k}} \nonumber \\
	&= \sum_{k=1}^{m-1} 2(k-1)
	\left( \frac{1}{\muCS} \right)^{k}
	\frac{W_{n, m-1-k}}{W_{n, m-1}}. \label{eq:EXcs2}
\end{align}
Here, step (b) follows from the change of variables $y = x - (k+1) \eCS$. Substituting \eqref{eq:EXcs2} into \eqref{eq:EXcsEXcs_inter}, we recover the cross-term:
\begin{align} \label{eq:EXcsEXcs}
	\mathbb{E}[\tYCS_i \tYCS_j] 
	= p_i p_j \sum_{k=1}^{m-1} 2(k-1)
	\left( \frac{1}{\muCS} \right)^{k}
	\frac{W_{n, m-1-k}}{W_{n, m-1}}
	= \tilde{\alpha}^{\mathrm{CS}}_{i,j}, \quad \text{if } i \neq j.
\end{align}

Conversely, for the second moment of a single class within the \gls{CS} ($i=j$), the multinomial variance property yields:
\begin{align*}
	\mathbb{E}[(\tYCS_i)^2] 
	= \mathbb{E}\bigl[\mathbb{E}[(\tYCS_i)^2 \mid \MCS]\bigr] 
	= p_i \mathbb{E}[\MCS] + p_i^2 \mathbb{E}[\MCS(\MCS-1)].
\end{align*}
Substituting the known expressions for $\mathbb{E}[\MCS]$ and $\mathbb{E}[\MCS(\MCS-1)]$ from \eqref{eq:EXcs_total} and \eqref{eq:EXcs2} respectively, and rearranging the terms, we find:
\begin{align}
	\mathbb{E}[(\tYCS_i)^2] 
	= p_i \sum_{k=1}^{m-1} 
	\left( \frac{1}{\muCS} \right)^{k}
	\frac{W_{n, m-1-k}}{W_{n, m-1}}
	\left[1 + 2(k-1)p_i\right] = \tilde{\alpha}^{\mathrm{CS}}_{i,i}.
\end{align}

Finally, applying identical tail-sum arguments to the local computation nodes (analogous to the derivations in \eqref{eq:EXc2} and \eqref{eq:EXcc}, but using the distribution $\nu_{n,m-1}$ instead of $\pi_{n,m-1}$), we recover the final queue-queue terms:
\begin{align}
	\label{eq:EXcc_cs}
	\mathbb{E}[\tYc_i \tYc_j]
	&= \sum_{k=1}^{m-2} \sum_{\ell = 1}^{m-1-k} \left(\frac{p_i}{\muc_i}\right)^{k} \left(\frac{p_j}{\muc_j}\right)^{\ell} \frac{W_{n,m-1-k-\ell}}{W_{n,m-1}} = \tilde{\alpha}_{i,j}, \quad \text{if } i \neq j, \\
	\label{eq:EXc2_cs}
	\mathbb{E}[(\tYc_i)^2]
	&= \sum_{k=1}^{m-1} (2k-1)
	\left( \frac{p_i}{\muc_i} \right)^{k}
	\frac{W_{n, m-1-k}}{W_{n, m-1}} = \tilde{\alpha}_{i,i}.
\end{align}

\subsubsection*{\textbf{2. Queue-Delay Correlations ($\mathbb{E}[\tYCS_i \tYd_j]$, $\mathbb{E}[\tYCS_i \tYu_j]$, $\mathbb{E}[\tYc_i \tYd_j]$ and $\mathbb{E}[\tYc_i \tYu_j]$)}}

Next, we evaluate the interaction between the \gls{CS} queue (or computation queue) of client $i$ and the communication delays of client $j$. We start with the cross-correlation $\mathbb{E}[\tYCS_i \tYd_j]$. Using the law of iterated expectations, we have:
\begin{align*}
	\mathbb{E}[\tYCS_i \tYd_j] 
	&= \mathbb{E}\bigl[\mathbb{E}[\tYCS_i \tYd_j \mid \MCS, \tYd_j]\bigr] 
	= \mathbb{E}\bigl[\mathbb{E}[\tYCS_i \mid \MCS] \tYd_j\bigr]
	= p_i \mathbb{E}[\MCS \tYd_j].
\end{align*}
To compute $\mathbb{E}[\MCS \tYd_j]$, we utilize the stationary distribution of the aggregated process $\nu_{n, m-1}$. Expanding the expectation as a sum of tail probabilities yields:
\begin{align*}
	\mathbb{E}[\MCS \tYd_j]
	&= \sum_{x \in \mathcal{X}_{n, m-1}} \sum_{k = 1}^{\xCS} \xd_j \nu_{n, m-1}(x) 
	= \sum_{k = 1}^{m-2} \sum_{\substack{x \in \mathcal{X}_{n, m-1} \\ \xCS \ge k}} \xd_j \nu_{n, m-1}(x) \\
	&= \frac1{W_{n, m-1}}
	\sum_{k = 1}^{m-2}
	\sum_{\substack{x \in \cX_{n, m-1} \\ \xCS \ge k, \, \xd_j \ge 1}} \left( \frac{1}{\muCS} \right)^{\xCS}
	\prod_{r=1}^n \left( \frac{p_r}{\muc_r} \right)^{\xc_r} \frac{1}{(\xd_r-\mathbf{1}_{\{r=j\}})!}\left(\frac{p_r}{\mud_r}\right)^{\xd_r} \frac{1}{\xu_r!}\left(\frac{p_r}{\muu_r}\right)^{\xu_r} \\
	&\overset{(\text{c})}= \frac{1}{W_{n, m-1}} \sum_{k = 1}^{m-2} \left( \frac{1}{\muCS} \right)^{k} \frac{p_j}{\mud_j} \underbrace{\sum_{y \in \mathcal{X}_{n, m - 2 - k}} \left( \frac{1}{\muCS} \right)^{\yCS} \prod_{r=1}^n \left( \frac{p_r}{\muc_r} \right)^{\yc_r} \frac{1}{\yd_r!}\left(\frac{p_r}{\mud_r}\right)^{\yd_r} \frac{1}{\yu_r!}\left(\frac{p_r}{\muu_r}\right)^{\yu_r}}_{= W_{n, m-2-k}} \\
	&= \frac{p_j}{\mud_j} \sum_{k = 1}^{m-2} \left( \frac{1}{\muCS} \right)^{k} \frac{W_{n, m-2-k}}{W_{n, m-1}} = \tilde{\beta}_{\mathrm{CS},2} \frac{p_j}{\mud_j}.
\end{align*}
Here, step (c) uses the change of variables $y = x - k \eCS - \ed_j$ along with the fact that $\xd_j \ge 1$. By symmetry, the uplink correlation is $\mathbb{E}[\tYCS_i \tYu_j] = p_i \tilde{\beta}_{\mathrm{CS},2} \frac{p_j}{\muu_j}$. Summing these components gives the total interaction between class~$i$'s count at the \gls{CS} and client $j$'s overall communication delay:
\begin{align} \label{eq:EXcsd}
	\mathbb{E}[\tYCS_i (\tYd_j + \tYu_j)] = p_i \tilde{\beta}_{\mathrm{CS},2} \, p_j \left(\frac{1}{\mud_j}+\frac{1}{\muu_j}\right) = p_i \tilde{\beta}_{\mathrm{CS},2} \gamma_j, \quad i,j \in \{1,\ldots,n\}.
\end{align}
Similarly, swapping the roles of $i$ and $j$ gives the symmetric interaction:
\begin{align} \label{eq:EXcsd_2}
	\mathbb{E}[\tYCS_j (\tYd_i + \tYu_i)] = p_j \tilde{\beta}_{\mathrm{CS},2} \gamma_i, \quad i,j \in \{1,\ldots,n\}.
\end{align}
Finally, following the exact same reasoning established in \eqref{eq:EXcd} for the local computation nodes (but using the distribution $\nu_{n,m-1}$ instead of $\pi_{n,m-1}$), we obtain the correlation between the computation queue of client~$i$ and the communication delay of client~$j$:
\begin{align} \label{eq:EXcd_cs}
	\mathbb{E}[\tYc_i (\tYd_j + \tYu_j)] = \tilde{\beta}_{i,2} \gamma_j, \quad i,j \in \{1,\ldots,n\}.
\end{align}

\noindent
\subsubsection*{\textbf{3. Delay-Delay Correlations}}

The delay-delay correlations ($\mathbb{E}[\tYd_i \tYu_j]$, $\mathbb{E}[\tYd_i \tYd_j]$, etc.) involve only the infinite server nodes. Their derivation is identical to the one presented to derive \eqref{eq:EXdd}, yielding the term $\tilde{\psi}_{i,j}$.

Summing all the derived contributions yields the final result in \Cref{exp:grad_cs}: the \gls{CS}-computation queue term $p_i \tilde{\alpha}_{\mathrm{CS},j} + p_j \tilde{\alpha}_{\mathrm{CS},i}$ from \eqref{eq:EXcsEXc}; the \gls{CS}-classes term $\tilde{\alpha}^{\mathrm{CS}}_{i,j}$ from \eqref{eq:EXcsEXcs} and \eqref{eq:EXcs2}; the computation queue-queue term $\tilde{\alpha}_{i,j}$ from \eqref{eq:EXcc_cs} and \eqref{eq:EXc2_cs}; the queue-delay interactions $\tilde{\beta}_{\mathrm{CS},2} (p_i \gamma_j + p_j \gamma_i) + \tilde{\beta}_{i,2}\gamma_j + \tilde{\beta}_{j,2}\gamma_i$ from \eqref{eq:EXcsd}, \eqref{eq:EXcsd_2}, and \eqref{eq:EXcd_cs}; and the delay-delay term $\tilde{\psi}_{i,j}$.

\section{Proof of Proposition~\ref{prop:time-eps2}} \label{proof_prop:time-eps2}

We establish the equations in sequence, starting with Equation~\eqref{eq:tau2}, followed by \eqref{eq:throughput2}, and finally \eqref{eq:grad_throughput2}.

\subsection{Proof of Equation~\eqref{eq:tau2}}

Let $T_0 = 0$, and let $(T_k)_{k \ge 1}$ represent the successive instants of service completion at the \gls{CS}. Consequently, the duration of the $k$-th global round is exactly $T_{k+1} - T_k$. To execute $K_{\epsilon}$ rounds, the total wall-clock time required is naturally given by the telescoping sum:
\begin{align} \label{tau_def2}
	\tau_{\epsilon} = \sum_{k=1}^{K_{\epsilon}(p,m)} (T_k - T_{k-1}).
\end{align}

To analyze the expected duration, consider the counting process $N = (N(t))_{t \ge 0}$ that tracks the cumulative number of \gls{CS} departures up to time $t$. This process is governed by the state transitions of the \gls{CS} queues:
$$
N(t) = \sum_{0 < s \le t} \sum_{i=1}^n \left( \YCS_i(s^-) - \YCS_i(s) \right)^+.
$$
Because the underlying state process $Y$ is stationary, the point process $N$ inherits this stationarity. Therefore, under the associated Palm probability measure $\bPp$, the inter-departure intervals are identically distributed, meaning that $\bEp[T_k - T_{k-1}] = \bEp[T_1]$ for all $k \ge 1$.

Invoking the fundamental inversion formula for stationary point processes (see, e.g., \cite[Corollary~6.16]{serfozo1999} applied to the constant function $X_t = 1$), we obtain the relation:
$$
\tilde{\lambda}(p,m) \, \bEp[T_1] = 1,
$$
where $\tilde{\lambda}$ denotes the steady-state departure intensity of $(N(t))_{t \ge 0}$ under the standard probability measure $\bP$. Substituting this relationship back into the expected value of \eqref{tau_def2} directly yields:
$$
\bEp[\tau_{\epsilon}] = \sum_{k=1}^{K_{\epsilon}(p,m)} \bEp[T_k - T_{k-1}] = \frac{K_{\epsilon}(p,m)}{\tilde{\lambda}(p,m)},
$$
which establishes \Cref{eq:tau2}. 

\subsection{Proof of Equation~\eqref{eq:throughput2}}

To rigorously derive the expression for $\tilde{\lambda}(p,m)$ in \eqref{eq:throughput2}, it is convenient to analyze the system's behavior by aggregating the tasks at the \gls{CS} across all classes.

\begin{lemma} \label{lem:aggregated_dist2}
	Define the continuous-time aggregated state process as
	$$
	\xi(t) = \Bigl(\sum_{j=1}^n \YCS_j(t), \,\Yd_i(t), \,\Yc_i(t), \,\Yu_i(t) ;\ i=1,\ldots,n\Bigr),
	$$
	which evolves over the state space $\mathcal{X}_{3n+1,m}$. This process forms an ergodic, homogeneous Markov chain. Its unique invariant probability measure is given by:
	\begin{align} \label{eq:nu_full}
		\nu_{n,m}(x)
		&= \frac{1}{W_{n,m}}
		\left(\frac{1}{\muCS}\right)^{\xCS}
		\prod_{i=1}^n
		\left[ \left( \frac{p_i}{\muc_i} \right)^{\xc_i}
		\frac{1}{\xd_i!}\left( \frac{p_i}{\mud_i} \right)^{\xd_i}
		\frac{1}{\xu_i!}\left( \frac{p_i}{\muu_i} \right)^{\xu_i} \right],
	\end{align}
	for any state $x = (\xCS,\xd_1,\dots,\xd_n,\xc_1,\dots,\xc_n,\xu_1,\dots,\xu_n) \in \mathcal{X}_{3n+1,m}$, where $W_{n,m}$ is the identical normalizing constant defined in Proposition~\ref{prop:jackson2}.
\end{lemma}
\begin{proof}
	By applying the same marginalization argument used in Lemma~\ref{lem:aggregated_dist}, we can directly collapse the detailed stationary distribution $\phi_{n,m}$ of the full process $Y$ to obtain $\nu_{n,m}$.
\end{proof}

In our proposed architecture (\Cref{sec:model_description2}), the \gls{CS} serves tasks at a rate $\muCS$ strictly when it is not idle. Thus, the expected instantaneous throughput is $\tilde{\lambda}(p,m) = \muCS \bP \bigl( \sum_{i=1}^n \YCS_i > 0 \bigr)$. Expanding this expectation over the aggregated invariant measure $\nu_{n,m}$ gives:
\begin{align} \label{tlambda_proof}
	\tilde{\lambda}(p,m) &=  \muCS\, \sum_{\substack{x \in \mathcal{X}_{3n+1,m} \\ \xCS \ge 1}} \nu_{n,m}(x) \nonumber \\[0.5ex]
	&= \frac{\muCS}{W_{n,m}} \sum_{\substack{x \in \mathcal{X}_{3n+1,m} \\ \xCS \ge 1}} \left(\frac{1}{\muCS}\right)^{\xCS}
	\prod_{l=1}^n \left(\frac{p_l}{\muc_l}\right)^{\xc_l} \frac{1}{\xd_l!}\left(\frac{p_l}{\mud_l}\right)^{\xd_l} \frac{1}{\xu_l!}\left(\frac{p_l}{\muu_l}\right)^{\xu_l} \nonumber \\[0.5ex]
	&= \frac{1}{W_{n,m}} 
	\underbrace{
		\sum_{\substack{x \in \mathcal{X}_{3n+1,m} \\ \xCS \ge 1}} \left(\frac{1}{\muCS}\right)^{\xCS-1}
		\prod_{l=1}^n \left(\frac{p_l}{\muc_l}\right)^{\xc_l} \frac{1}{\xd_l!}\left(\frac{p_l}{\mud_l}\right)^{\xd_l} \frac{1}{\xu_l!}\left(\frac{p_l}{\muu_l}\right)^{\xu_l}}_{= W_{n,m-1}} \nonumber \\[0.5ex]
	&= \frac{W_{n,m-1}}{W_{n,m}}. 
\end{align}
The step mapping the summation to $W_{n,m-1}$ is achieved by factoring out $1/\muCS$ and performing the geometric state shift $y = x - \eCS$ (where $\eCS$ denotes the canonical unit vector corresponding to the \gls{CS} component). This concludes the proof of \Cref{eq:throughput2}.

\subsection{Proof of Equation~\eqref{eq:grad_throughput2}}

We now compute the sensitivity of the throughput with respect to the routing decisions. From our earlier derivation leading to \Cref{eq:grad-logW}, the partial derivative of the normalization constant $W_{n,m-1}$ is given by:
$$
\frac{\partial W_{n, m-1}}{\partial p_j}
= \frac{W_{n, m-1}}{p_j} \mathbb{E}[\tYCS_j + \tYd_j + \tYc_j + \tYu_j].
$$
Similarly, applying identical probabilistic logic to the full $m$-task system yields:
$$
\frac{\partial W_{n, m}}{\partial p_j}
= \frac{W_{n, m}}{p_j} \mathbb{E}[\YCS_j + \Yd_j + \Yc_j + \Yu_j].
$$
To find the gradient of the throughput $\tilde{\lambda}(p,m) = W_{n,m-1}/W_{n,m}$, we simply apply the quotient rule and substitute the above derivatives:
\begin{align*}
	\frac{\partial}{\partial p_j} \tilde{\lambda}(p,m)
	&= \frac{1}{W_{n, m}^2} \left( W_{n,m}\frac{\partial W_{n, m-1}}{\partial p_j} - W_{n,m-1}\frac{\partial W_{n, m}}{\partial p_j} \right) \\
	&= \frac{W_{n, m-1}}{W_{n, m}} \left( \frac{1}{W_{n, m-1}}\frac{\partial W_{n, m-1}}{\partial p_j} - \frac{1}{W_{n, m}}\frac{\partial W_{n, m}}{\partial p_j} \right) \\
	&= \frac{\tilde{\lambda}(p,m)}{p_j} \left( \mathbb{E}[\tYCS_j + \tYd_j + \tYc_j + \tYu_j] - \mathbb{E}[\YCS_j + \Yd_j + \Yc_j + \Yu_j] \right).
\end{align*}
This provides the exact gradient and concludes the proof.

\section{Proof of Proposition~\ref{prop:energy-eps2}} \label{proof_prop:energy-eps2}

We begin by formalizing the instantaneous power draw of the network at any wall-clock time $t$, denoted by $P(t)$. Consistent with the physical energy model established in \Cref{sec:nrg_model2}, the total power is the sum of the active states across all network components:
\begin{align} \label{eq:inst_power2}
	P(t)
	=  \PCS \,\mathbf{1}\Bigl\{ \sum_{i=1}^n \YCS_i(t) > 0 \Bigr\}
	+ \sum_{i=1}^n \Bigl(
	\Pc_i \,\mathbf{1}\{\Yc_i(t) > 0\}
	+ \Pu_i \,\Yu_i(t)
	+ \Pd_i \,\Yd_i(t)
	\Bigr).
\end{align}

Recall that $K_{\epsilon}$ represents the deterministic number of global rounds required to achieve $\epsilon$-accuracy (which depends entirely on the routing vector $p$, the concurrency level $m$, and the hardware service rates), and $\tau_{\epsilon}$ denotes the corresponding random wall-clock time. The cumulative energy consumed up to this accuracy threshold is the integral of the power over time:
$$
E_{\epsilon}
= \int_0^{\tau_{\epsilon}} P(t)\,dt
= \sum_{k=1}^{K_{\epsilon}(p,m)} \int_{T_{k-1}}^{T_k} P(t)\,dt,
$$
where $T_0=0$ and $(T_k)_{k \ge 1}$ is the sequence of update completion times at the \gls{CS}.

Because the aggregated continuous-time state process $(\xi(t))_{t \ge 0}$ is strictly stationary and ergodic (as proved in Lemma~\ref{lem:aggregated_dist}), the sequence of energy increments across the renewal intervals $[T_{k-1}, T_k)$ is also stationary. Therefore, under the Palm probability measure $\bPp$ associated with the completion instants $\{T_k\}$, the expected energy expended during any single round is invariant:
$$
\mathbb{E}^0\!\left[\int_{T_{k-1}}^{T_k} P(t)\,dt\right]
= \mathbb{E}^0\!\left[\int_0^{T_1} P(t)\,dt\right].
$$

By invoking the fundamental inversion formula for stationary point processes (see \cite[Corollary~6.16]{serfozo1999}, evaluated for the function $X_t=P(t)$), we map this Palm expectation back to the standard time-average expectation:
$$
\mathbb{E}^0\!\left[\int_0^{T_1} P(t)\,dt\right]
= \frac{\mathbb{E}[P(0)]}{\tilde{\lambda}(p,m)},
$$
where $\tilde{\lambda}$ is the steady-state system throughput derived in \Cref{eq:throughput2}. Substituting this relationship into the cumulative energy summation yields:
\begin{align} \label{eq:energy_exp2}
	\mathbb{E}^0[E_{\epsilon}]
	= \sum_{k=1}^{K_{\epsilon}(p,m)} \frac{\mathbb{E}[P(0)]}{\tilde{\lambda}(p,m)}
	= \frac{K_{\epsilon}(p,m)}{\tilde{\lambda}(p,m)}\,\mathbb{E}[P(0)].
\end{align}

The next step is to compute the steady-state expected power $\mathbb{E}[P(0)]$. Taking the expectation of \eqref{eq:inst_power2} with respect to the invariant measure $\nu_{n,m}$ (from Equation~\eqref{eq:nu_full}), and exploiting the linearity of expectation, we decompose the total power as:
$$
\mathbb{E}[P(0)]
= \PCS \,\mathbb{P}\Bigl(\sum_{i=1}^n \YCS_i > 0\Bigr)
+ \sum_{i=1}^n \Bigl(
\Pc_i \,\mathbb{P}(\Yc_i > 0)
+ \Pd_i \,\mathbb{E}[\Yd_i]
+ \Pu_i \,\mathbb{E}[\Yu_i]
\Bigr).
$$

We evaluate each term by leveraging our previously established queueing metrics. For the \gls{CS} active probability, we use the identity from Equation~\eqref{tlambda_proof}:
$$
\mathbb{P}\Bigl(\sum_{i=1}^n \YCS_i > 0\Bigr)
= \frac{W_{n,m-1}}{W_{n,m}}\,\frac{1}{\muCS}.
$$
Applying identical marginalization logic to the local computation nodes yields:
$$
\mathbb{P}(\Yc_i > 0)
= \frac{W_{n,m-1}}{W_{n,m}}\,\frac{p_i}{\muc_i}.
$$
Similarly, for the infinite-server communication phases (downlink and uplink), we have:
$$
\mathbb{E}[\Yd_i]
= \frac{W_{n,m-1}}{W_{n,m}}\,\frac{p_i}{\mud_i},
\qquad
\mathbb{E}[\Yu_i]
= \frac{W_{n,m-1}}{W_{n,m}}\,\frac{p_i}{\muu_i}.
$$

Finally, substituting these steady-state terms back into $\mathbb{E}[P(0)]$, we observe that every term shares the common factor $\frac{W_{n,m-1}}{W_{n,m}}$, which is precisely the throughput $\tilde{\lambda}(p,m)$. When injected into \eqref{eq:energy_exp2}, this factor cancels the $1/\tilde{\lambda}(p,m)$ multiplier. This cancellation yields the final closed-form expression:
$$
\mathbb{E}^0[E_{\epsilon}]
= K_{\epsilon}(p,m) \left(\frac{\PCS}{\muCS} +
\sum_{i=1}^n p_i
\left( 
\frac{\Pc_i}{\muc_i}
+ \frac{\Pd_i}{\mud_i}
+ \frac{\Pu_i}{\muu_i}
\right)\right).
$$
This demonstrates that the expected energy per round depends strictly on the service capacities and routing decisions, entirely independent of the queueing delays, concluding the proof.

\end{document}